\documentclass[10pt]{article} 
\PassOptionsToPackage{table}{xcolor}

\newcounter{anonymous}
\ifnum\value{anonymous}=1
        \usepackage{tmlr}
\else
        \usepackage[preprint]{tmlr}
\fi

\usepackage{hyperref}
\usepackage{url}      
\usepackage{booktabs}        
\usepackage{amsfonts}        
\usepackage{nicefrac}        
\usepackage{microtype}       
\usepackage[table]{xcolor}   
\usepackage{latexsym}
\usepackage{amssymb}
\usepackage{amsmath}
\usepackage{siunitx}
\usepackage{mathtools}
\usepackage{booktabs}
\usepackage{enumitem}
\usepackage{graphicx}
\usepackage{color}
\usepackage{tikz}
\usepackage{booktabs}
\usepackage{url}
\usepackage{pifont}
\usepackage{multirow}
\usepackage{anyfontsize}
\usepackage{tabularx}
\usepackage{xfrac}
\usepackage{wrapfig}
\usepackage{lipsum}
\usepackage{subcaption}
\usepackage[most]{tcolorbox}
\usepackage[referable]{threeparttablex}
\usepackage{footnote}
\usepackage{bold-extra}
\usepackage{eurosym}
\usepackage{makecell}
\usepackage{wheelchart}
\usepackage{bm}
\usepackage{wrapfig}
\usepackage[capitalize,noabbrev]{cleveref}
\usepackage{cases}
\usepackage{ifthen}

\usetikzlibrary{positioning, fadings, calc, shapes.callouts,shapes.arrows, bayesnet, fit, backgrounds}

\definecolor{baseline}{HTML}{666666}
\definecolor{aggregation}{HTML}{000080}
\definecolor{crowd}{HTML}{800080}
\definecolor{ensemble}{HTML}{008080}
\definecolor{color_annot_green}{RGB}{0,128,128}
\definecolor{color_annot_blue}{RGB}{42,127,255}
\definecolor{color_annot_violet}{RGB}{127,0,127}
\definecolor{pink_ground_truth}{RGB}{128,0,128}
\definecolor{green_annotations}{RGB}{0,128,128}
\definecolor{datasetcolor}{RGB}{240,240,240}
\definecolor{navy_accent}{RGB}{  0,  0,128}
\newcommand{\cmark}{\raisebox{-0.1em}[0pt][0pt]{\ding{51}}}
\newcommand{\xmark}{\raisebox{-0.1em}[0pt][0pt]{\ding{55}}}

\newcommand{\yesnoise}{\cellcolor{yescolor!20} $\text{\cmark}_{\text{\textBF{NL}}}$} 
\newcommand{\yestrue}{\cellcolor{unknowncolor!20} $\text{\cmark}_{\text{\textBF{TL}}}$} 

\newcommand{\no}{\cellcolor{nocolor!20}\xmark}
\newcommand{\unknown}{\textbf{?}}

\definecolor{yescolor}{rgb}{0,0.502,0.502}
 
\definecolor{nocolor}{rgb}{0.502,0,0.502}
\definecolor{unknowncolor}{rgb}{0.402,0.402,0.402}
\definecolor{partialcolor}{rgb}{0.302,0.302,0.302}

\newcommand{\defas}{\coloneqq}

\DeclareMathOperator*{\argmax}{arg\,max}
\DeclareMathOperator*{\argmin}{arg\,min}
\newsavebox\CBox
\def\textBF#1{\sbox\CBox{#1}\resizebox{\wd\CBox}{\ht\CBox}{\textbf{#1}}}
\newcounter{loopc}
\NewDocumentCommand\towrite{O{1}m}%
  {{\color{red}#2\forloop{loopc}{1}{\value{loopc} < #1}{; #2}}}
\newdimen\abovecrulesep
\newdimen\belowcrulesep
\makeatletter
\patchcmd{\@@@cmidrule}{\aboverulesep}{\abovecrulesep}{}{}
\patchcmd{\@xcmidrule}{\belowrulesep}{\belowcrulesep}{}{}
\makeatother
\newcommand*\circled[1]{\tikz[baseline=(char.base)]{\node[shape=circle,draw,inner sep=0.5pt] (char) {#1};}}
\renewcommand{\mid}{\hspace*{0.01cm}\vert\hspace*{0.01cm}}  

\definecolor{myboxcolor}{RGB}{0,128,128}
\newenvironment{mybox}[1][]
{
    \begin{tcolorbox}
    [
        enhanced,
        title=#1,
        colback=myboxcolor!7,
        colbacktitle=myboxcolor!7,
        coltitle=black,
        left=5.25pt,
        right=5.25pt,
        top=9pt,
        bottom=5pt,
        attach boxed title to top left={xshift=8pt, yshift=-9pt},
        boxed title style={frame hidden, size=small, colback=myboxcolor!7},
        sharp corners,
        rounded corners,
        arc=7pt,
    ]
}
{
    \end{tcolorbox}
}

\newcommand{\BlackBox}{\rule{1.5ex}{1.5ex}}  
\ifdefined\proof
    \renewenvironment{proof}{\par\noindent{\bf Proof\ }}{\hfill\BlackBox\\[2mm]}
\else
    \newenvironment{proof}{\par\noindent{\bf Proof\ }}{\hfill\BlackBox\\[2mm]}
\fi
 
\newtheorem{theorem}{Theorem}
 
\newtheorem{proposition}[theorem]{Proposition}

\newcommand{\impact}[1]{%
  \subsection*{Broader Impact Statement}
  #1%
}
\newcommand{\acks}[1]{%
  \subsection*{Acknowledgements}
  #1%
}
\newboolean{tmlr}
\setboolean{tmlr}{true}

\title{
    \texttt{crowd-hpo}: Realistic Hyperparameter Optimization and Benchmarking for Learning from Crowds with Noisy Labels
}

\author{%
    \name Marek Herde, Lukas Lührs, Denis Huseljic, and Bernhard Sick \email \href{mailto:marek.herde@uni-kassel.de}{marek.herde@uni-kassel.de} \\
    \addr Intelligent Embedded Systems, University of Kassel, Hesse, Germany
}

\def\month{MM}  
\def\year{YYYY} 
\def\openreview{\url{https://openreview.net/forum?id=XXXX}} 

\begin{document}

\maketitle

\begin{abstract}
    Crowdworking is a cost-efficient solution for acquiring class labels. Since these labels are subject to noise, various approaches to learning from crowds have been proposed. Typically, these approaches are evaluated with default hyperparameter configurations, resulting in unfair and suboptimal performance, or with hyperparameter configurations tuned via a validation set with ground truth class labels, representing an often unrealistic scenario. Moreover, both setups can produce different approach rankings, complicating study comparisons. Therefore, we introduce \texttt{crowd-hpo} as a framework for evaluating approaches to learning from crowds in combination with criteria to select well-performing hyperparameter configurations with access only to noisy crowd-labeled validation data. Extensive experiments with neural networks demonstrate that these criteria select hyperparameter configurations, which improve the learning from crowd approaches' generalization performances, measured on separate test sets with ground truth labels. Hence, incorporating such criteria into experimental studies is essential for enabling fairer and more realistic benchmarking.
\end{abstract}

\ifthenelse{\boolean{tmlr}}
    {%
    }
    {%
    }

\ifnum\value{anonymous}=1
    \newcommand{\mygithub}{\url{https://github.com/anonymous/anonymous}}
\else
    \newcommand{\mygithub}{\url{https://github.com/ies-research/multi-annotator-machine-learning/tree/crowd-hpo}}
\fi
\section{Introduction}
\label{sec:introduction}
Crowdworking represents a popular and cost-efficient solution to label data instances for classification tasks~\citep{vaughan2017making}. However, the corresponding crowdworkers are error-prone for various reasons, e.g., missing domain knowledge, lack of concentration, or even adversarial behavior~\citep{herde2021survey}. Training deep neural networks with noisy crowd-labeled data decreases generalization performance because these networks tend to memorize the false class labels~\citep{zhang2017understanding}. Hence, many approaches intend to improve the robustness against noisy labels. Together, they form the research area of \textit{learning from noisy labels} (LNL) with the core topics of regularization, sample selection, robust loss functions, or dedicated neural network architectures~\citep{song2022learning}. Within this area, \textit{learning from crowds}\footnote{We use the term learning from crowds, whereas other publications in the same research area refer to multiple annotators~\citep{li2022beyond} or labelers~\citep{rodrigues2013learning} instead of crowdworkers.} (LFC, ~\citeauthor{raykar2010learning},~\citeyear{raykar2010learning}) approaches explicitly handle crowd-labeled data, where each instance receives a (potentially varying) number of noisy class labels and where we know which label originates from which crowdworker. Accordingly, these approaches estimate the crowdworkers' performances (e.g., labeling accuracies) to infer the instances' true (i.e., ground truth) class labels. Many experimental evaluation studies have demonstrated the performance gains of such approaches~\citep{rodrigues2018deep,chu2021learning,nguyen2024noisy}.
\begin{figure}[t!]
    \centering
    \includegraphics[width=\textwidth]{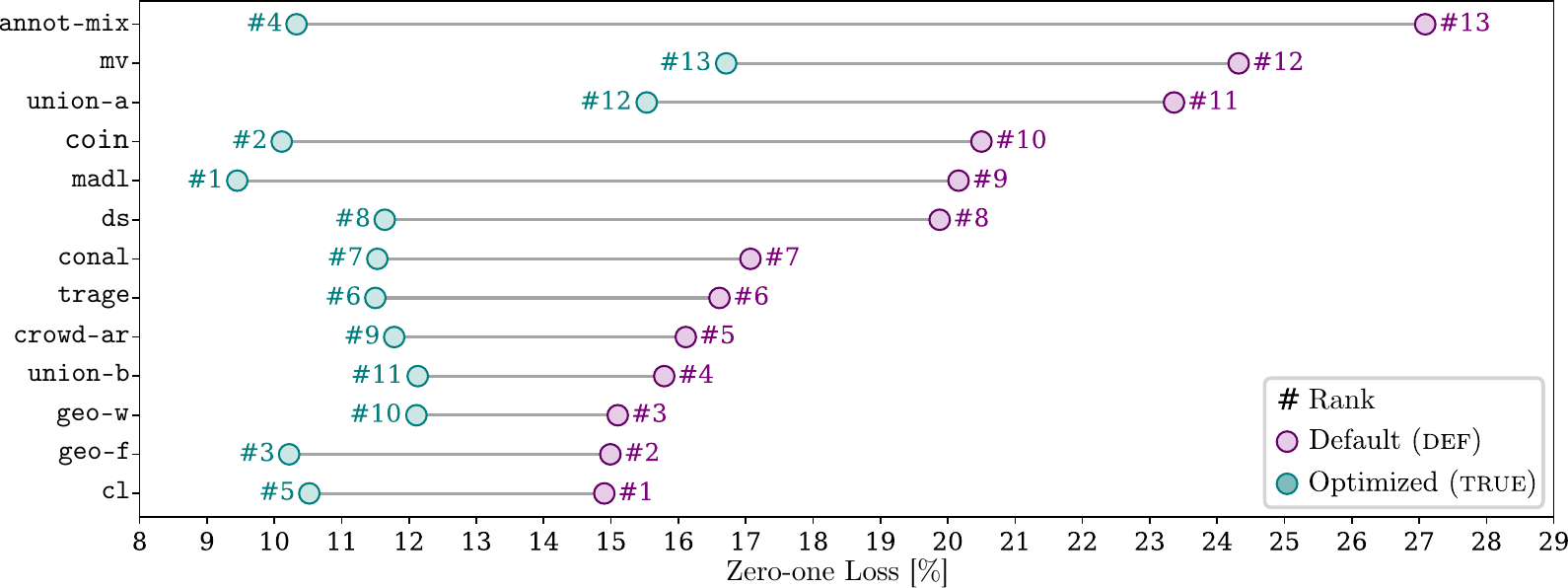}
    \caption{
        {\color{nocolor}\textbf{Default} (\textsc{def})} \textbf{versus} {\color{yescolor}\textbf{optimized}  (\textsc{true})} \textbf{HPCs for LFC approaches}. The $y$-axis lists the LFC approaches, and the $x$-axis the zero-one loss evaluated on a test set with true class labels of the \texttt{reuters-full} dataset~\citep{rodrigues2017learning}, whose training set contains noisy class labels from crowdworkers. Default HPCs result in substantially worse performance than HPCs optimized via validation data with true class labels. Further, HPO alters the approaches' ranking. For example, \texttt{cl}~\citep{rodrigues2018deep} performs best under default and only fifth-best under optimized HPCs, whereas \texttt{madl}~\citep{herde2023multi} moves from the ninth place with the default HPCs to the first place after optimization.
    }
    \label{fig:graphical-abstract}
\end{figure}
Translating these gains into practice demands effective \textit{hyperparameter optimization} (HPO) to find well-performing \textit{hyperparameter configurations} (HPCs). While approaches for training standard neural networks are tuned against a validation set with true class labels, LFC approaches have no access to such a set if all class labels originate from crowdworkers. As a result, HPO becomes a more difficult challenge. Potential workarounds are using data-agnostic default HPCs~\citep{chen2021robustness} or explicitly requiring access to a validation set with true class labels~\citep{herde2023multi}. We refer to both procedures as \textit{hyperparameter selection} (HPS) criteria because each chooses HPCs for the LFC approaches. Figure~\ref{fig:graphical-abstract} exemplifies that both HPS criteria lead to different test losses per LFC approach and even rankings between LFC approaches. Specifying default HPCs is realistic because it requires no true class labels. Nevertheless, such an HPS criterion often yields suboptimal performance results that are heavily influenced by nonobjective choices such as the experimenters' or software frameworks' presets, undermining fairness~\citep{bagnall2017use}. By contrast, HPO on a validation set with true labels can produce superior and fairer results when every LFC approach receives the same search budget. Nonetheless, this HPS criterion is unrealistic in an LFC setting where only noisy crowdworkers provide labels. Existing literature lacks HPS criteria to perform experiments for a fair (involving HPO) and realistic (with access only to crowd-labeled validation data) comparison of LFC approaches. Motivated by these observations and related ones in areas such as partial label learning~\citep{wang2025realistic}, we analyze the following \textit{research questions} (RQs):
\begin{mybox}[\texttt{crowd-hpo}: \textbf{Research Questions and Contributions}]
    \begin{enumerate}
        \item[$RQ_1$:] \textit{Given access only to crowd-labeled validation data, which evaluated hyperparameter selection criterion yields the highest performances for LFC approaches?}
        \item[$RQ_2$:] \textit{Given the best-evaluated hyperparameter selection criterion for crowd-labeled validation data, how do learning from crowds approaches compare in performance?}
    \end{enumerate}    
    Based on these research questions, we propose \texttt{crowd-hpo} for learning from crowds approaches with crowd-labeled validation data contributing:
    \begin{itemize}[leftmargin=*]
        \item a framework of hyperparameter selection criteria based on empirical risk measures to be combined in a robust ensemble,
        \item an extensive benchmark of $13$ learning from crowds approaches across $5$ real-world datasets, each with $7$ variants of noisy labels from humans,
        \item recommendations for a realistic and fair experimentation to compare learning from crowds approaches' performances in combination with hyperparameter optimization,
        \item and a comprehensive codebase\footnote{\mygithub} to reproduce and perform experimental studies for learning from crowds approaches in combination with hyperparameter optimization.
    \end{itemize}    
\end{mybox}
\section{Related Work}
\label{sec:related-work}
This section presents a discussion on foundational and related works on LFC~\citep{raykar2010learning} approaches for classification tasks, their experimental studies, and validation with noisy class labels in the broader area of LNL~\citep{song2022learning}. In short, this discussion confirms that most experimental studies on LFC follow different experimentation protocols, of which none consider HPO with noisy crowd-labeled validation data. Although other works on LNL address the validation issues in the presence of noisy labels in varying contexts, they do not explicitly validate with crowdworkers of varying performances.

\begin{table}[h!]
\centering
\scriptsize
\setlength{\tabcolsep}{4.85pt}
\def\arraystretch{1.0}
\caption{
    \textbf{Overview of experimental studies of LFC approaches training neural networks for classification tasks.} Each row represents one study sorted by publication years, while the columns refer to the characteristics of such a study. We denote counts by the $\#$ symbol. We account for multiple simulation methods for the same single-labeled and variants for the same crowd-labeled dataset by ($\times$ \ldots). The symbols \yestrue\ (\textBF{True Labels}) and \yesnoise\ (\textBF{Noisy Labels}) denote the validation label type, whereas \no\ indicates that the respective aspect has been ignored. If no information is available, we denote \unknown\ as a symbol.
}
\label{tab:hyperparam_overview}
\begin{tabular}{llcccccccc}
    \toprule
    \multirow{2}{*}{\textBF{Study}} & \multirow{2}{*}{\textBF{Venue}} & \multicolumn{2}{c}{\textBF{Approaches [$\#$]}} & \multicolumn{2}{c}{\textBF{Datasets [$\#$]}} & \multicolumn{2}{c}{\textBF{Hyperparameter Optimization}} & \multirow{2}{*}{\textBF{Early Stopping}} \\ 
    \cmidrule(lr){3-4} \cmidrule(lr){5-6} \cmidrule(lr){7-8}
    
    &  & \textBF{Two-stage}  & \textBF{One-stage} & \textBF{Simulated} & \textBF{Real} & \makecell{\textBF{\hspace*{.7mm} Per Dataset \hspace*{.7mm}}} & \makecell{\textBF{Per Approach}} & \\
    \midrule
    \citeauthor{rodrigues2018deep}        & AAAI & 3 & 4 & 1 & 1 & \yestrue & \no & \no \\
    \citeauthor{cao2018maxmig}           & ICLR & 1 & 4 & 3 ($\times$ 6) & 1 & \no & \no & \no \\
    \citeauthor{tanno2019learning}       & CVPR & 1 & 5 & 2 ($\times$ 2) & 0 & \no & \no & \yestrue \\
    \citeauthor{li2021trustable}         & TMM & 4 & 6 & 4 & 2 & \no & \no & \no \\ 
    \citeauthor{wei2022deep}             & TNNLS & 1 & 6 & 4 ($\times$ 4) & 2 & \no & \no & \no \\
    \citeauthor{li2022beyond}            & MLJ & 2 & 5 & 4 ($\times$ 2) & 2 & \no & \yestrue &  \yesnoise \\
    \citeauthor{herde2023multi}          & TMLR & 1 & 6 & 4 ($\times$ 4) & 2 & \yestrue & \yestrue & \yestrue \\
    \citeauthor{ibrahim2023deep}         & ICLR & 2 & 8 & 2 ($\times$ 2) & 2 & \yestrue & \yestrue & \yestrue  \\
    \citeauthor{cao2023learning}         & SIGIR & 5 & 5 & 0 & 3 & \unknown & \unknown &  \unknown \\
    \citeauthor{herde2024annot}          & ECAI & 2 & 9 & 6 & 5 & \yestrue & \no &  \yestrue \\
    \citeauthor{zhang2024coupled}        & AAAI & 1 & 6 & 2 ($\times$ 4) & 3 & \no & \no & \yestrue  \\
    \citeauthor{li2024transferring}      & TPAMI & 6 & 7 & 4 ($\times$ 5) & 3 & \no & \no & \no \\
    \citeauthor{nguyen2024noisy}         & NeurIPS & 1 & 5 & 2 ($\times$ 3) & 3 & \no & \no & \no \\ 
    \citeauthor{han2024collaborative}    & NeurIPS & 3 & 9 & 13 ($\times$ 2) & 2 & \no & \no & \no \\
    \citeauthor{guo2024learning}         & NeurIPS & 2 & 7 & 2 ($\times$ 3) & 4 & \no & \no & \yesnoise \\
    \citeauthor{herde2024annot}         & NeurIPS & 2 & 10 & 0 & 1 ($\times$ 7) & \yestrue & \no & \no \\
    \hline
    \texttt{crowd-hpo} & -- & 2 & 11 & 0 & 5 ($\times$ 7) & \yesnoise & \yesnoise & \no \\
    \bottomrule
    \end{tabular}
\end{table}


\paragraph{Learning from Crowds Approaches} Literature differs between two-stage and one-stage LFC approaches~\citep{li2022beyond}. Two-stage approaches aggregate the noisy crowd-labeled class labels per instance in the first stage and use these aggregated labels as true class label estimates for training neural networks in the second stage. The most common aggregation algorithm is majority voting (\texttt{mv}), which implicitly assumes equal performances across the crowdworkers~\citep{chen2022label,jiang2021learning}. In contrast, the Dawid-Skene algorithm (\texttt{ds},~\citeauthor{dawid1979maximum},~\citeyear{dawid1979maximum}) leverages the \textit{expectation-maximization} (EM) algorithm, where the true label probabilities are estimated in the E-step to update the crowdworkers' confusion matrices in the M-step. Typically, such label aggregation approaches only operate with the given labels as inputs~\citep{zhang2016learning} and expect more than one class label per instance~\citep{khetan2018learning}. One-stage approaches aim to overcome these limitations by jointly training a neural network for estimating the true labels and a model for evaluating the crowdworkers' performances~\citep{herde2023multi}. The latter model is often implemented as weights of noise adaptation layers~\citep{rodrigues2018deep,chu2021learning} or probabilistic confusion matrices~\citep{tanno2019learning,chu2021learning,ibrahim2023deep} to model crowdworkers' class-dependent performances. More complex models, designed as (deep) neural networks, estimate performances as a function of instances and crowdworkers~\citep{le2020disentangling,li2022beyond,cao2023learning,herde2024annot}.

\paragraph{Experimental Studies for Learning from Crowds} For a better understanding of experimenting with LFC approaches, Table~\ref{tab:hyperparam_overview} overviews and characterizes recent experimental studies of LFC approaches. Most studies focus on presenting a new LFC approach compared to state-of-the-art competitors. 
We report the number of evaluated two-stage and one-stage LFC approaches for each study. 
Here, we count individual approaches if they incorporate distinct methodological ideas. 
In addition, we report the number of datasets used in each study. We distinguish between simulated and real crowd-labeled datasets. Simulated datasets are built on top of standard single-labeled datasets, e.g., \texttt{cifar10}~\citep{krizhevsky2009learning}, by simulating the labeling process of the crowdworkers. For the simulated data, most experimental studies consider multiple single-labeled datasets and multiple simulation methods for the noisy class labels. Analog to this, multiple variants of crowd-labeled datasets can be constructed by subsampling the crowdorkers' labels, e.g., by keeping only a certain number of class labels per instance~\citep{wei2021learning,herde2024dopanim}. We take both procedures into account by denoting the product term $(\#\,\mathrm{datasets} \times \#\,\mathrm{variants})$. Central to our analysis is the handling of the HPs for the LFC approaches. Here, we note the distinction between HPO, which involves systematically searching for the best HPC, and early stopping, a regularization technique that halts training once validation performance deteriorates to prevent overfitting. If the HPO is only done for each dataset, e.g., to select the basic architecture and optimizer parameters, we set a check mark at ``per dataset''. If the HPO is only performed to select the specific HPs of an individual approach over multiple datasets, e.g., the best value for a regularization term, we set a check mark at ``per approach''. If HPO is performed for each dataset and approach, we set a check mark at ``per dataset'' and ``per approach''. If no HPO is performed, we set a cross for both columns. We also mark if noisy validation labels from crowdworkers are used for the HPO or if access to a validation set with true labels is assumed.
For those studies without any HPO, some experimentation relies on standard architectures with default HPs across their study~\citep{tanno2019learning, zhang2024coupled, li2024transferring, nguyen2024noisy, han2024collaborative, guo2024learning}. In contrast, others specify the HPs for each dataset and approach without further explanation~\citep{cao2018maxmig, li2021trustable, wei2022deep}. Several studies~\citep{tanno2019learning, herde2023multi, herde2024annot, zhang2024coupled, nguyen2024noisy, guo2024learning} provide an extra ablation study for the HPs of their own LFC approaches. 

\paragraph{Validation with Noisy Class Labels} A few works exist on different aspects of validation with noisy class labels in the broader area of LNL. \cite{chen2021robustness} theoretically prove that for diagonally-dominant confusion matrices, the validation accuracy remains a reliable indicator of true performance. However, in practice, complex types of noise can still pose challenges, especially when the noise is systematic or when not enough data are available to average it out. For example, the empirical findings of \cite{kuo2023noisy} indicate that even small amounts of (not necessarily label) noise in the validation signal can significantly degrade HPO outcomes. The observations of \cite{inouye2017hyperparameter} also confirm that standard validation can be misleading for localized, systematic label noise. Their proposed solution injects synthetic label noise into the training data (based on an estimated noise model) while keeping validation labels unchanged. This penalizes models that overfit spurious patterns and improves over standard cross-validation. \cite{guo2024learning} evaluate LFC approaches with early stopping using noisy validation data. However, no analysis regarding the effects of such an early stopping is reported. \cite{yuan2024early} also recognizes the issues of training and validating with noisy class labels in the context of early stopping. Therefore, they propose a solution for early stopping without relying on a separate validation set. However, they do not perform any HPO but focus on demonstrating their solution's robustness across different HPCs. In contrast, \cite{wang2025realistic} tackle the issue of HPO by proposing HPS criteria when learning from partial labels.

\section{Hyperparameter Optimization with Noisy Labels from Crowds}
\label{sec:approach}
This section first formalizes the problem setting and approaches to LFC, then outlines the basics of HPO, and finally introduces corresponding HPS criteria for handling noisy crowd-labeled validation data.

\subsection{Problem Setting}
\label{sec:problem-setting}
Here, we describe the data generation process to define the objective of LFC approaches.

\paragraph{Data Generation Process} Figure~\ref{fig:probabilistic-graphical-model} depicts the probabilistic graphical model of the commonly assumed data generation process in LFC settings~\citep{li2022beyond,herde2024annot}. Let the multiset \mbox{$\mathcal{X} \defas \{\boldsymbol{{x}}_n\}_{n=1}^N \subset \Omega_{X}, N \in \mathbb{N}_{\geq 1}$} denote the observed instances, which are independently drawn from $\Pr(\boldsymbol{x})$. Then, their one-hot encoded true class labels, denoted as the multiset $\mathcal{Y} \defas \{\boldsymbol{{y}}_n\}_{n=1}^{N} \subseteq \Omega_Y \defas \{\boldsymbol{e}_c\}_{c=1}^{C}, C \in \mathbb{N}_{\geq 2}$, are distributed according to $\Pr(\boldsymbol{y} \mid \boldsymbol{{x}}_n)$ and latent. Only the multiset $\mathcal{Z} \defas \{\boldsymbol{{z}}_{nm}\}_{n=1, m=1}^{N, M} \subseteq \Omega_Z \defas \Omega_Y \cup \{\boldsymbol{0}\}$ of one-hot encoded noisy class labels provided by $M \in \mathbb{N}_{\geq 2}$ independent crowdworkers is observable. Since not every crowdworker is requested to label each instance, some class labels from the crowdworkers are unobserved, denoted as an all-zero vector~$\boldsymbol{0}$. An observed class label is assumed to be drawn from the instance- and crowdworker-specific distribution $\Pr(\boldsymbol{z}_{nm} \mid \boldsymbol{{x}}_n, \boldsymbol{{y}}_n)$. 

\begin{figure}[!h]
    \centering
          

    \includegraphics[width=\textwidth]{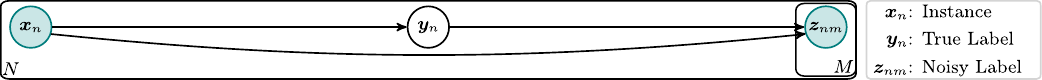}
    \caption{\textbf{Probabilistic graphical model of LFC.} Arrows show dependencies between random variables, while shaded circles indicate observed variables and unshaded latent ones.}
    \label{fig:probabilistic-graphical-model}
\end{figure}

\paragraph{Objective} LFC approaches aim to optimize the parameters\footnote{In later sections, $\boldsymbol{\theta}$ also encompasses other models' parameters of an LFC approach for ease of notation.} $\boldsymbol{\theta} \in \Omega_\Theta$ of a data classification model $\boldsymbol{f}_{\boldsymbol{\theta}}: \Omega_X \rightarrow \Delta_{C}$ by minimizing its expected risk:
\begin{equation}
    \label{eq:expected-loss}
    \boldsymbol{\theta}^\star \defas  \argmin_{\boldsymbol{\theta} \in \Omega_\Theta} \left(\mathbb{E}_{\Pr(\boldsymbol{x}, \boldsymbol{y})}\left[L\left(\boldsymbol{y}, \boldsymbol{f}_{\boldsymbol{\theta}}(\boldsymbol{x})\right)\right]\right),
\end{equation}
where $\Delta_{C}$ is a probability simplex and $L: \Delta_{C} \times \Delta_{C} \rightarrow \mathbb{R}$ denotes an appropriate loss function. Throughout this article, we employ the zero-one loss~\citep{vapnik1995the} to assess the data classification model's predictions:\footnote{The dot product of two one-hot encoded label vectors is one if and only if they represent the same class.}
\begin{align}
    \label{eq:zero-one-loss}
    L_{0/1}\left(\boldsymbol{y}, \boldsymbol{\hat{y}}\right) \defas 1 -  \left(\argmax_{\boldsymbol{e}_c \in \Omega_Y} \left(\boldsymbol{e}_c^\mathrm{T}\boldsymbol{y}\right)\right)^\mathrm{T}\left(\argmax_{\boldsymbol{e}_c \in \Omega_Y} \left(\boldsymbol{e}_c^\mathrm{T}\boldsymbol{\hat{y}}\right)\right).
\end{align}

\subsection{Approaches to Learning from Crowds}
\label{sec:approaches-to-learning-from-crowds}
Given the objective in Eq.~\eqref{eq:expected-loss}, LFC approaches do not directly optimize the outputs of the data classification model $\boldsymbol{f}_{\boldsymbol{\theta}}$ due to the lack of true labels $\mathcal{Y}$. Instead, the noisy class labels $\mathcal{Z}$ are used to train a crowdworker classification model $\boldsymbol{g}_{\boldsymbol{\theta}}: \Omega_X \times [M] \rightarrow \Delta_C$ with $[M] \defas \{1, \dots, M\}$. This model predicts the probability distribution over all class labels for each instance-crowdworker pair, where each label's value indicates the probability that the given crowdworker will assign that label to the given instance. The estimates of both classification models are typically linked through transformations based on confusion matrices~\citep{tanno2019learning} or noise adaptation layers~\citep{rodrigues2018deep}, which try to separate the crowdworkers' noise from the true class label distribution. Furthermore, such a noise separation allows defining a crowdworker performance model $h_{\boldsymbol{\theta}}: \Omega_X \times [M] \rightarrow [0, 1]$ quantifying crowdworkers' labeling accuracies. These three different models' predictions have the following probabilistic interpretations:
\begin{align}
    \label{eq:true-label-probabilities}
    \left[\boldsymbol{f}_{\boldsymbol{\theta}}(\boldsymbol{x}_n)\right]_{c} &\defas \Pr(\boldsymbol{y}_n = \boldsymbol{e}_{c} \mid \boldsymbol{x}_n, \boldsymbol{\theta}) \\
    \label{eq:noisy-label-probabilities}
    \left[\boldsymbol{g}_{\boldsymbol{\theta}}(\boldsymbol{x}_n, m)\right]_{c} &\defas \Pr(\boldsymbol{z}_{nm} = \boldsymbol{e}_{c} \mid \boldsymbol{x}_n, \boldsymbol{\theta}),\\
    \label{eq:labeling-accuracy}
    h_{\boldsymbol{\theta}}(\boldsymbol{x}_n, m) &\defas \Pr(\boldsymbol{z}_{nm}^\mathrm{T} \boldsymbol{y}_n = 1 \mid \boldsymbol{x}_n, \boldsymbol{\theta}),
\end{align}
where $[\cdot]_c$ denotes the $c$-th element of a vector. We treat all three models as black-box functions throughout the main part of this article and refer to Appendix~\ref{app:inference-overview-lfc-approaches} for an overview of the LFC approaches' concrete implementations to infer the quantities in Eqs.~\eqref{eq:true-label-probabilities}--\eqref{eq:labeling-accuracy}.

\subsection{Hyperparameter Optimization}
\label{subsec:hpo}
Given the dataset $\mathcal{D} \defas \{(\boldsymbol{x}_n, \mathcal{Z}_n)\}_{n=1}^N$ with $\mathcal{Z}_n = \{\boldsymbol{z}_{nm} \mid \boldsymbol{z}_{nm} \neq \boldsymbol{0}\}_{m=1}^{M}$ encompassing only an instance's observed class labels, a learning algorithm $\boldsymbol{A}_{\boldsymbol{\lambda}}$ (corresponding to an LFC approach) with the HPC $\boldsymbol{\lambda} \in \Omega_{\Lambda}$ outputs the optimized models' parameters, denoted as $\boldsymbol{A}_{\boldsymbol{\lambda}}(\mathcal{D}) \in \Omega_\Theta$. Each dimension in the search space $\Omega_\Lambda$ corresponds to a single HP, e.g., the number of epochs (integer), the learning rate (continuous), or the type of the optimizer (categorical). Ideally, we find the optimal HPC $\boldsymbol{\lambda^\star} \in \Omega_\Lambda$ such that our learning algorithm outputs the optimal classification model parameters (see Eq.~\eqref{eq:expected-loss}): 
\begin{equation}
    \label{eq:optimal-hpc}
    \boldsymbol{A}_{\boldsymbol{\lambda^\star}}(\mathcal{D}) = \boldsymbol{\theta}^{\star}.
\end{equation}
In practice, finding the optimal solution is difficult due to many challenges, of which two of the most critical are the following:
\begin{enumerate}[leftmargin=1.5em]
    \item[\textbf{\color{nocolor}\circled{1}}] Evaluating each HPC $\boldsymbol{\lambda} \in \Omega_\Lambda$ is computationally infeasible for a large HP search space~$\Omega_\Lambda$.
    \item[\textbf{\color{yescolor}\circled{2}}] We can only estimate the expected risk (see Eq.~\eqref{eq:expected-loss}) because $\Pr(\boldsymbol{x}, \boldsymbol{y})$ is unknown.
\end{enumerate}
In this article, we focus exclusively on challenge~{\color{yescolor}\textbf{\circled{2}}} because the risk estimation is difficult and underexplored with only access to crowd-labeled validation data (see Section~\ref{sec:related-work}). Challenge~{\color{nocolor}\textbf{\circled{1}}} is not part of our contributions. Instead, we briefly review established solutions so the two challenges provide the context for the HPO loop shown in Figure \ref{fig:hyperparameter-optimization-loop}.
\begin{figure}[!h]
    \centering
    \includegraphics[width=\textwidth]{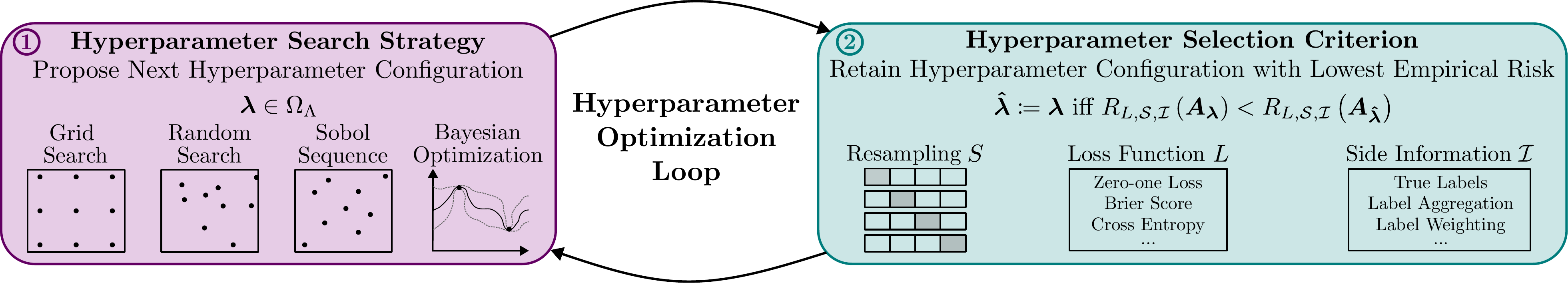}
    \caption{\textbf{HPO loop.} In an iterative process, HPO techniques explore the HP search space by retraining and evaluating the learning algorithm with different HPCs.}
    \label{fig:hyperparameter-optimization-loop}
\end{figure}

\paragraph{\color{nocolor} \protect\circled{1} Hyperparameter Search Strategy} In supervised learning with access to true labels, research focuses primarily on improving the search strategy for (iteratively) proposing a set of candidate HPCs $\Lambda \subset \Omega_\Lambda$ by balancing the exploration-exploitation trade-off within the HP search space~$\Omega_\Lambda$ given a budget of $|\Lambda| \in \mathbb{N}_{\geq 0}$ evaluated HPCs. For this purpose, random search is a popular choice that samples HP values randomly from predefined ranges, often outperforming exhaustive grid search in high-dimensional spaces~\citep{bergstra2012random}. Meanwhile, Sobol sequences~\citep{sobol1998quasi} and Bayesian optimization~\citep{wang2023recent} guide the search of candidate HPCs even more efficiently. Bayesian optimization usually excels because it adaptively selects each new HPC. However, we employ Sobol sequences so that every HPS criterion evaluates the same set $\Lambda$ of candidate HPCs. Hence, any performance differences between the HPS criteria come solely from choosing different HPCs.

\paragraph{\color{yescolor} \protect\circled{2}  Hyperparameter Selection Criterion} Ideally, an HP search strategy has access to a reliable empirical risk estimation~\citep{vapnik1995the}, which assigns a learning algorithm a scalar value $R_{L, \mathcal{S}, \mathcal{I}}(\boldsymbol{A}_{\boldsymbol{\lambda}}) \in \mathbb{R}$. Thereby, $L$ denotes the loss function (see Eq.~\eqref{eq:zero-one-loss}), $S$ a resampling technique, and $\mathcal{I}$ additional side information. Formally, we represent a resampling technique, e.g., hold-out, cross-validation, or bootstrapping, through a set of $K \in \mathbb{N}_{\geq 1}$ disjoint training ($\mathcal{T}$) and validation ($\mathcal{V}$) splits of the full training set~$\mathcal{D}$:
\begin{equation}
    \mathcal{S} \defas \{(\mathcal{T}_k, \mathcal{V}_k) \mid \mathcal{T}_k \cup \mathcal{V}_k = \mathcal{D} \wedge \mathcal{T}_k \cap \mathcal{V}_k = \emptyset\}_{k=1}^K.
\end{equation}
Side information $\mathcal{I}$ encompasses all required inputs beyond the loss function $L$ and resampling technique $S$ for computing the empirical risk. 
Based on such empirical risk estimates, we define an HPS criterion as a rule picking the evaluated HPC with the lowest empirical risk:
\begin{equation}
\label{eq:hpc-selection}
    \boldsymbol{\hat{\lambda}} \defas \arg\min_{\boldsymbol{\lambda} \in \Lambda} \left(R_{L, \mathcal{S}, \mathcal{I}}(\boldsymbol{A}_{\boldsymbol{\lambda}})\right).
\end{equation}
For example, suppose the true class labels are obtained from an expert as side information such that $\mathcal{I} \defas \mathcal{Y}$. Then, the true empirical risk of the learning algorithm~$\boldsymbol{A}_{\boldsymbol{\lambda}}$ is computed as:
\begin{equation}
    \label{eq:true-empirical-risk}
    R_{L, \mathcal{S}, \mathcal{Y}}(\boldsymbol{A}_{\boldsymbol{\lambda}}) \defas \sum_{(\mathcal{T}_k, \mathcal{V}_k) \in \mathcal{S}} \sum_{(\boldsymbol{x}_n, \mathcal{Z}_n) \in \mathcal{V}_k} \frac{1}{K \cdot |\mathcal{V}_k|}L\left(\boldsymbol{y}_n, \boldsymbol{f}_{\boldsymbol{A}_{\boldsymbol{\lambda}}(\mathcal{T}_k)}(\boldsymbol{x}_n)\right).
\end{equation}
Since we have only access to noisy crowd-labeled validation data in an LFC setting, the HPS criterion based on the true empirical risk $R_{L, \mathcal{S}, \mathcal{Y}}$ (see Eq.~\eqref{eq:true-empirical-risk}) represents our upper baseline criterion for HPO, denoted as \textsc{true}. In contrast, our lower baseline criterion \textsc{def} constantly outputs a default HPC $\boldsymbol{\lambda}_\textsc{def} \in \Omega_{\Lambda}$, which corresponds to a naive risk estimation using the default HPC as side information such that $\mathcal{I} \defas \boldsymbol{\lambda}_{\textsc{def}}$:
\begin{equation}
    \label{eq:naive-empirical-risk}
    R_{L, \mathcal{S}, \boldsymbol{\lambda}_\textsc{def}}(\boldsymbol{A}_{\boldsymbol{\lambda}}) \defas \delta(\boldsymbol{\lambda}_\textsc{def} \neq \boldsymbol{\lambda}),
\end{equation}
where $\delta: \{\mathrm{false}, \mathrm{true}\} \rightarrow \{0, 1\}$ denotes an indicator function.

\subsection{Hyperparameter Selection Criteria for Crowd-labeled Validation Data}
\label{subsec:hps-criteria}
Given crowd-labeled validation data, we introduce multiple HPS criteria that differ in their empirical risk estimation type, assumptions about crowdworker performances, and label weighting. Figure~\ref{fig:model-selection-criteria-overview}~provides a corresponding overview of a hierarchical framework, whose four levels are detailed in the following.

\begin{figure}[!h]
    \centering
    \includegraphics[width=\linewidth]{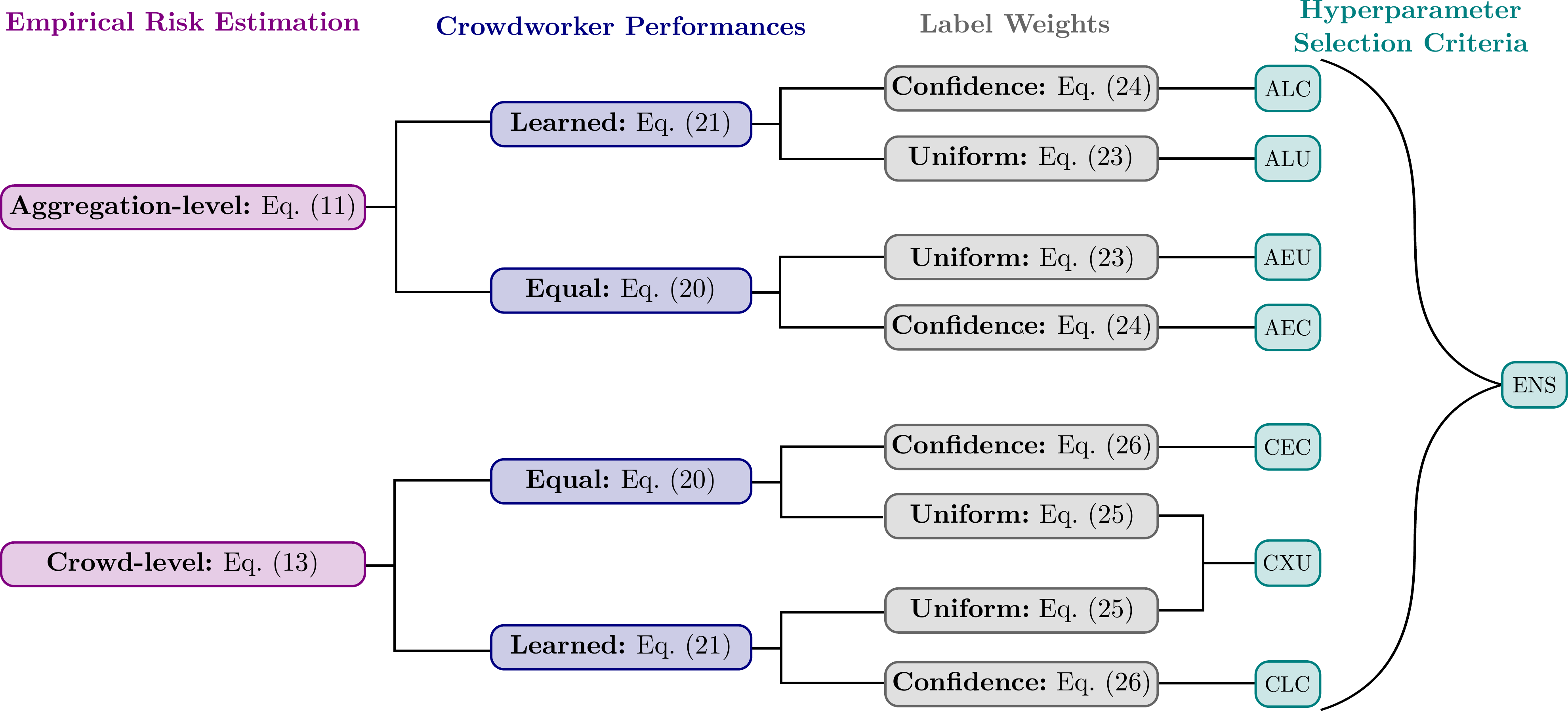}
    \caption{
        \textbf{HPS criteria for noisy crowd-labeled data.} The two trees illustrate two risk estimation types. Every leaf marks the concrete criterion reached by the respective path, and the criterion's name is simply the abbreviation obtained by concatenating the first letters of the nodes along that path. The abbreviation \textsc{c\underline{x}u} indicates that \textsc{c\underline{e}u} and \textsc{c\underline{l}u} are identical. The ensemble-based criterion \textsc{ens} combines all criteria for improved robustness.
    }
    \label{fig:model-selection-criteria-overview}
\end{figure}

\paragraph{{\color{nocolor}Empirical Risk Estimation Type}} LFC approaches typically involve the joint training of the data classification model~$\boldsymbol{f}_{\boldsymbol{\theta}}$ and the crowdworker classification model~$\boldsymbol{g}_{\boldsymbol{\theta}}$ (see Section~\ref{sec:approaches-to-learning-from-crowds}), whose predictive performances capture distinct facets of the training outcome. On the one hand, we assess the data classification model~$\boldsymbol{f}_{\boldsymbol{\theta}}$ by computing the \textbf{aggregation-level} empirical risk through:
\begin{align}
    \label{eq:aggregation-level-risk}
    R_{L, \mathcal{S}, \{\boldsymbol{\overline{z}}, w\}}(\boldsymbol{A}_{\boldsymbol{\lambda}}) &\defas  \sum_{(\mathcal{T}_k, \mathcal{V}_k) \in \mathcal{S}} \sum_{(\boldsymbol{x}_n, \mathcal{Z}_n) \in \mathcal{V}_k} \frac{w(\boldsymbol{x}_n, \mathcal{Z}_n) }{K \cdot W_k} L\left(\boldsymbol{\overline{z}}(\boldsymbol{x}_n, \mathcal{Z}_n), \boldsymbol{f}_{\boldsymbol{A}_{\boldsymbol{\lambda}}(\mathcal{T}_k)}(\boldsymbol{x}_n)\right) \text{ with} \\
    W_k &\defas \sum_{(\boldsymbol{x}_n, \mathcal{Z}_n) \in \mathcal{V}_k} w(\boldsymbol{x}_n, \mathcal{Z}_n),
\end{align}
where the two functions in $\mathcal{I} \defas \{\boldsymbol{\overline{z}}, w\}$ denote the side information. The label aggregation function $\boldsymbol{\overline{z}}: \Omega_X \times \mathcal{P}\left(\Omega_Y\right) \rightarrow \Delta_C$ aims to infer the latent true class labels with $\mathcal{P}\left(\Omega_Y\right)$ referring to the power set of multisets constructed from the class labels in $\Omega_Y$. The label weighting function $w: \Omega_X \times \mathcal{P}\left(\Omega_Y\right) \rightarrow \mathbb{R}_{\geq0}$ weights individual loss contributions of the aggregated class labels. 
On the other hand, we compute the \textbf{crowd-level} empirical risk to assess the crowdworker classification model~$\boldsymbol{g}_{\boldsymbol{\theta}}$ through:
\begin{align}
    \label{eq:crowd-level-risk}
    R_{L, \mathcal{S}, v}(\boldsymbol{A}_{\boldsymbol{\lambda}}) &\defas  \sum_{(\mathcal{T}_k, \mathcal{V}_k) \in \mathcal{S}} \sum_{(\boldsymbol{x}_n, \mathcal{Z}_n) \in \mathcal{V}_k} \sum_{\boldsymbol{z}_{nm} \in \mathcal{Z}_n} \frac{v(\boldsymbol{x}_n, m, \mathcal{Z}_n)}{K \cdot V_k} L\left(\boldsymbol{z}_{nm}, \boldsymbol{g}_{\boldsymbol{A}_{\boldsymbol{\lambda}}(\mathcal{T}_k)}(\boldsymbol{x}_n, m)\right) \text{ with} \\
    V_k &\defas \sum_{(\boldsymbol{x}_n, \mathcal{Z}_n) \in \mathcal{V}_k} \sum_{\boldsymbol{z}_{nm} \in \mathcal{Z}_n} v(\boldsymbol{x}_n, m, \mathcal{Z}_n), 
\end{align}
where the label weighting function $v: \Omega_X \times [M] \times \mathcal{P}\left(\Omega_Y\right) \rightarrow \mathbb{R}_{\geq 0}$ as side information $\mathcal{I} \defas v$ weights the individual loss contributions of the crowdworkers' labels. Intuitively, both label weighting functions are to downweight the influence of labels that are likely false, allowing the more reliable ones to dominate the loss and thereby enhancing robustness to label noise.

\paragraph{{\color{navy_accent}Crowdworker Performance}} Starting from different assumptions about crowdworker performances, we estimate true class label probabilities as a basis for defining the label aggregation function $\boldsymbol{\overline{z}}$ and the label weighting functions $w, v$. For this purpose, let us assume that we have the estimated confusion probabilities $\widehat{\Pr}(\boldsymbol{z}_{nm} \mid \boldsymbol{x}_n, \boldsymbol{y}_n)$ for each crowdworker $m \in [M]$ and instance $\boldsymbol{x}_n \in \mathcal{X}$. Then, we estimate the posterior true class label probabilities, i.e., after observing the crowdworkers' labels $\mathcal{Z}_n$, through:
\begin{align}
    \widehat{\Pr}(\boldsymbol{y}_n = \boldsymbol{e}_c \mid \boldsymbol{x}_n, \mathcal{Z}_n) 
    &\stackrel{(\star)}{\propto} \widehat{\Pr}\left(\boldsymbol{y}_n=\boldsymbol{e}_c \mid \boldsymbol{x}_n\right) \widehat{\Pr}(\mathcal{Z}_n \mid \boldsymbol{x}_n, \boldsymbol{y}_n = \boldsymbol{e}_c) \label{eq:bayes-theorem}\\
    &\stackrel{(\dagger)}{=}  \widehat{\Pr}\left(\boldsymbol{y}_n=\boldsymbol{e}_c \mid \boldsymbol{x}_n\right) \prod_{\boldsymbol{z}_{nm} \in \mathcal{Z}_n} \widehat{\Pr}(\boldsymbol{z}_{nm} \mid \boldsymbol{x}_n, \boldsymbol{y}_n = \boldsymbol{e}_c)\label{eq:independence},
\end{align}
where the transformation $(\star)$ corresponds to Bayes' theorem and $(\dagger)$ to the assumed crowdworkers' independence. If we employed the data classification model~$\boldsymbol{f}_{\boldsymbol{\theta}}$ to predict $\widehat{\Pr}\left(\boldsymbol{y}_n \mid \boldsymbol{x}_n\right)$, we would leak its information into the posterior true class label probability estimation (see Proposition~\ref{prop:uniform-class-probabilities} in Appendix~\ref{subapp:prior-class-probabilities}), which can be interpreted as a form of double-dipping~\citep{ball2020double}, making associated risk estimates over-optimistic. Instead, we assume uniform prior class probabilities, i.e., before observing any crowdworkers' labels~$\mathcal{Z}_n$: 
\begin{equation}
    \label{eq:simplified-class-probabilities}
    \widehat{\Pr}\left(\boldsymbol{y}_n = \boldsymbol{e}_c \mid \boldsymbol{x}_n\right) \defas \frac{1}{C}.
\end{equation}
If we employed the full crowdworkers' confusion probability estimates $\widehat{\Pr}(\boldsymbol{z}_{nm} \mid \boldsymbol{x}_n, \boldsymbol{\hat{y}}_n)$, the posterior computation would be vulnerable to class-specific biases (see Proposition~\ref{prop:crowdworker-confusion-probabilities} in Appendix~\ref{subapp:crowdworker-probabilities}). Hence, we resort to a Bernoulli model requiring only each crowdworker's instance-wise performance estimate $\widehat{\Pr}(\boldsymbol{z}_{nm}^\mathrm{T}\boldsymbol{y}_n = 1 \mid \boldsymbol{x}_n)$ such that:
\begin{gather}
    \label{eq:simplified-performance}
    \widehat{\Pr}(\boldsymbol{z}_{nm} = \boldsymbol{e}_k \mid \boldsymbol{x}_n, \boldsymbol{y}_n = \boldsymbol{e}_c) \defas \widehat{\Pr}(\boldsymbol{z}_{nm}^\mathrm{T}\boldsymbol{y}_n = 1 \mid \boldsymbol{x}_n)^{\boldsymbol{e}_k^\mathrm{T}\boldsymbol{e}_c} \left(\frac{\widehat{\Pr}(\boldsymbol{z}_{nm}^\mathrm{T}\boldsymbol{y}_n = 0 \mid \boldsymbol{x}_n)}{C-1}\right)^{1 - \boldsymbol{e}_k^\mathrm{T}\boldsymbol{e}_c}.
\end{gather}
By summarizing each crowdworker's behavior on a given instance with one scalar performance value, the Bernoulli model applies the same performance value to every class. All classes are, therefore, treated identically, and no systematic bias toward any particular class arises. As a result of Eq.~\eqref{eq:simplified-class-probabilities} and Eq.~\eqref{eq:simplified-performance}, the posterior estimation from Eq.~\eqref{eq:independence} reduces to:
\begin{align}
    \label{eq:simplified-class-posteriors}
    \widehat{\Pr}(\boldsymbol{y}_n = \boldsymbol{e}_c \mid \boldsymbol{x}_n, \mathcal{Z}_n) \propto \hspace{-0.5em}\prod_{\boldsymbol{z}_{nm} \in \mathcal{Z}_n} \hspace{-0.5em}   
    \widehat{\Pr}(\boldsymbol{z}_{nm}^\mathrm{T}\boldsymbol{y}_n = 1 \mid \boldsymbol{x}_n)^{\boldsymbol{z}_{nm}^\mathrm{T} \boldsymbol{e}_c}
    \left(\frac{\widehat{\Pr}(\boldsymbol{z}_{nm}^\mathrm{T}\boldsymbol{y}_n = 0 \mid \boldsymbol{x}_n)}{C-1}\right)^{1-\boldsymbol{z}_{nm}^\mathrm{T}\boldsymbol{e}_c},
\end{align}
where we distinguish between:
\begin{align}
    \label{eq:uniform-crowdworker-performances}
    \widehat{\Pr}(\boldsymbol{z}_{nm}^\mathrm{T}\boldsymbol{y}_n = 1 \mid \boldsymbol{x}_n) &\defas p \in (\nicefrac{1}{C}, 1] &\text{as }\textbf{equal } \text{crowdworker performances,} \\
    \widehat{\Pr}(\boldsymbol{z}_{nm}^\mathrm{T}\boldsymbol{y}_n = 1 \mid \boldsymbol{x}_n) &\defas h_{\boldsymbol{A}_{\boldsymbol{\lambda}}(\mathcal{T}_k)}(\boldsymbol{x}_n, m) &\text{as }\textbf{learned } \text{crowdworker performances.}
\end{align}
In the first case, the exact value of the labeling accuracy $p$ is irrelevant. Instead, $p$ only encodes the assumption that each crowdworker has the same performance across all instances, which is better than randomly guessing. In the second case, the instance-wise performances are estimated by the crowdworker performance model $h_{\boldsymbol{A}_{\boldsymbol{\lambda}}(\mathcal{T}_k)}$ obtained after training with the respective LFC approach $\boldsymbol{A}_{\boldsymbol{\lambda}}$ on the $k$-th training fold using the candidate HPC $\boldsymbol{\lambda}$. In both cases, the label aggregation function $\boldsymbol{\overline{z}}$ outputs the \textit{maximum a posteriori} (MAP) estimate of the true class label:
\begin{equation}
    \label{eq:aggregation-function}
    \boldsymbol{\overline{z}}(\boldsymbol{x}_n, \mathcal{Z}_n) \defas \argmax_{\boldsymbol{e}_c \in \Omega_Y} \left(\widehat{\Pr}\left(\boldsymbol{y}_n = \boldsymbol{e}_c \mid \boldsymbol{x}_n, \mathcal{Z}_n\right)\right).
\end{equation}
When all crowd workers are assumed to perform equally, the MAP estimate reduces to simple majority voting. In contrast, it naturally becomes weighted majority voting once their performances are learned (see Proposition~\ref{prop:weighted-majority-voting} in Appendix~\ref{subapp:weighted-majority-vote}).

\paragraph{{\color{partialcolor}Label Weights}} The label weighting functions control a class label's impact on the risk estimate. A uniform weighting gives every label the same weight, whereas confidence weighting scales a label's weight in proportion to its estimated correctness probability. Accordingly, the label weighting functions for aggregation-level risk estimation take the forms:
\begin{align}
    \label{eq:uniform-aggregation-weights}
    w(\boldsymbol{x}_n, \mathcal{Z}_n) &\defas 1 & \text{as }\textbf{uniform } \text{weighting},  \\
    \label{eq:confidence-aggregation-weights}
    w(\boldsymbol{x}_n, \mathcal{Z}_n) &\defas \max_{\boldsymbol{e}_c \in \Omega_Y} \left(\widehat{\Pr}\left(\boldsymbol{y}_n = \boldsymbol{e}_c \mid \boldsymbol{x}_n, \mathcal{Z}_n\right)\right) &\text{as }\textbf{confidence } \text{weighting},
\end{align}
whereas the label weighting functions for the crowd-level risk take the forms:
\begin{align}
    \label{eq:uniform-crowd-weights}
    v(\boldsymbol{x}_n, m, \mathcal{Z}_n) &\defas 1 & \text{as }\textbf{uniform } \text{weighting},  \\
    \label{eq:confidence-crowd-weights}
    v(\boldsymbol{x}_n, m, \mathcal{Z}_n) &\defas  \widehat{\Pr}\left(\boldsymbol{y}_n = \boldsymbol{z}_{nm} \mid \boldsymbol{x}_n, \mathcal{Z}_n\right) &\text{as }\textbf{confidence } \text{weighting}.
\end{align}

\paragraph{{\color{yescolor}Hyperparameter Selection Criteria}} The different combinations of risk estimation type, crowdworker performance estimation, and label weighting procedure correspond to $J=7$ distinct HPS criteria with their associated risk measures $\mathcal{R} \defas \{R_{L, \mathcal{S}, \mathcal{I}_j}\}_{j=1}^{J}$. For example, the criterion \textsc{aeu} combining Eq.~\eqref{eq:aggregation-level-risk}, Eq.~\eqref{eq:uniform-crowdworker-performances}, and Eq.~\eqref{eq:uniform-aggregation-weights} corresponds to the risk estimation with plain majority vote labels for the validation instances. Each of the empirical risk measures n $\mathcal{R}$ may be different from the true empirical risk (see Eq.~\eqref{eq:true-empirical-risk}):
\begin{equation}
    R_{L, \mathcal{S}, \mathcal{I}_j}(\boldsymbol{A}_{\boldsymbol{\lambda}}) = R_{L,\mathcal{S},\mathcal{Y}}(\boldsymbol{A}_{\boldsymbol{\lambda}}) + \epsilon_{L, \mathcal{S}, \mathcal{Y}, \mathcal{I}_j}\left(\boldsymbol{A}_{\boldsymbol{\lambda}}\right) \text{ with } \epsilon_{L, \mathcal{S}, \mathcal{Y}, \mathcal{I}_j}\left(\boldsymbol{A}_{\boldsymbol{\lambda}}\right) \in \mathbb{R}.
\end{equation}
The differences arise through multiple sources, e.g., imperfect label aggregation, imprecise crowdworker performance estimates, or even different target models. For example, the crowd-level risk estimation measures the risk of the crowdworker classification model, which is different from estimating the risk of the data classification model as in Eq.~\eqref{eq:true-empirical-risk}. Because of these issues, we propose combining our empirical risk measures into an ensemble-based selection criterion \textsc{ens}. For this purpose, we adopt the classic Borda count~\citep{borda1781}, a rank-aggregation rule widely used in robust meta-evaluation~\citep{abdulrahman2018speeding}. Let $\mathcal{O} \defas \{\boldsymbol{o}_j: \Lambda \rightarrow \{1, \dots, \Lambda\}]\}_{j=1}^J$ denote the set of ranking functions such that:
\begin{equation}
    \boldsymbol{o}_j(\boldsymbol{\lambda}) \defas 1 + \sum_{\boldsymbol{\lambda^\prime} \in \Lambda} \delta(R_{L, \mathcal{S}, \mathcal{I}_j}(\boldsymbol{A}_{\boldsymbol{\lambda}}) > R_{L, \mathcal{S}, \mathcal{I}_j}(\boldsymbol{A}_{\boldsymbol{\lambda^\prime}})).
\end{equation}
Accordingly, $\boldsymbol{o}_j(\boldsymbol{\lambda}) = 1$ corresponds to the HPC with the lowest empirical risk, and $\boldsymbol{o}_j(\boldsymbol{\lambda}) = |\Lambda|$ to the one with the highest risk estimate. The empirical risk estimate based on the Borda count is then defined as:
\begin{equation}
    R_{L, \mathcal{S}, \mathcal{O}}(\boldsymbol{A}_{\boldsymbol{\lambda}}) \defas \sum_{\boldsymbol{o}_j \in \mathcal{O}} \boldsymbol{o}_j(\boldsymbol{\lambda}),
\end{equation}
where the ranking functions serve as our side information such that $\mathcal{I} \defas \mathcal{O}$.
Intuitively, summing rankings stabilizes decisions by balancing individual biases of each noisy risk measure in $\mathcal{R}$, increasing the likelihood of choosing a robust HPC.

\section{Experimental Study}
\label{sec:empirical_evaluation}
This section starts with a comprehensive description of our experimental setup. Subsequently, we analyze our experimental results to answer $RQ_1$ and $RQ_2$ as our two central research questions. The corresponding answers serve as a basis for formulating recommendations to design realistic and fair experiments when benchmarking LFC approaches in the future.

\subsection{Experimental Setup}
Our setup covers datasets, neural network architectures, LFC approaches, HP search, and HPS criteria. We describe our design choices for each of these aspects in the following.

\paragraph{Datasets} Realistic datasets are a requirement for a meaningful evaluation of LFC approaches. Therefore, we rely only on real-world datasets annotated by error-prone humans, mostly actual crowdworkers. Table~\ref{tab:datasets} overviews these datasets by detailing their key attributes. 
The dataset \texttt{mgc}~\citep{tzanetakis2002musical} originally contains $\si{30}{s}$ audio files of songs to be classified according to their music genres. 
A subset of the well-known image benchmark dataset \texttt{label-me}~\citep{russell2008labelme} concerns the classification of scenes,
while the dataset \texttt{dopanim}~\citep{herde2024dopanim} targets the classification of doppelganger (groups of highly similar) animals.
There are two text datasets, which are a subset of the dataset~\texttt{reuters}~\citep{reuters19878lewis} for news article classification and a subset of the dataset \texttt{spc}~\citep{pang2005seeing} for sentiment polarity classification of movie reviews. Based on large sets of noisy labels resulting from the datasets' labeling campaigns, we follow the ideas of \cite{wei2021learning} and \cite{herde2024dopanim} by introducing variants of these noisy label sets. 
These variants simulate different levels of crowdworker performance and varying amounts of label redundancy. For each instance, we either retain only the labels produced by the \texttt{worst} crowdworkers, i.e., we keep false labels if any are available, or we select labels uniformly at \texttt{rand}om. The suffixes \texttt{-1}, \texttt{-2}, and \texttt{-v} then specify how many labels to keep per instance: exactly one, exactly two, or a variable number, respectively. Because we preserve only a subset of all submitted labels, the total number of crowdworkers contributing labels can change across variants. The variant \texttt{full} refers to the originally published set of class labels from crowdworkers. Together, these dataset variants cover a wide range of different LFC settings. Concretely, the number of crowdworkers ranges from small groups of $M=20$ people to a large group of $M=203$ people. Label noise, defined as the fraction of erroneous labels from crowdworkers, ranges from approximately \SIrange{20}{87}{\percent}. When these labels are aggregated via majority voting, the resulting aggregation noise, i.e., the fraction of aggregated labels that are incorrect, ranges from circa \SIrange{11}{87}{\percent}. Finally, the dataset variants encompass scenarios ranging from no label redundancy, i.e., only one class label per instance, to those exhibiting substantial label redundancy, i.e., over five class labels per instance. 

\begin{table}[!t]
    \centering
    \scriptsize
    \setlength{\tabcolsep}{0pt}
    \def\arraystretch{1.0}
    \caption{
        \textbf{Dataset overview.} The first column indicates the names of the datasets, while the remaining columns refer to the datasets' attributes. We denote counts by the $\#$ symbol, fractions by the $\%$ symbol, and means are supplemented by standard deviations.
    }
    \label{tab:datasets}
    \begin{tabular*}{\linewidth}{@{\extracolsep{\fill}} lccccccccc }
        \toprule
        \multirow{2}{*}{\textBF{Dataset}} & \multirow{2}{*}{\textBF{Variant}} & \textBF{Labeling} & \textBF{Training} & \textBF{Test} & \multirow{2}{*}{\textBF{Classes}} & \multirow{2}{*}{\textBF{Workers}} & \textBF{Labels per} & \textBF{Label} & \textBF{Aggregation} \\ 
        &  & \textBF{Campaign} & \textBF{Instances} & \textBF{Instances} &  &  & \textBF{Instance} &  \textBF{Noise} &  \textBF{Noise} \\
        & & & \textBF{[$\#$]} & \textBF{[$\#$]} & \textBF{[$\#$]} & \textBF{[$\#$]} & \textBF{[$\#$]} &  \textBF{[$\%$]} &  \textBF{[$\%$]} \\
        \hline
        \rowcolor{yescolor!10} \multicolumn{10}{c}{Audio Data}\\
        \hline
        \multirow{7}{*}{\texttt{mgc}}      
        & \texttt{worst-1}
        & \multirow{7}{*}{\citeauthor{rodrigues2013learning}}
        & \multirow{7}{*}{$700$}
        & \multirow{7}{*}{$300$}
        & \multirow{7}{*}{$10$}
        & $32$
        & $1.0_{\pm 0.0}$
        & $87.4$ 
        & $87.4$ \\
        &
        \texttt{worst-2}      
        & 
        & 
        & 
        & 
        & $37$
        & $1.9_{\pm 0.3}$
        & $72.5$ 
        & $69.4$ \\
        &
        \texttt{worst-v}      
        & 
        & 
        & 
        & 
        & $42$
        & $2.5_{\pm 1.6}$
        & $59.2$ 
        & $58.6$ \\
        &
        \texttt{rand-1}      
        & 
        & 
        & 
        & 
        & $37$
        & $1.0_{\pm 0.0}$
        & $47.1$ 
        & $47.1$ \\
        &
        \texttt{rand-2}      
        & 
        & 
        & 
        & 
        & $43$
        & $1.9_{\pm 0.3}$
        & $45.7$ 
        & $43.9$ \\
        &
        \texttt{rand-v}      
        & 
        & 
        & 
        & 
        & $43$
        & $2.6_{\pm 1.6}$
        & $44.6$ 
        & $38.3$ \\
        &
        \texttt{full}      
        & 
        & 
        & 
        & 
        & $44$
        & $4.2_{\pm 2.0}$
        & $44.0$ 
        & $30.3$ \\
        \hline
        \rowcolor{yescolor!10} \multicolumn{10}{c}{Image Data}\\
        \hline
        \multirow{7}{*}{\texttt{label-me}}
        & \texttt{worst-1}
        & \multirow{7}{*}{\citeauthor{rodrigues2017learning}}
        & \multirow{7}{*}{$1{,}000$}
        & \multirow{7}{*}{$1{,}188$}
        & \multirow{7}{*}{$8$}
        & $57$
        & $2.5_{\pm 0.6}$
        & $41.1$
        & $41.1$ \\
        &
        \texttt{worst-2}  
        & 
        & 
        & 
        & 
        & $59$
        & $2.0_{\pm 0.2}$
        & $30.8$ 
        & $30.1$ \\
        &
        \texttt{worst-v}  
        & 
        & 
        & 
        & 
        & $59$
        & $1.8_{\pm 0.8}$
        & $31.6$ 
        & $32.5$ \\
        &
        \texttt{rand-1}  
        & 
        & 
        & 
        & 
        & $57$
        & $1.0_{\pm 0.0}$
        & $23.9$ 
        & $23.9$ \\
        &
        \texttt{rand-2}  
        & 
        & 
        & 
        & 
        & $59$
        & $2.0_{\pm 0.2}$
        & $25.5$ 
        & $25.7$ \\
        &
        \texttt{rand-v}  
        & 
        & 
        & 
        & 
        & $59$
        & $1.8_{\pm 0.8}$
        & $25.5$ 
        & $25.0$ \\
        &
        \texttt{full}  
        & 
        & 
        & 
        & 
        & $59$
        & $2.5_{\pm 0.6}$
        & $26.0$ 
        & $23.7$ \\
        \hline
        \multirow{7}{*}{\texttt{dopanim}}
        & \texttt{worst-1}
        & \multirow{7}{*}{\citeauthor{herde2024dopanim}}
        & \multirow{7}{*}{$10{,}484$}
        & \multirow{7}{*}{$4{,}500$}
        & \multirow{7}{*}{$15$}
        & \multirow{7}{*}{$20$}
        & $1.0_{\pm 0.0}$
        & $77.6$ 
        & $77.6$ \\
        &
        \texttt{worst-2}  
        & 
        & 
        & 
        & 
        & 
        & $2.0_{\pm 0.0}$
        & $62.7$
        & $62.2$ \\
        &
        \texttt{worst-v}  
        & 
        & 
        & 
        & 
        & 
        & $3.0_{\pm 1.4}$
        & $45.2$
        & $46.9$ \\
        &
        \texttt{rand-1}  
        & 
        & 
        & 
        & 
        & 
        & $1.0_{\pm 0.0}$
        & $32.5$ 
        & $32.5$ \\
        &
        \texttt{rand-2}  
        & 
        & 
        & 
        & 
        & 
        & $2.0_{\pm 0.0}$
        & $32.8$
        & $33.2$ \\
        &
        \texttt{rand-v}  
        & 
        & 
        & 
        & 
        & 
        & $3.0_{\pm 1.4}$
        & $32.7$
        & $26.3$ \\
        &
        \texttt{full}  
        & 
        & 
        & 
        & 
        & 
        & $5.0_{\pm 0.2}$
        & $32.7$ 
        & $19.3$ \\
        \hline
        \rowcolor{yescolor!10} \multicolumn{10}{c}{Text Data}\\
        \hline
        \multirow{7}{*}{\texttt{reuters}}
        & \texttt{worst-1}      
        & \multirow{7}{*}{\citeauthor{rodrigues2017learning}}
        & \multirow{7}{*}{$1{,}786$}
        & \multirow{7}{*}{$4{,}217$}
        & \multirow{7}{*}{$8$}
        & \multirow{7}{*}{$38$}
        & $1.0_{\pm 0.0}$
        & $69.2$ 
        & $69.2$ \\
        &
        \texttt{worst-2}      
        & 
        & 
        & 
        & 
        & 
        & $2.0_{\pm 0.2}$
        & $54.0$ 
        & $54.0$ \\
        &
        \texttt{worst-v}      
        & 
        & 
        & 
        & 
        & 
        & $2.0_{\pm 1.0}$
        & $50.8$ 
        & $51.8$ \\
        &
        \texttt{rand-1}      
        & 
        & 
        & 
        & 
        & 
        & $1.0_{\pm 0.0}$
        & $38.5$
        & $38.5$ \\
        &
       \texttt{rand-2}      
        & 
        & 
        & 
        & 
        & 
        & $2.0_{\pm 0.2}$
        & $39.9$ 
        & $40.9$ \\
        &
        \texttt{rand-v}      
        & 
        & 
        & 
        & 
        & 
        & $2.0_{\pm 1.0}$
        & $40.8$ 
        & $38.4$ \\
        &
        \texttt{full}      
        & 
        & 
        & 
        & 
        & 
        & $3.0_{\pm 1.0}$
        & $40.4$ 
        & $35.5$ \\
        \hline
        \multirow{7}{*}{\texttt{spc}}
        & \texttt{worst-1}
        & \multirow{7}{*}{\citeauthor{rodrigues2013learning}}
        & \multirow{7}{*}{$3{,}000$}
        & \multirow{7}{*}{$1{,}999$}
        & \multirow{7}{*}{$2$}
        & $185$
        & $1.0_{\pm 0.0}$
        & $63.4$ 
        & $63.4$ \\
        &
        \texttt{worst-2}
        & 
        & 
        & 
        & 
        & $199$
        & $2.0_{\pm 0.0}$
        & $47.1$ 
        & $47.0$ \\
        &
        \texttt{worst-v}
        & 
        & 
        & 
        & 
        & $202$
        & $3.2_{\pm 1.6}$
        & $31.6$ 
        & $32.5$ \\
        &
        \texttt{rand-1}
        & 
        & 
        & 
        & 
        & $184$
        & $1.0_{\pm 1.0}$
        & $21.2$ 
        & $21.2$ \\
        &
        \texttt{rand-2}
        & 
        & 
        & 
        & 
        & $200$
        & $2.0_{\pm 0.0}$
        & $20.8$ 
        & $20.6$ \\
        &
        \texttt{rand-v}
        & 
        & 
        & 
        & 
        & $202$
        & $3.3_{\pm 1.6}$
        & $21.1$ 
        & $14.9$ \\
        &
        \texttt{full}
        & 
        & 
        & 
        & 
        & $203$
        & $5.5_{\pm 0.7}$
        & $20.9$ 
        & $11.0$ \\
        \bottomrule
    \end{tabular*}
\end{table}

\paragraph{Neural Network Architectures} The original audio files are unavailable for the crowd-labeled dataset \texttt{mgc}. Instead, \cite{rodrigues2013learning} published features extracted via a music information retrieval tool. Similarly, only term counts published by \cite{rodrigues2017learning} are available for the crowd-labeled dataset \texttt{reuters} for which we apply a \textit{term frequency-inverse document frequency} (TF-IDF) transformation. As a result, the instances for these two datasets correspond to plain feature vectors. Thus, we employ \textit{multi-layer perceptrons} (MLPs) as the data classification model~$\boldsymbol{f}_{\boldsymbol{\theta}}$. Apart from the input dimension, which depends on the respective dataset, the MLPs share two hidden layers (256 and 128 neurons) enhanced by batch normalization~\citep{ioffe2015batch} and \textit{rectified linear unit}~(ReLU,~\citeauthor{glorot2011deep},~\citeyear{glorot2011deep}) activation functions. For all image datasets, where the actual images with their associated noisy class labels from the crowdworkers are published, we employ a DINOv2 vision transformer (vit-s/14, \cite{oquab2023dinov2}) as a backbone model. Analog to this, we use an MPNet sentence transformer (all-mpnet-base-v2, \citeauthor{song2020mpnet}, \citeyear{song2020mpnet}; \citeauthor{reimers2019sentence}, \citeyear{reimers2019sentence}) as the backbone for the sentences of the dataset~\texttt{spc}. Both backbones' pre-trained weights remain fixed to preserve the robust feature representations as inputs to an MLP head as the data classification model~$\boldsymbol{f}_{\boldsymbol{\theta}}$ with the same architecture (apart from the input dimensions) as for the other datasets.

\paragraph{Learning from Crowds Approaches} Table \ref{tab:hyperparam_overview} lists the $13$ LFC approaches evaluated in our study. We focus on one-stage LFC approaches, whose end-to-end training has yielded state-of-the-art performances~\citep{nguyen2024noisy,herde2024annot,ibrahim2023deep}. Of this type, we include approaches that model class-dependent and instance-dependent crowdworker performances. Further, we consider two two-stage approaches, of which \texttt{mv} serves as a lower baseline because it trains the data classification model~$\boldsymbol{f}_{\boldsymbol{\theta}}$ on majority vote labels. Implementing the crowdworker classification model~$\boldsymbol{g}_{\boldsymbol{\theta}}$ and the crowdworker performance model~$h_{\boldsymbol{\theta}}$ depends on the respective LFC approach, of which Appendix~\ref{app:inference-overview-lfc-approaches} provides more detailed descriptions.

\begin{table}[!ht]
\scriptsize
\centering
\caption{
    \textbf{Overview of LFC approaches' general and individual HP search spaces.} For each HP, we define a default value and a search space as the basis for the HPO. The notation \textit{not applicable} (N/A) indicates that an LFC approach does not introduce additional HPs or that an HP is not optimized. The expressions \texttt{uniform} and \texttt{log-uniform} define the search spaces as distributions used for generating HPCs.
}
\label{tab:value_ranges}
\setlength{\tabcolsep}{9.4pt}
\def\arraystretch{1.0}
\begin{tabular}{lllll}

\toprule

\textbf{Approach} & \textbf{Reference} & \textbf{Hyperparameter} & \textbf{Default Value} & \textbf{Search Space} \\ 

\midrule
\multirow{7}{*}{General}    & \multirow{7}{*}{N/A} & optimizer                 & RAdam             & N/A \\
                            & & learning rate scheduler   & cosine annealing  & N/A \\
                            & & number of epochs          & $30$  & $\texttt{uniform}(\{5, 30, 50\})$ \\
                            & & batch size                & $32$              & $\texttt{uniform}(\{16, 32, 64\})$ \\
                            & & initial learning rate     & $10^{-3}$         & $\texttt{loguniform}([10^{-4}, 10^{-1}])$ \\
                            & & weight decay              & $0$         & $\texttt{loguniform}([10^{-6}, 10^{-3}])$ \\
                            & & dropout rate              & $0.0$             & $\texttt{uniform}([0.0, 0.5])$ \\
\hline
\rowcolor{yescolor!10} \multicolumn{5}{c}{Two-stage Approach without Crowdworker Performance Modeling} \\
\hline
\texttt{mv}             & N/A & N/A               & N/A           & N/A \\
\hline
\rowcolor{yescolor!10} \multicolumn{5}{c}{Two-stage Approach with Class-dependent Crowdworker Performance Modeling} \\
\hline
\texttt{ds} & \citeauthor{dawid1979maximum}        & N/A               & N/A           & N/A \\
\hline
\rowcolor{yescolor!10} \multicolumn{5}{c}{One-stage Approaches with Class-dependent Crowdworker Performance Modeling} \\
\hline
\texttt{cl} & \citeauthor{rodrigues2018deep} & N/A & N/A & N/A \\
\hline
\texttt{trace} & \citeauthor{tanno2019learning} & confusion matrix regularization ($\lambda$) & $10^{-2}$ & $\texttt{loguniform}([10^{-3}, 10^{-1}])$ \\
\hline
\multirow{2}{*}{\texttt{conal}} & \multirow{2}{*}{\citeauthor{chu2021learning}} & confusion matrix regularization ($\lambda$) & $10^{-5}$ & $\texttt{loguniform}([10^{-6}, 10^{-3}])$ \\
                                                       & & embedding dimension & 20 & $\texttt{uniform}(\{20, 40, 60, 80\})$ \\
\hline
\texttt{union-a} & \multirow{2}{*}{\citeauthor{wei2022deep}} & \multirow{2}{*}{confusion matrix initialization ($\epsilon$)} & \multirow{2}{*}{$10^{-5}$} & \multirow{2}{*}{$\texttt{loguniform}([10^{-6}, 10^{-4}])$} \\
\texttt{union-b} & & & & \\
\hline
\texttt{geo-f} &  \multirow{2}{*}{\citeauthor{ibrahim2023deep}} & \multirow{2}{*}{confusion matrix regularization ($\lambda$)} & \multirow{2}{*}{$10^{-3}$} & \multirow{2}{*}{$\texttt{loguniform}([10^{-4}, 10^{-2}])$} \\
\texttt{geo-w} & &  & &  \\
\hline
\rowcolor{yescolor!10} \multicolumn{5}{c}{One-stage Approaches with Instance-dependent Crowdworker Performance Modeling} \\
\hline
\multirow{4}{*}{\texttt{madl}} & \multirow{4}{*}{\citeauthor{herde2023multi}} & confusion matrix initialization ($\eta$) & $0.8$ & $\texttt{uniform}([0.75, 0.95])$ \\
                                                     & & gamma distribution parameter ($\alpha$) & $1.25$ & $\texttt{uniform}([1.0, 1.5])$ \\
                                                     & & gamma distribution parameter ($\beta$) & $0.25$ & $\texttt{uniform}([0.25, 0.5])$ \\
                                                     & & embedding dimension ($Q$) & $16$ & $\texttt{uniform}(\{8, 16, 32\})$ \\
\hline
\texttt{crowd-ar} & \citeauthor{cao2023learning}   & loss balancing & $0.9$ & $\texttt{uniform}([0.5, 1.0])$ \\
\hline
\multirow{2}{*}{\texttt{annot-mix}} & \multirow{2}{*}{\citeauthor{herde2024annot}}   & confusion matrix initialization ($\eta$) & $0.9$ & $\texttt{uniform}([0.75, 0.95])$ \\
                                                            & & mixup ($\alpha$) & $1.0$ & $\texttt{uniform}([0.0, 2.0])$ \\                                                            
\hline
\multirow{3}{*}{\texttt{coin}} & \multirow{3}{*}{\citeauthor{nguyen2024noisy}}   & outlier regularization ($\mu_1$) & $10^{-2}$ & $\texttt{loguniform}([10^{-3}, 10^{-1}])$ \\
                                                                                & & volume regularization ($\mu_2$) & $10^{-2}$ & $\texttt{loguniform}([10^{-3}, 10^{-1}])$ \\
                                                                                & & norm computation ($p$) & $0.4$ & $\texttt{uniform}((0.0, 1.0])$ \\

\bottomrule
\end{tabular}
\end{table}

\paragraph{Hyperparameter Search} Table~\ref{tab:hyperparam_overview} lists general (approach-agnostic) HPs and approach-specific HPs. All approaches share the general HPs. Concretely, we fix RAdam~\citep{liu2019variance} as the optimizer in combination with a cosine annealing learning rate scheduler~\citep{loshchilov2017sgdr} without restarts to gradually reduce the learning rate over the training process, thereby promoting stable convergence. For the remaining general HPs, we define suitable search spaces derived from related literature and default values of PyTorch~\citep{paszke2019pytorch} optimizers (e.g., no weight decay). The distributions for sampling HP values are denoted as \texttt{uniform} and \texttt{log-uniform}. For approach-specific HPs, we adopt default values from the publications or codebases. If available, the search spaces are also extracted from these two sources. Otherwise, they are defined based on reasonable value ranges. The defined HP search spaces are sampled using Sobol sequences~\citep{sobol1998quasi} as HP search strategy (see Section~\ref{subsec:hpo} for our rationale). A total of $50$ distinct HPCs are generated per combination of LFC approach and dataset variant. Together with the default HPC, this yields a set of $|\Lambda| = 51$ HPCs, from which the HPS criteria must pick the best performer.

\paragraph{Hyperparameter Selection Criteria} Each HPC is evaluated via a $K=5$-fold cross-validation given the crowd-labeled training set to obtain risk estimates for the respective HPS criterion. The HPC picked by the respective criterion (see Eq.~\eqref{eq:hpc-selection}) is then tested on the hold-out test set with true labels for five different initializations of the neural networks' weights, ensuring a reliable assessment of the data classification model's generalization performance. We include two variants of default HPCs. On the one hand, we use the default values specified in Table~\ref{tab:value_ranges} across all datasets. This criterion, to which we refer as \textsc{def} (see Eq.~\eqref{eq:naive-empirical-risk}), is the most naive one since there is no consideration of the datasets' individual requirements. A more advanced and commonly used alternative in LFC evaluation studies is to fix one default HPC per dataset, denoted \textsc{def-data}. For each dataset variant, we first perform a conventional HPO: a vanilla classification model, ignoring the LFC setting, is trained and validated on the true class labels, tuning only the approach-agnostic HPs~\citep{herde2024dopanim}. The best HPC obtained from this search is then frozen and transferred to every LFC approach. Consequently, every approach shares the same general HPC for that dataset variant while its approach-specific HPs stay at their default values. For well-studied benchmarks, one could alternatively adopt HPCs reported in the literature~\citep{tanno2019learning}, yet we rely on the HPO variant as \textsc{def-data} criterion to guarantee consistency across datasets. This \textsc{def-data} criterion differs from our upper baseline criterion \textsc{true}, where all HPs of an LFC approach are optimized by training on the noisy crowd-labeled data and validating on the true labels (see Eq.~\eqref{eq:true-empirical-risk}). Despite this difference, \textsc{def-data} and \textsc{true} are the only HPS criteria requiring access to the true labels.

\subsection{Experimental Results}
Our setup encompasses $5$ datasets, each with $7$ variants, $13$ LFC approaches, and $11$ HPS criteria. Due to the many combinations, we present only the main results addressing our two research questions here and refer to Appendix~\ref{subapp:supplementary-results} for the complete results list.
\newline

$\mathbf{RQ_1}$: \textit{Given access only to crowd-labeled validation data, which evaluated HPS criterion yields the highest performance for LFC approaches?}

\begin{table}[!b]
    \scriptsize
    \setlength{\tabcolsep}{8.2pt}
    \def\arraystretch{1}
    \caption{
        \textbf{HPS criteria's results.} One column per criterion reports the rank compared to the other criteria and the zero-one loss reductions (absolute as percentage points [$\%_p$] and relative as percentages [$\%$]) compared to \textsc{def} as criterion. Means and standard errors are computed over all combinations of dataset variants and LFC approaches (excluding the approach \texttt{mv} that is not compatible with each criterion). The arrows show whether a smaller ($\downarrow$) or higher ($\uparrow$) value is better. The \textBF{best} and \underline{second best} value is marked per result type. A $\star$ marks criteria that had access to the true validation labels.
    }
    \label{tab:model-selection-criteria-results}
    \begin{tabular}{ccccccccccc}
        \toprule
        
        \multicolumn{3}{c}{\color{baseline}\textBF{Baseline}} & \multicolumn{4}{c}{\color{aggregation}\textBF{Aggregation-level}} & \multicolumn{3}{c}{\color{crowd}\textBF{Crowd-level}} & \multicolumn{1}{c}{\color{ensemble}\textBF{Ensemble}} \\
        \cmidrule(lr){1-3} \cmidrule(lr){4-7} \cmidrule(lr){8-10} \cmidrule(lr){11-11}

        \multicolumn{1}{c}{{\color{baseline}$\textsc{true}^\star$}} & \multicolumn{1}{c}{\color{baseline}$\textsc{def-data}^\star$} & \multicolumn{1}{c}{\color{baseline}\textsc{def}} & \multicolumn{1}{c}{\color{aggregation}\textsc{aeu}} & \multicolumn{1}{c}{\color{aggregation}\textsc{aec}} & \multicolumn{1}{c}{\color{aggregation}\textsc{alu}} & \multicolumn{1}{c}{\color{aggregation}\textsc{alc}}  & \multicolumn{1}{c}{\color{crowd}\textsc{cxu}} & \multicolumn{1}{c}{\color{crowd}\textsc{cec}} & \multicolumn{1}{c}{\color{crowd}\textsc{clc}} & \multicolumn{1}{c}{\color{ensemble}\textsc{ens}}\\

        \hline
        \rowcolor{yescolor!10} \multicolumn{11}{c}{\scriptsize Ranks ($\downarrow$)}  \\
        \hline
        $\phantom{+1}\textBF{4.64}$ & $\phantom{+1} 7.84$ & $\phantom{+1} 9.76$ & $\phantom{+1} 5.66$ & $\phantom{+1} 5.57$ & $\phantom{+1} 5.29$ & $\phantom{+1} 5.60$ & $\phantom{+1} 5.46$ & $\phantom{+1} 5.36$ & $\phantom{+1} 5.91$ & $\phantom{+1} \underline{4.89}$ \\
        $\pm\phantom{1}0.15$ & $\pm\phantom{1}0.17$ & $\pm\phantom{1}0.11$ & $\pm\phantom{1}0.11$ & $\pm\phantom{1}0.11$ & $\pm\phantom{1}0.11$ & $\pm\phantom{1}0.13$ & $\pm\phantom{1}0.11$ & $\pm\phantom{1}0.11$ & $\pm\phantom{1}0.13$ & $\pm\phantom{1}0.09$ \\
        
        \hline
        \rowcolor{yescolor!10} \multicolumn{11}{c}{\scriptsize Absolute Zero-one Loss Reductions Compared to \textsc{def} ($\Delta_{\textsc{def}} \,L_{0/1}$ [$\%_p$] $\uparrow$)}  \\
        \hline
        $+\phantom{1}\textBF{5.36}$ & $+\phantom{1}0.76$ & $\phantom{+1}0.00$ & $+\phantom{1}4.55$ & $+\phantom{1}4.59$ & $+\phantom{1}4.73$ & $+\phantom{1}4.43$ & $+\phantom{1}4.52$ & $+\phantom{1}4.57$ & $+\phantom{1}4.07$ & $+\phantom{1}\underline{5.15}$ \\
        $\pm \phantom{1}0.24$ & $\pm\phantom{1}0.31$ & $\pm\phantom{1}0.00$ & $\pm\phantom{1}0.26$ & $\pm\phantom{1}0.26$ & $\pm\phantom{1}0.26$ & $\pm\phantom{1}0.27$ & $\pm\phantom{1}0.29$ & $\pm\phantom{1}0.28$ & $\pm\phantom{1}0.31$ & $\pm\phantom{1}0.26$ \\
        \hline
        \rowcolor{yescolor!10} \multicolumn{11}{c}{\scriptsize Relative Zero-one Loss Reductions Compared to \textsc{def} ($\Delta_{\textsc{def}} \,L_{0/1}$ [$\%$] $\uparrow$)}  \\
        \hline
        $+\textBF{18.20}$ & $=\phantom{1}3.65$ & $+\phantom{1}0.00$ & $+16.18$ & $+16.27$ & $+16.48$ & $+14.66$ & $+15.48$ & $+15.93$ & $+13.30$ & $+\underline{17.60}$ \\
        $\pm \phantom{1}0.73$ & $\pm\phantom{1}1.12$ & $\pm\phantom{1}0.00$ & $\pm\phantom{1}0.76$ & $\pm\phantom{1}0.76$ & $\pm\phantom{1}0.80$ & $\pm\phantom{1}0.81$ & $\pm\phantom{1}0.87$ & $\pm\phantom{1}0.84$ & $\pm\phantom{1}0.91$ & $\pm\phantom{1}0.77$ \\
        \bottomrule
    \end{tabular}
\end{table}
\begin{figure}[!b]
    \centering
    \includegraphics[width=\textwidth]{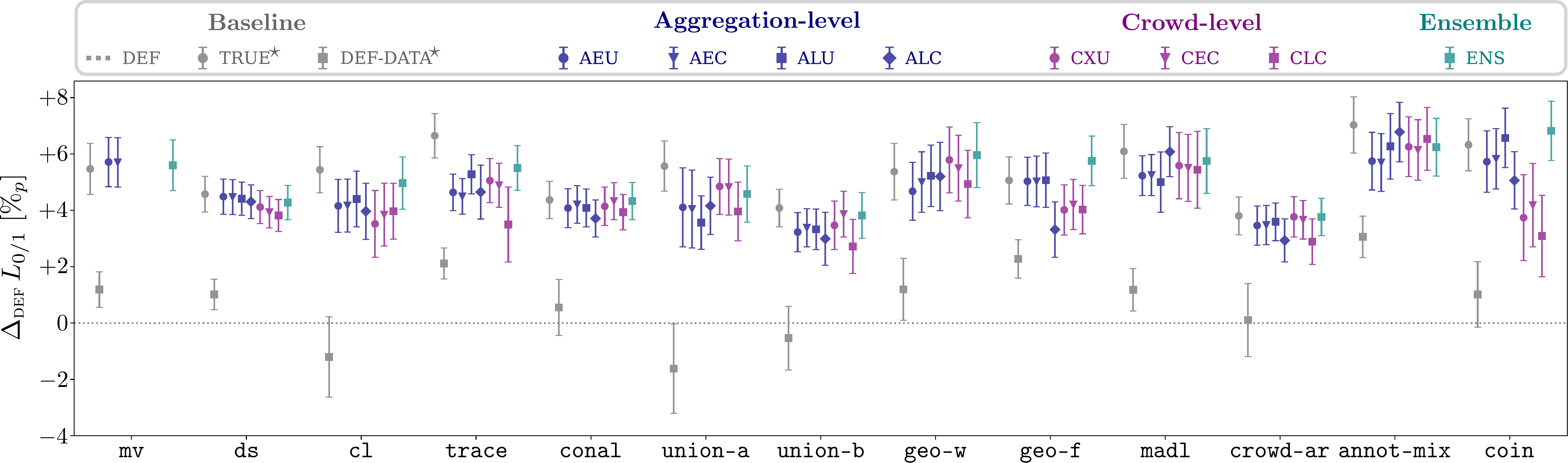}
    \caption{
        \textbf{Absolute zero–one loss reductions of HPS criteria.} For each LFC approach ($x$-axis), the scatter plot displays the mean and standard error of a criterion's reduction in zero–one loss ($y$-axis) as percentage points [$\%_p$] compared to the criterion \textsc{def} (see Appendix~\ref{app:supplementary-experimental-evaluation} for relative reductions as percentages [$\%$]). Higher reductions correspond to greater improvements. A~$\star$ marks criteria that had access to the true validation labels.
    }
    \label{fig:criteria-diff-per-approach}
\end{figure}

We compare the HPS criteria with each other across multiple datasets and LFC approaches. Therefore, we report performance results in the form of rankings and zero-one losses in Table~\ref{tab:model-selection-criteria-results}. One central finding is that the criterion \textsc{def} consistently ranks poorly, supporting our motivation that using default HPCs leads to underestimating the approaches' potential performances. Moreover, the poor rank of the \textsc{def-data} criterion indicates that using well‐tuned HPCs from standard classification tasks (without any noisy labels) does not account for the unique requirements of each LFC approach. As anticipated, the \textsc{true} criterion, which utilizes the true class labels for selecting the best HPC per LFC approach and per dataset, performs superiorly to every other criterion. Concretely, it ranks best and provides the highest loss reductions compared to the \textsc{def} criterion across all dataset variants. Among the criteria relying solely on crowd-labeled validation data, the ensemble-based \textsc{ens} criterion is the clear runner-up: it ranks better and provides higher loss reductions than all other competing criteria apart from the upper baseline criterion \textsc{true}. Reducing the size of the ensemble decreases its performance (see ablation study in Appendix~\ref{subapp:supplementary-results}). At the same time, the extra computation demanded by \textsc{ens} remains negligible (see time complexity analysis in Appendix~\ref{subapp:time-complexity}): the expensive steps are the training and inferring the predictions of the models, and once those predictions are available, they can be reused to compute every risk estimate. Finally, although the \textsc{aeu} criterion, which aggregates the noisy validation labels per instance via majority voting, performs worse than the criteria \textsc{true} and \textsc{ens}, it still yields substantial improvements over default HPCs. Together, these results highlight the benefits of HPO in LFC settings with crowd-labeled validation data. Besides these results averaged over the LFC approaches, we also examine how each criterion performs in combination with each approach individually. For this purpose, Figure~\ref{fig:criteria-diff-per-approach} breaks down the absolute zero-one loss reductions according to the LFC approaches. This way, we measure how much a specific criterion improves an approach relative to using that approach's default HPC. It does not indicate which LFC approach is superior to others. We observe that for a few LFC approaches, e.g., \texttt{coin} and \texttt{geo-f}, the criterion \textsc{true} does not achieve the highest loss reduction on average. An explanation is that cross-validation uses only subsets of the data for training, so the HPC that minimizes the validation zero-one loss on the subsets may not be optimal when training is performed on the full dataset. This observation also suggests an interdependence between the criterion and the approach. Another example of this interdependence can be found when comparing the results for the \texttt{union-a} approach, where criteria that rely on crowd-level risk estimates outperform those that use aggregation-level risks, to the results for the \texttt{coin} approach, where aggregation-level criteria take the lead. Consequently, when it is unclear which criterion will pair best with a given approach, the ensemble-based criterion \textsc{ens} offers a robust compromise.   

\begin{mybox}[\textbf{$\mathbf{RQ_1}$: Takeaway}]
    Aggregation-level and crowd-level HPS criteria, estimating risks only from noisy crowdworker labels, enable effective HPO in LFC. Combining them via the ensemble-based criterion \textsc{ens} yields the highest performance across dataset variants and LFC approaches.
\end{mybox}
\vspace*{1em}

$\mathbf{RQ_2}$: \textit{Given the best-evaluated HPS criterion for crowd-labeled validation data, how do LFC approaches compare in performance?}

\begin{table}[!t]
    \scriptsize
    \setlength{\tabcolsep}{5.1pt}
    \def\arraystretch{1}
    \caption{
        \textbf{LFC approaches' results with} \textsc{ens} \textbf{as HPS criterion.} One column per approach reports the rank compared to the other approaches and the zero-one loss reductions (absolute as percentage points [$\%_p$] and relative as percentages [$\%$]) compared to the approach \texttt{mv} trained with its default (\textsc{def}) HPC. Means and standard errors are computed over all dataset variants. The arrows indicate whether a smaller ($\downarrow$) or higher ($\uparrow$) value is better. The \textBF{best} and \underline{second best} value is marked per result type.
    }
    \label{tab:approaches-ens-results}
    \begin{tabular}{ccccccccccccc}
        \toprule
        
        \multicolumn{2}{c}{\color{baseline}\textBF{Baseline}} & \multicolumn{7}{c}{\color{aggregation}\textBF{Class-dependent}} & \multicolumn{4}{c}{\color{crowd}\textBF{Instance-dependent}}  \\
        
        \cmidrule(lr){1-2} \cmidrule(lr){3-9} \cmidrule(lr){10-13}

        \multicolumn{1}{c}{{\color{baseline}\texttt{mv}}} & \multicolumn{1}{c}{\color{baseline}\texttt{ds}} & \multicolumn{1}{c}{\color{aggregation}\texttt{cl}} & \multicolumn{1}{c}{\color{aggregation}\texttt{trace}} & \multicolumn{1}{c}{\color{aggregation}\texttt{conal}} & \multicolumn{1}{c}{\color{aggregation}\texttt{union-a}} & \multicolumn{1}{c}{\color{aggregation}\texttt{union-b}}  & \multicolumn{1}{c}{\color{aggregation}\texttt{geo-w}} & \multicolumn{1}{c}{\color{aggregation}\texttt{geo-f}} & \multicolumn{1}{c}{\color{crowd}\texttt{madl}} & \multicolumn{1}{c}{\color{crowd}\texttt{crowd-ar}} & \multicolumn{1}{c}{\color{crowd}\texttt{annot-mix}} & \multicolumn{1}{c}{\color{crowd}\texttt{coin}} \\

        \hline
        \rowcolor{yescolor!10} \multicolumn{13}{c}{\scriptsize Ranks ($\downarrow$)}  \\
        \hline
        $\phantom{+}10.57$ & $\phantom{+1}9.09$ & $\phantom{+1}7.90$ & $\phantom{+1}7.31$ & $\phantom{+1}7.37$ & $\phantom{+1}7.60$ & $\phantom{+1}7.61$ & $\phantom{+1}5.81$ & $\phantom{+1}\underline{4.10}$ & $\phantom{+1}5.36$ & $\phantom{+1}8.97$ & $\phantom{+1}5.34$ & $\phantom{+1}\textBF{3.96}$ \\
        $\pm\phantom{1}0.38$ & $\pm\phantom{1}0.53$ & $\pm\phantom{1}0.64$ & $\pm\phantom{1}0.53$ & $\pm\phantom{1}0.53$ & $\pm\phantom{1}0.67$ & $\pm\phantom{1}0.55$ & $\pm\phantom{1}0.59$ & $\pm\phantom{1}0.45$ & $\pm\phantom{1}0.63$ & $\pm\phantom{1}0.43$ & $\pm\phantom{1}0.65$ & $\pm\phantom{1}0.50$ \\
        \hline
        \rowcolor{yescolor!10} \multicolumn{13}{c}{\scriptsize Absolute Zero-one Loss Reductions Compared to \texttt{mv} with \textsc{def} ($\Delta_{\texttt{mv}[\textsc{def}]} \,L_{0/1}$ [$\%_p$] $\uparrow$)}  \\
        \hline
        $+\phantom{0}5.60$ & $+\phantom{0}6.05$ & $+\phantom{0}8.04$ & $+\phantom{0}7.81$ & $+\phantom{0}7.18$ & $+\phantom{0}7.05$ & $+\phantom{0}7.42$ & $+\phantom{0}9.70$ & $+\textBF{10.05}$ & $+\phantom{0}8.52$ & $+\phantom{0}6.87$ & $+\phantom{0}9.21$ & $+\phantom{0}\underline{9.88}$ \\
        $\pm\phantom{0}0.88$ & $\pm\phantom{0}1.17$ & $\pm\phantom{0}1.21$ & $\pm\phantom{0}1.05$ & $\pm\phantom{0}0.92$ & $\pm\phantom{0}0.97$ & $\pm\phantom{0}1.05$ & $\pm\phantom{0}1.30$ & $\pm\phantom{0}1.08$ & $\pm\phantom{0}1.34$ & $\pm\phantom{0}0.97$ & $\pm\phantom{0}0.93$ & $\pm\phantom{0}0.99$ \\
        \hline
        \rowcolor{yescolor!10} \multicolumn{13}{c}{\scriptsize Relative Zero-one Loss Reductions Compared to \texttt{mv} with \textsc{def} ($\Delta_{\texttt{mv}[\textsc{def}]}\,L_{0/1}$ [$\%$] $\uparrow$)}  \\
        \hline
        $+19.31$ & $+21.04$ & $+23.95$ & $+24.99$ & $+24.20$ & $+21.67$ & $+23.82$ & $+28.75$ & $+\underline{30.79}$ & $+27.66$ & $+22.68$ & $+28.22$ & $+\textBF{30.86}$ \\
        $\pm\phantom{0}2.75$ & $\pm\phantom{0}3.14$ & $\pm\phantom{0}2.95$ & $\pm\phantom{0}3.06$ & $\pm\phantom{0}2.83$ & $\pm\phantom{0}2.92$ & $\pm\phantom{0}2.86$ & $\pm\phantom{0}2.99$ & $\pm\phantom{0}2.86$ & $\pm\phantom{0}3.45$ & $\pm\phantom{0}2.87$ & $\pm\phantom{0}2.60$ & $\pm\phantom{0}2.84$ \\
        \bottomrule
    \end{tabular}
\end{table}
\begin{figure}[!t]
    \centering
    \includegraphics[width=\textwidth]{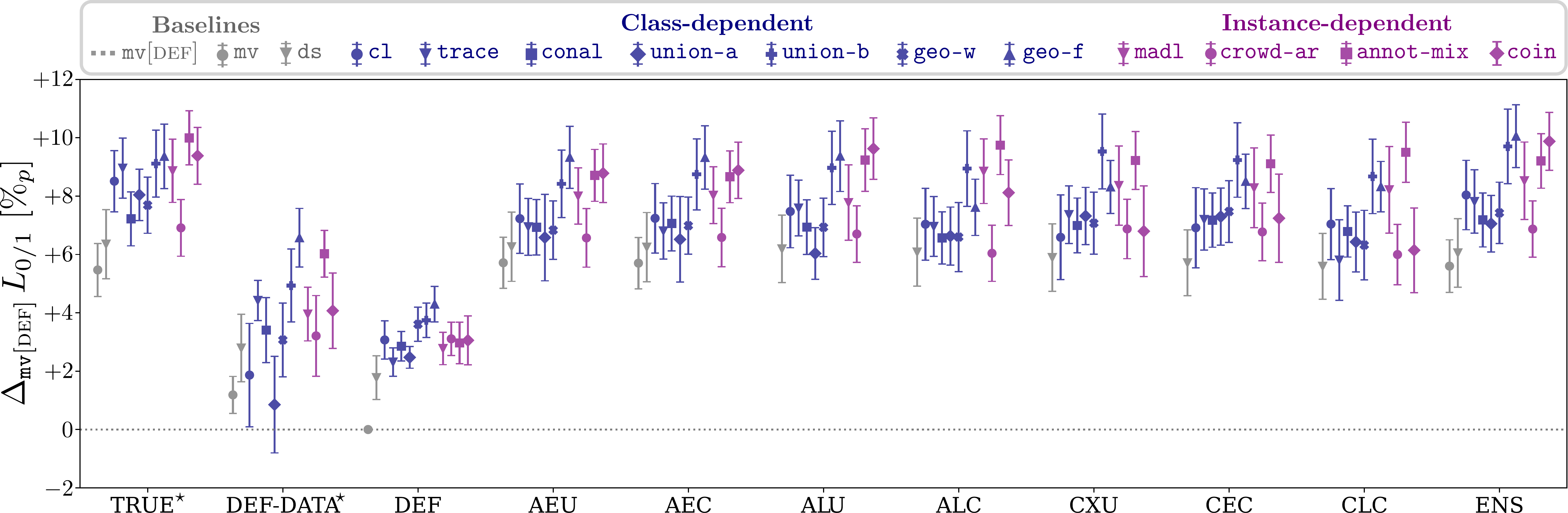}
    \caption{
        \textbf{Absolute zero–one loss reductions of LFC approaches.} For each HPS criterion ($x$-axis), the scatter plot displays the mean and standard error of an LFC approach's reduction in zero–one loss ($y$-axis) as percentage points [$\%_p$] compared to majority voting (\texttt{mv}) trained with its default (\textsc{def}) HPC (see Appendix~\ref{app:supplementary-experimental-evaluation} for relative reductions as percentages [$\%$]). Higher reductions correspond to greater improvements. A~$\star$ marks criteria that had access to the true validation labels.
    }
    \label{fig:approaches-diff-per-criterion}
\end{figure}

We compare the LFC approaches whose HPCs have been selected via the best-evaluated criterion \textsc{ens} across all dataset variants and report the performance results again as rankings and zero-one loss reductions in Table~\ref{tab:approaches-ens-results}. These results underscore the benefits of one-stage LFC approaches that estimate class- or instance-dependent crowdworker performances, as the two-stage baseline approaches \texttt{mv} and \texttt{ds}~\citep{dawid1979maximum}  attain the worst ranks. Moreover, they provide the lowest zero-one loss reductions compared to training \texttt{mv} with its default (\textsc{def}) HPC. The two one-stage approaches \texttt{geo-f}~\citep{ibrahim2023deep} and \texttt{coin}~\citep{nguyen2024noisy} provide performance gains to most of their competitors with an average zero-one loss reduction of around $\si{10}{\%_p}$ corresponding to a relative improvement of over $\si{30}{\%}$ compared to training \texttt{mv} with \textsc{def} as a criterion. Both approaches' idea is to identify the crowdworkers' confusion matrices via special regularization terms, where \texttt{geo-f} estimates a single confusion matrix per worker, while \texttt{coin} relaxes this assumption by modeling instance-dependent outlier terms for the crowdworkers' confusion matrices. Figure~\ref{fig:approaches-diff-per-criterion} breaks down the approaches' absolute zero-one loss reductions compared to training \texttt{mv} with the default (\textsc{def}) HPC according to the HPS criteria.  We observe that similar absolute reductions of the zero-one loss around $\si{10}{\%_p}$ are also achievable when relying on other criteria, such as \textsc{alc} paired with the \texttt{annot-mix} approach~\citep{herde2024annot}. Moreover, the results indicate that the comparison between the approaches' performances is affected by the choice of the criterion. For example, the loss reduction of the approach \texttt{coin} is much inferior to the reduction of the approach \texttt{annot-mix} for the baseline criterion \textsc{def-data}, but is superior when relying on \textsc{ens} as a criterion. 

\begin{figure}[!h]
    \centering
    \includegraphics[width=\linewidth]{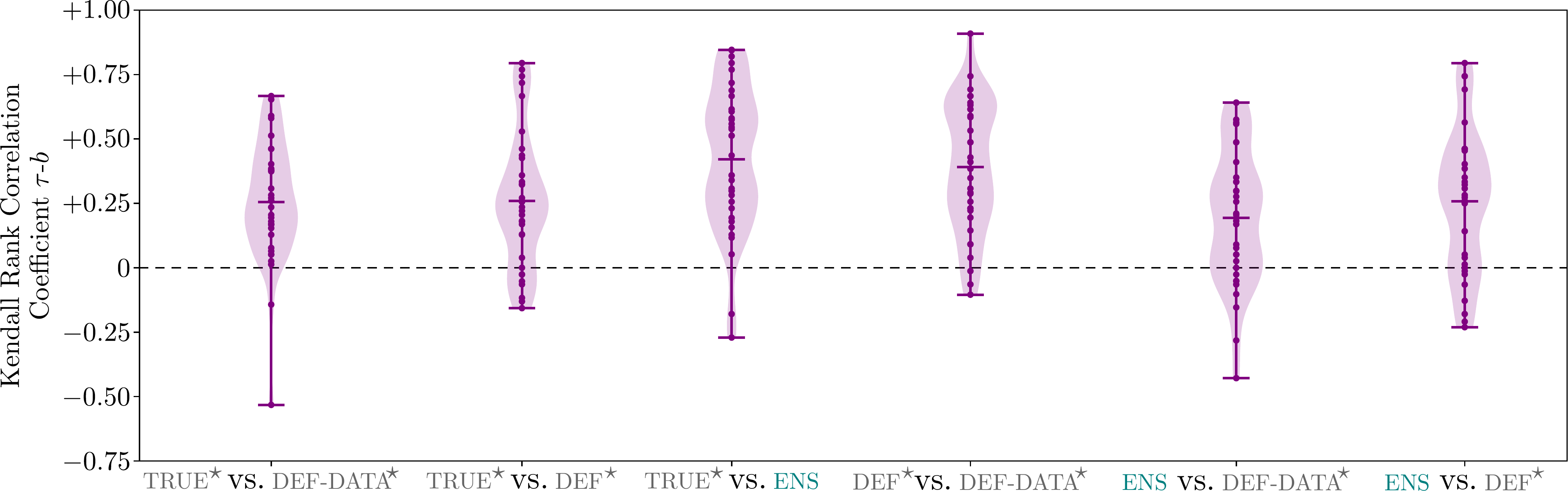}
    \caption{
        \textbf{Rank correlation between HPS criteria.} Each violin plot shows the distribution of pairwise Kendall $\tau$-$b$ coefficients, visualized as violet dots and obtained from 35 dataset variants when comparing the ranking of LFC approaches with their HPCs selected via {\color{baseline} baseline} and {\color{ensemble} ensemble-based} criteria. Higher coefficients indicate stronger correlation.  A $\star$ marks criteria that had access to the true validation labels.
    }
    \label{fig:kendall-violinplots}
\end{figure}
To systematically analyze a criterion's impact on comparing the approaches, we compute the Kendall rank correlation coefficient~\citep{kendall1945treatment} to judge whether the various criteria disagree on the ranking of approaches. Figure~\ref{fig:kendall-violinplots} shows violin plots of these coefficients across all dataset variants for $6$ example criteria pairs. Each coefficient is a kind of distance measurement between two rankings of the approaches obtained after using two different criteria for the same dataset variant. The coefficients are in the interval $[-1, 1]$, where $-1$ indicates a perfect negative ordering, $0$ no ordering, and $1$ a perfect positive ordering. In the absence of ties, a of Kendall $\tau$-$b$ coefficient of $0.50$ between two rankings of the approaches means that, if you randomly sample a pair of approaches, there is a $50\%_p$ advantage that both criteria place the same approach higher (or lower) than the other, rather than disagreeing on which one ranks higher. The criteria \textsc{def} and \textsc{def-data} paired with \textsc{ens} or \textsc{true} have an average Kendall $\tau$-$b$ coefficient of around $0.25$, showing only modest agreement in the ranking of LFC approaches. Criterion \textsc{ens} versus criterion \textsc{true} has an average Kendall $\tau$-$b$ coefficient of about $0.40$, reflecting a moderate positive overlap of rankings. The violin plots, however, make clear that these average Kendall $\tau$-$b$ coefficients mask considerable dispersion: dataset-wise coefficients spread widely around each average, and a few even reach into the negative value range, showing that two criteria reverse the ordering of approaches for some datasets.

\begin{mybox}[\textbf{$\mathbf{RQ_2}$: Takeaway}]
    With the HPS criterion \textsc{ens}, all one-stage LFC approaches, particularly \texttt{geo-f} and \texttt{coin}, outperform the two-stage baselines. The approaches' gains and rankings vary with the criterion, underscoring a criterion's importance for a fair and realistic evaluation.
\end{mybox}

\subsection{Recommendations for Experimentation}
Based on \texttt{crowd-hpo} with its experimental study and related ones (see Section~\ref{sec:related-work}), we make recommendations towards more realistic and fairer experimentation in LFC settings via HPO. We use the term experimentation to emphasize our focus on experiments to set up a benchmark, while a study's specific objectives guide the subsequent analyses. Our \textbf{recommendations} address the following central aspects.
\begin{itemize}[leftmargin=*]
    \item \textit{Datasets:} Focus on datasets with noisy class labels collected from human crowdworkers. Many existing experimental studies in LFC primarily rely on datasets with simulated crowdworkers (see Table~\ref{tab:datasets}). Typically, these simulations apply hand-crafted noise models to probe specific properties of an LFC approach, e.g., robustness to adversarial crowdworkers~\citep{cao2018maxmig} or varying noise levels~\citep{li2022beyond}.
    Yet, how well such noise models mirror real human labeling behavior is unclear. A more realistic alternative is to create variants of noisy label sets~\citep{wei2021learning,herde2024dopanim} for the datasets with human crowdworkers (see Table~\ref{tab:datasets}), which also enables the study of LFC approaches under different noise levels and label redundancies.  
    \item \textit{Learning from crowds approaches:} Evaluate LFC approaches with different training principles and crowdworker performance modeling. By including approaches with two-stage and one-stage training~\citep{li2022beyond}, a study can systematically contrast label aggregation as a separate stage with end-to-end training, enabling the investigation of how factors such as label redundancy and noise level influence their predictive performances. Further, experimenting with approaches modeling class-dependent or instance-dependent crowdworker performances~\citep{herde2023multi} allows a fundamental assessment of when additional modeling complexity is justified.
    \item \textit{Hyperparameter selection criteria:} Employ non-default HPS criteria. Default HPCs lead to underestimating the LFC approaches' actual performances and rendering ranking results less meaningful. In applications where assuming the availability of a separate validation set with true labels is reasonable, validating with those labels (corresponding to the criterion~\textsc{true}) is fine. Otherwise, \texttt{crowd-hpo} offers criteria for crowd-labeled validation data to enable fairer and more realistic experimentation. Given that the optimal criterion may depend on the individual LFC approach, a basic criterion such as \textsc{aeu} or the more robust ensemble-based criterion \textsc{ens} is recommended. In the future, a novel LFC approach should be introduced alongside an HPS criterion tailored to its specific characteristics, thereby improving real-world applicability. This recommendation mirrors recent findings from partial label learning, where HPO with only partial labels for validation has likewise been identified as a key challenge~\citep{wang2025realistic}. 
    \item \textit{Hyperparameter search spaces:} Define tight search spaces covering the most critical HPs for each LFC approach. Established findings in HPO~\citep{bergstra2012random} revealed a substantial impact of the complexity of the search space~$\Omega_{\Lambda}$ on the efficiency and effectiveness of a search strategy. To reduce this complexity, we need to know suitable search spaces for the critical HPs of an LFC approach. Therefore, we advocate that future LFC approaches should include recommendations for search spaces of their approach-specific HPs. When such guidance is unavailable, reasonable search spaces should be established based on theoretical considerations, each HP's role, or potential ablation studies.
\end{itemize}
Our recommendations have several \textbf{limitations}. First, due to inherent interdependencies, we cannot make design choices regarding the above aspects in isolation. For example, the choice of the best HPC criterion for an approach can vary across datasets (see Table~\ref{tab:zero-one-loss-results-1} in Appendix~\ref{app:supplementary-experimental-evaluation}). Second, other essential aspects, such as the HPO search strategy, the number of folds $K$ for the cross-validation, and the number of candidate HPCs $|\Lambda|$, remain unexplored. Studying these aspects can reduce the computational complexity of the HPO. Second, our analysis based on computing means across all datasets does not account for the influence of specific dataset attributes, including noise level or label redundancy. In other words, latent patterns in these characteristics may help decide the HPS criterion or even allow optimizing the risk estimates' weighting for the ensemble-based criterion \textsc{ens} via techniques from metalearning~\citep{brazdil2022metalearning}. Third, while we rely on the zero-one loss function, alternative loss functions such as the Brier score~\citep{brier1950verification} may also be relevant when assessing probabilistic estimates. Fourth and finally, we refer to the number $|\Lambda|$ of candidate HPCs as the computational budget, which does not account for different training and inference complexities between the LFC approaches.
\section{Conclusion}
We introduced \texttt{crowd-hpo} as a framework to enable more realistic and fairer benchmarking of LFC approaches by leveraging HPO with only access to crowd-labeled validation data. We started with exemplary results demonstrating notable zero-one loss reductions and changes in the rankings of LFC approaches when performing HPO with true class labels in the validation set compared to default HPCs. Subsequently, we identified a lack of research regarding HPO with crowd-labeled validation data. Therefore, we proposed and evaluated HPS criteria accounting for the potential noise in class labels from crowdworkers. Across extensive experiments,  our proposed HPS criteria strongly reduced the losses of the LFC approaches relative to their default HPCs. This applies particularly to the ensemble-based criterion \textsc{ens}, which is also easily extensible by including future empirical risk measures for crowd-labeled validation data. However, the ranking of LFC approaches shifted with the criterion applied. These findings grounded our recommendations for future experimentation and benchmarking in LFC settings. To further improve the HPS, future work should rigorously explore advanced HP search strategies, particularly Bayesian optimization \citep{wang2023recent}, and examine how they interact with criteria accounting for crowd-labeled validation data. Finally, \texttt{crowd-hpo} serves as a starting point for developing empirical risk measures that are not only suited for improving HPS given crowd-labeled validation data but ultimately should provide reliable estimates of the approaches' generalization performances given only crowd-labeled test data. 

\impact{%
    In a broader context, we identify an impact on real-world applications and crowdworkers as two branches of potential societal consequences of \texttt{crowd-hpo}. (1) On the one hand, the validation with noisy class labels from the crowd makes the LFC approaches' employment more practical in real-world applications because potential users do not need to rely on default or manually picked HPCs when training the LFC approaches. However, selecting the best HPC is different from accurate risk estimates. In the first case, we need only to rank the HPCs, while in the latter case, we want to have a precise estimate of the resulting data classification model's risk. In other words, despite being able to optimize the selection of the HPC, we do not know the actual generalization performance of the data classification model after training. Correspondingly, there is still no solution to how practitioners can obtain such an estimate without access to a separate test set with true labels. Therefore, practitioners must only employ LFC approaches in safety-critical applications after thorough evaluation based on a separate test set with true labels. (2) On the other hand, wider adoption of LFC approaches can raise demand for crowdworkers. However, crowdworkers often endure difficult working conditions~\citep{bhatti2020general}, e.g., insufficient payments and job insecurity. Therefore, collecting crowd-labeled data for training LFC approaches should always be coupled with explicit provisions for fair working conditions.  
}

\ifnum\value{anonymous}=1
    \acks{%
        This work was funded by Anonymous.
    }
\else
    \acks{%
        This work was funded by the CIL project (P/710, P/1082) through the University of Kassel.
    }
\fi

\bibliography{main}

\begin{thebibliography}{61}
\providecommand{\natexlab}[1]{#1}
\providecommand{\url}[1]{\texttt{#1}}
\expandafter\ifx\csname urlstyle\endcsname\relax
  \providecommand{\doi}[1]{doi: #1}\else
  \providecommand{\doi}{doi: \begingroup \urlstyle{rm}\Url}\fi

\bibitem[Abdulrahman et~al.(2018)Abdulrahman, Brazdil, Van~Rijn, and Vanschoren]{abdulrahman2018speeding}
Salisu~Mamman Abdulrahman, Pavel Brazdil, Jan~N Van~Rijn, and Joaquin Vanschoren.
\newblock Speeding up algorithm selection using average ranking and active testing by introducing runtime.
\newblock \emph{Mach. Learn.}, 107:\penalty0 79--108, 2018.

\bibitem[Bagnall \& Cawley(2017)Bagnall and Cawley]{bagnall2017use}
Anthony Bagnall and Gavin~C Cawley.
\newblock {On the Use of Default Parameter Settings in the Empirical Evaluation of Classification Algorithms}.
\newblock \emph{arXiv:1703.06777}, 2017.

\bibitem[Ball et~al.(2020)Ball, Squeglia, Tapert, and Paulus]{ball2020double}
Tali~M Ball, Lindsay~M Squeglia, Susan~F Tapert, and Martin~P Paulus.
\newblock {Double Dipping in Machine Learning: Problems and Solutions}.
\newblock \emph{Biol. Psychiatry: Cogn. Neurosci. Neuroimaging}, 5\penalty0 (3):\penalty0 261--263, 2020.

\bibitem[Bergstra \& Bengio(2012)Bergstra and Bengio]{bergstra2012random}
James Bergstra and Yoshua Bengio.
\newblock {Random Search for Hyper-Parameter Optimization}.
\newblock \emph{J. Mach. Learn. Res.}, 13\penalty0 (2):\penalty0 281--305, 2012.

\bibitem[Bhatti et~al.(2020)Bhatti, Gao, and Chen]{bhatti2020general}
Shahzad~Sarwar Bhatti, Xiaofeng Gao, and Guihai Chen.
\newblock {General framework, opportunities and challenges for crowdsourcing techniques: A comprehensive survey}.
\newblock \emph{J. Syst. Softw.}, 167:\penalty0 110611, 2020.

\bibitem[Brazdil et~al.(2022)Brazdil, van Rijn, Soares, and Vanschoren]{brazdil2022metalearning}
Pavel Brazdil, Jan~N van Rijn, Carlos Soares, and Joaquin Vanschoren.
\newblock {Metalearning for Hyperparameter Optimization}.
\newblock In \emph{Metalearning: Applications to Automated Machine Learning and Data Mining}, pp.\  103--122. Springer, 2022.

\bibitem[Brier(1950)]{brier1950verification}
Glenn~W Brier.
\newblock {Verification of Forecasts Expressed in Terms of Probability}.
\newblock \emph{{Mon. Weather Rev.}}, 78\penalty0 (1):\penalty0 1--3, 1950.

\bibitem[Cao et~al.(2019)Cao, Xu, Kong, and Wang]{cao2018maxmig}
Peng Cao, Yilun Xu, Yuqing Kong, and Yizhou Wang.
\newblock {Max-MIG: an Information Theoretic Approach for Joint Learning from Crowds}.
\newblock In \emph{Int. Conf. Learn. Represent.}, 2019.

\bibitem[Cao et~al.({2023})Cao, Chen, Huang, Shen, and Huang]{cao2023learning}
Zhi Cao, Enhong Chen, Ye~Huang, Shuanghong Shen, and Zhenya Huang.
\newblock {Learning from Crowds with Annotation Reliability}.
\newblock In \emph{Int. ACM SIGIR Conf. Res. Dev. Inf. Retr.}, pp.\  2103--2107, {2023}.

\bibitem[Chen et~al.(2021)Chen, Ye, Chen, Zhao, and Heng]{chen2021robustness}
Pengfei Chen, Junjie Ye, Guangyong Chen, Jingwei Zhao, and Pheng-Ann Heng.
\newblock {Robustness of Accuracy Metric and its Inspirations in Learning with Noisy Labels}.
\newblock In \emph{AAAI Conf. Artif. Intell.}, pp.\  11451--11461, 2021.

\bibitem[Chen et~al.(2022)Chen, Jiang, and Li]{chen2022label}
Ziqi Chen, Liangxiao Jiang, and Chaoqun Li.
\newblock Label augmented and weighted majority voting for crowdsourcing.
\newblock \emph{Inf. Sci.}, 606:\penalty0 397--409, 2022.

\bibitem[Chu et~al.(2021)Chu, Ma, and Wang]{chu2021learning}
Zhendong Chu, Jing Ma, and Hongning Wang.
\newblock {Learning from Crowds by Modeling Common Confusions}.
\newblock In \emph{AAAI Conf. Artif. Intell.}, pp.\  5832--5840, 2021.

\bibitem[Dawid \& Skene(1979)Dawid and Skene]{dawid1979maximum}
Alexander~Philip Dawid and Allan~M Skene.
\newblock {Maximum Likelihood Estimation of Observer Error-Rates Using the EM Algorithm}.
\newblock \emph{J. R. Stat. Soc.}, 28\penalty0 (1):\penalty0 20--28, 1979.

\bibitem[de~Borda(1781)]{borda1781}
Jean-Charles de~Borda.
\newblock Mémoire sur les élections au scrutin.
\newblock \emph{Histoire de l'Académie Royale des Sciences}, 1781.

\bibitem[Glorot et~al.(2011)Glorot, Bordes, and Bengio]{glorot2011deep}
Xavier Glorot, Antoine Bordes, and Yoshua Bengio.
\newblock {Deep Sparse Rectifier Neural Networks}.
\newblock In \emph{Int. Conf. Artif. Intell. Stat.}, pp.\  315--323, 2011.

\bibitem[Guo et~al.(2024)Guo, Yi, and Wang]{guo2024learning}
Hui Guo, Grace Yi, and Boyu Wang.
\newblock {Learning from Noisy Labels via Conditional Distributionally Robust Optimization}.
\newblock In \emph{Adv. Neural Inf. Process. Syst.}, pp.\  82627--82672, 2024.

\bibitem[Han et~al.(2024)Han, Sun, Zhang, Zhang, Hu, Li, Zhou, Ye, and He]{han2024collaborative}
Bin Han, Yi-Xuan Sun, Ya-Lin Zhang, Libang Zhang, Haoran Hu, Longfei Li, Jun Zhou, Guo Ye, and Huimei He.
\newblock {Collaborative Refining for Learning from Inaccurate Labels}.
\newblock In \emph{Adv. Neural Inf. Process. Syst.}, pp.\  92745--92768, 2024.

\bibitem[Herde et~al.(2021)Herde, Huseljic, Sick, and Calma]{herde2021survey}
Marek Herde, Denis Huseljic, Bernhard Sick, and Adrian Calma.
\newblock {A Survey on Cost Types, Interaction Schemes, and Annotator Performance Models in Selection Algorithms for Active Learning in Classification}.
\newblock \emph{IEEE Access}, 9:\penalty0 166970--166989, 2021.

\bibitem[Herde et~al.(2023)Herde, Huseljic, and Sick]{herde2023multi}
Marek Herde, Denis Huseljic, and Bernhard Sick.
\newblock {Multi-annotator Deep Learning: A Probabilistic Framework for Classification}.
\newblock \emph{Trans. Mach. Learn. Res.}, 2023.

\bibitem[Herde et~al.(2024{\natexlab{a}})Herde, Huseljic, Rauch, and Sick]{herde2024dopanim}
Marek Herde, Denis Huseljic, Lukas Rauch, and Bernhard Sick.
\newblock {dopanim: A Dataset of Doppelganger Animals with Noisy Annotations from Multiple Humans}.
\newblock In \emph{Adv. Neural Inf. Process. Syst.}, pp.\  51085--51117, 2024{\natexlab{a}}.

\bibitem[Herde et~al.(2024{\natexlab{b}})Herde, L{\"u}hrs, Huseljic, and Sick]{herde2024annot}
Marek Herde, Lukas L{\"u}hrs, Denis Huseljic, and Bernhard Sick.
\newblock {Annot-Mix: Learning with Noisy Class Labels from Multiple Annotators via a Mixup Extension}.
\newblock In \emph{Eur. Conf. Artif. Intell.}, pp.\  2910--2918, 2024{\natexlab{b}}.

\bibitem[Ibrahim et~al.(2023)Ibrahim, Nguyen, and Fu]{ibrahim2023deep}
Shahana Ibrahim, Tri Nguyen, and Xiao Fu.
\newblock {Deep Learning From Crowdsourced Labels: Coupled Cross-Entropy Minimization, Identifiability, and Regularization}.
\newblock In \emph{Int. Conf. Learn. Represent.}, 2023.

\bibitem[Inouye et~al.(2017)Inouye, Ravikumar, and Das]{inouye2017hyperparameter}
David~I Inouye, Pradeep Ravikumar, and Pradipto Das.
\newblock {Hyperparameter Selection under Localized Label Noise via Corrupt Validation}.
\newblock In \emph{Learn. Limit. Label. Data Workshop}, 2017.

\bibitem[Ioffe \& Szegedy(2015)Ioffe and Szegedy]{ioffe2015batch}
Sergey Ioffe and Christian Szegedy.
\newblock {Batch Normalization: Accelerating Deep Network Training by Reducing Internal Covariate Shift}.
\newblock In \emph{Int. Conf. Mach. Learn.}, pp.\  448--456, 2015.

\bibitem[Jiang et~al.(2021)Jiang, Zhang, Tao, and Li]{jiang2021learning}
Liangxiao Jiang, Hao Zhang, Fangna Tao, and Chaoqun Li.
\newblock {Learning From Crowds With Multiple Noisy Label Distribution Propagation}.
\newblock \emph{IEEE Trans. Neural Netw. Learn. Syst.}, 33\penalty0 (11):\penalty0 6558--6568, 2021.

\bibitem[Kendall(1945)]{kendall1945treatment}
Maurice~George Kendall.
\newblock {The Treatment of Ties in Ranking Problems}.
\newblock \emph{Biometrika}, 33\penalty0 (3):\penalty0 239--251, 1945.

\bibitem[Khetan et~al.(2018)Khetan, Lipton, and Anandkumar]{khetan2018learning}
Ashish Khetan, Zachary~C. Lipton, and Animashree Anandkumar.
\newblock {Learning From Noisy Singly-labeled Data}.
\newblock In \emph{Int. Conf. Learn. Represent.}, 2018.

\bibitem[Krizhevsky(2009)]{krizhevsky2009learning}
Alex Krizhevsky.
\newblock {Learning Multiple Layers of Features from Tiny Images}.
\newblock Master's thesis, University of Toronto, 2009.

\bibitem[Kuo et~al.(2023)Kuo, Thaker, Khodak, Nguyen, Jiang, Talwalkar, and Smith]{kuo2023noisy}
Kevin Kuo, Pratiksha Thaker, Mikhail Khodak, John Nguyen, Daniel Jiang, Ameet Talwalkar, and Virginia Smith.
\newblock {On Noisy Evaluation in Federated Hyperparameter Tuning}.
\newblock In \emph{Annual Conf. Mach. Learn. Syst.}, pp.\  127--144, 2023.

\bibitem[Lewis(1987)]{reuters19878lewis}
David Lewis.
\newblock {Reuters-21578 Text Categorization Collection}.
\newblock UCI Machine Learning Repository, 1987.
\newblock {DOI}: https://doi.org/10.24432/C52G6M.

\bibitem[Li et~al.(2022)Li, Sun, and Li]{li2022beyond}
Jingzheng Li, Hailong Sun, and Jiyi Li.
\newblock Beyond confusion matrix: learning from multiple annotators with awareness of instance features.
\newblock \emph{Mach. Learn.}, pp.\  1--23, 2022.

\bibitem[Li et~al.(2021)Li, Liu, Tan, Zeng, and Ge]{li2021trustable}
Shikun Li, Tongliang Liu, Jiyong Tan, Dan Zeng, and Shiming Ge.
\newblock {Trustable Co-Label Learning from Multiple Noisy Annotators}.
\newblock \emph{IEEE Trans. Multimed.}, 25:\penalty0 1045--1057, 2021.

\bibitem[Li et~al.(2024)Li, Xia, Deng, Gey, and Liu]{li2024transferring}
Shikun Li, Xiaobo Xia, Jiankang Deng, Shiming Gey, and Tongliang Liu.
\newblock {Transferring Annotator- and Instance-Dependent Transition Matrix for Learning From Crowds}.
\newblock \emph{IEEE Trans. Pattern Anal. Mach. Intell}, 46\penalty0 (11):\penalty0 7377--7391, 2024.

\bibitem[Liu et~al.(2019)Liu, Jiang, He, Chen, Liu, Gao, and Han]{liu2019variance}
Liyuan Liu, Haoming Jiang, Pengcheng He, Weizhu Chen, Xiaodong Liu, Jianfeng Gao, and Jiawei Han.
\newblock {On the Variance of the Adaptive Learning Rate and Beyond}.
\newblock In \emph{Int. Conf. Learn. Represent.}, 2019.

\bibitem[Loshchilov \& Hutter(2017)Loshchilov and Hutter]{loshchilov2017sgdr}
Ilya Loshchilov and Frank Hutter.
\newblock {{SGDR}: Stochastic Gradient Descent with Warm Restarts}.
\newblock In \emph{Int. Conf. Learn. Represent.}, 2017.

\bibitem[Nguyen et~al.(2024)Nguyen, Ibrahim, and Fu]{nguyen2024noisy}
Tri Nguyen, Shahana Ibrahim, and Xiao Fu.
\newblock {Noisy Label Learning with Instance-Dependent Outliers: Identifiability via Crowd Wisdom}.
\newblock In \emph{Adv. Neural Inf. Process. Syst.}, pp.\  97261--97298, 2024.

\bibitem[Oquab et~al.(2023)Oquab, Darcet, Moutakanni, Vo, Szafraniec, Khalidov, Fernandez, Haziza, Massa, El-Nouby, et~al.]{oquab2023dinov2}
Maxime Oquab, Timoth{\'e}e Darcet, Th{\'e}o Moutakanni, Huy~V Vo, Marc Szafraniec, Vasil Khalidov, Pierre Fernandez, Daniel Haziza, Francisco Massa, Alaaeldin El-Nouby, et~al.
\newblock {DINOv2: Learning Robust Visual Features without Supervision}.
\newblock \emph{Trans. Mach. Learn. Res.}, 2023.

\bibitem[Pang \& Lee(2005)Pang and Lee]{pang2005seeing}
Bo~Pang and Lillian Lee.
\newblock {Seeing Stars: Exploiting Class Relationships for Sentiment Categorization with Respect to Rating Scales}.
\newblock In \emph{Annu. Meet. Assoc. Comput. Linguist.}, pp.\  115--124, 2005.

\bibitem[Paszke et~al.(2019)Paszke, Gross, Massa, Lerer, Bradbury, Chanan, Killeen, Lin, Gimelshein, Antiga, et~al.]{paszke2019pytorch}
Adam Paszke, Sam Gross, Francisco Massa, Adam Lerer, James Bradbury, Gregory Chanan, Trevor Killeen, Zeming Lin, Natalia Gimelshein, Luca Antiga, et~al.
\newblock {PyTorch: An Imperative Style, High-Performance Deep Learning Library}.
\newblock In \emph{Adv. Neural Inf. Process. Syst.}, 2019.

\bibitem[Raykar et~al.(2010)Raykar, Yu, Zhao, Valadez, Florin, Bogoni, and Moy]{raykar2010learning}
Vikas~C. Raykar, Shipeng Yu, Linda~H. Zhao, Gerardo~Hermosillo Valadez, Charles Florin, Luca Bogoni, and Linda Moy.
\newblock {Learning from Crowds}.
\newblock \emph{J. Mach. Learn. Res.}, 11\penalty0 (4):\penalty0 1297--1322, 2010.

\bibitem[Reimers \& Gurevych(2019)Reimers and Gurevych]{reimers2019sentence}
Nils Reimers and Iryna Gurevych.
\newblock {Sentence-{BERT}: Sentence Embeddings using {S}iamese {BERT}-Networks}.
\newblock In \emph{Conf. Empir. Methods Nat. Lang. Process. Int. Jt. Conf. Nat. Lang. Process.}, pp.\  3982--3992, 2019.

\bibitem[Rodrigues \& Pereira(2018)Rodrigues and Pereira]{rodrigues2018deep}
Filipe Rodrigues and Francisco Pereira.
\newblock {Deep Learning from Crowds}.
\newblock In \emph{{AAAI Conf. Artif. Intell.}}, pp.\  1611--1618, 2018.

\bibitem[Rodrigues et~al.(2013)Rodrigues, Pereira, and Ribeiro]{rodrigues2013learning}
Filipe Rodrigues, Francisco Pereira, and Bernardete Ribeiro.
\newblock {Learning from multiple annotators: Distinguishing good from random labelers}.
\newblock \emph{Pattern Recognit. Lett.}, 34\penalty0 (12):\penalty0 1428--1436, 2013.

\bibitem[Rodrigues et~al.(2017)Rodrigues, Lourenco, Ribeiro, and Pereira]{rodrigues2017learning}
Filipe Rodrigues, Mariana Lourenco, Bernardete Ribeiro, and Francisco~C. Pereira.
\newblock {Learning Supervised Topic Models for Classification and Regression from Crowds}.
\newblock \emph{IEEE Trans. Pattern Anal. Mach. Intell.}, 39\penalty0 (12):\penalty0 2409--2422, 2017.

\bibitem[Russell et~al.(2008)Russell, Torralba, Murphy, and Freeman]{russell2008labelme}
Bryan~C. Russell, Antonio Torralba, Kevin~P. Murphy, and William~T. Freeman.
\newblock {LabelMe: A Database and Web-Based Tool for Image Annotation}.
\newblock \emph{Int. J. Comput. Vis}, 77\penalty0 (1-3):\penalty0 157--173, 2008.

\bibitem[Sobol(1998)]{sobol1998quasi}
Ilya~M Sobol.
\newblock {On quasi-Monte Carlo integrations}.
\newblock \emph{Math. Comput. Simul.}, 47\penalty0 (2-5):\penalty0 103--112, 1998.

\bibitem[Song et~al.(2022)Song, Kim, Park, Shin, and Lee]{song2022learning}
Hwanjun Song, Minseok Kim, Dongmin Park, Yooju Shin, and Jae-Gil Lee.
\newblock {Learning From Noisy Labels With Deep Neural Networks: A Survey}.
\newblock \emph{IEEE Trans. Neural Netw. Learn. Syst.}, 34\penalty0 (11):\penalty0 8135--8153, 2022.

\bibitem[Song et~al.(2020)Song, Tan, Qin, Lu, and Liu]{song2020mpnet}
Kaitao Song, Xu~Tan, Tao Qin, Jianfeng Lu, and Tie-Yan Liu.
\newblock {MPNet: Masked and Permuted Pre-training for Language Understanding}.
\newblock In \emph{Adv. Neural Inf. Process. Syst.}, pp.\  16857--16867, 2020.

\bibitem[Tanno et~al.(2019)Tanno, Saeedi, Sankaranarayanan, Alexander, and Silberman]{tanno2019learning}
Ryutaro Tanno, Ardavan Saeedi, Swami Sankaranarayanan, Daniel~C. Alexander, and Nathan Silberman.
\newblock {Learning from Noisy Labels by Regularized Estimation of Annotator Confusion}.
\newblock In \emph{Conf. Comput. Vis. Pattern Recognit.}, pp.\  11244--11253, 2019.

\bibitem[Tzanetakis \& Cook(2002)Tzanetakis and Cook]{tzanetakis2002musical}
George Tzanetakis and Perry Cook.
\newblock Musical genre classification of audio signals.
\newblock \emph{IEEE Trans. Signal Process.}, 10\penalty0 (5):\penalty0 293--302, 2002.

\bibitem[Vapnik(1995)]{vapnik1995the}
Vladimir Vapnik.
\newblock \emph{{The Nature of Statistical Learning Theory}}.
\newblock Springer, 1995.

\bibitem[Vaughan(2018)]{vaughan2017making}
Jennifer~W. Vaughan.
\newblock {Making Better Use of the Crowd: How Crowdsourcing Can Advance Machine Learning Research}.
\newblock \emph{J. Mach. Learn. Res.}, 18\penalty0 (193):\penalty0 1--46, 2018.

\bibitem[Wang et~al.(2025)Wang, Wu, Wang, Niu, Zhang, and Sugiyama]{wang2025realistic}
Wei Wang, Dong-Dong Wu, Jindong Wang, Gang Niu, Min-Ling Zhang, and Masashi Sugiyama.
\newblock {Realistic Evaluation of Deep Partial-Label Learning Algorithms}.
\newblock In \emph{Int. Conf. Learn. Represent.}, 2025.

\bibitem[Wang et~al.(2023)Wang, Jin, Schmitt, and Olhofer]{wang2023recent}
Xilu Wang, Yaochu Jin, Sebastian Schmitt, and Markus Olhofer.
\newblock {Recent Advances in Bayesian Optimization}.
\newblock \emph{ACM Comput. Surv.}, 55\penalty0 (13s):\penalty0 1--36, 2023.

\bibitem[Wei et~al.(2023)Wei, Xie, Feng, Han, and An]{wei2022deep}
Hongxin Wei, Renchunzi Xie, Lei Feng, Bo~Han, and Bo~An.
\newblock {Deep Learning From Multiple Noisy Annotators as A Union}.
\newblock \emph{IEEE Trans. Neural Netw. Learn. Syst.}, 34\penalty0 (12):\penalty0 10552--10562, 2023.

\bibitem[Wei et~al.(2021)Wei, Zhu, Cheng, Liu, Niu, and Liu]{wei2021learning}
Jiaheng Wei, Zhaowei Zhu, Hao Cheng, Tongliang Liu, Gang Niu, and Yang Liu.
\newblock {Learning with Noisy Labels Revisited: A Study Using Real-World Human Annotations}.
\newblock In \emph{Int. Conf. Learn. Represent.}, 2021.

\bibitem[Yuan et~al.(2024)Yuan, Feng, and Liu]{yuan2024early}
Suqin Yuan, Lei Feng, and Tongliang Liu.
\newblock {Early Stopping Against Label Noise Without Validation Data}.
\newblock In \emph{Int. Conf. Learn. Represent.}, 2024.

\bibitem[Zhang et~al.(2017)Zhang, Bengio, Hardt, Recht, and Vinyals]{zhang2017understanding}
Chiyuan Zhang, Samy Bengio, Moritz Hardt, Benjamin Recht, and Oriol Vinyals.
\newblock {Understanding Deep Learning Requires Rethinking Generalization}.
\newblock In \emph{Int. Conf. Learn. Represent.}, 2017.

\bibitem[Zhang et~al.(2024)Zhang, Li, Zeng, Yan, and Ge]{zhang2024coupled}
Hansong Zhang, Shikun Li, Dan Zeng, Chenggang Yan, and Shiming Ge.
\newblock {Coupled Confusion Correction: Learning from Crowds with Sparse Annotations}.
\newblock In \emph{AAAI Conf. Artif. Intell.}, pp.\  16732--16740, 2024.

\bibitem[Zhang et~al.(2016)Zhang, Wu, and Sheng]{zhang2016learning}
Jing Zhang, Xindong Wu, and Victor~S. Sheng.
\newblock {Learning from Crowdsourced Labeled Data: A Survey}.
\newblock \emph{{Artif. Intell. Rev.}}, 46\penalty0 (4):\penalty0 543--576, 2016.

\bibitem[Zhang et~al.(2020)Zhang, Tanno, Xu, Jin, Jacob, Cicarrelli, Barkhof, and Alexander]{le2020disentangling}
Le~Zhang, Ryutaro Tanno, Mou-Cheng Xu, Chen Jin, Joseph Jacob, Olga Cicarrelli, Frederik Barkhof, and Daniel Alexander.
\newblock {Disentangling Human Error from Ground Truth in Segmentation of Medical Images}.
\newblock In \emph{Adv. Neural Inf. Process. Syst.}, pp.\  15750--15762, 2020.

\end{thebibliography}
\bibliographystyle{tmlr}

\appendix
\section{Inference Mechanisms of Learning from Crowds Approaches}
\label{app:inference-overview-lfc-approaches}
This appendix overviews the common inference mechanisms of LFC approaches to better understand the connections between the data classification model~$\boldsymbol{f}_{\boldsymbol{\theta}}$, the crowdworker classification model~$\boldsymbol{g}_{\boldsymbol{\theta}}$, and the crowdworker performance model~$h_{\boldsymbol{\theta}}$. Moreover, the probabilistic estimates of \mbox{Eqs.~\eqref{eq:true-label-probabilities}-\eqref{eq:labeling-accuracy}} are required to evaluate our presented HPS criteria. To explain the inference, we distinguish between two architecture types employed by LFCs approaches: the ones with probabilistic confusion matrices and those with non-probabilistic noise adaptation layers. 

\subsection{Probabilistic Confusion Matrices}
Many LFC approaches~\citep{dawid1979maximum,tanno2019learning,herde2023multi,ibrahim2023deep,cao2023learning,herde2023multi,herde2024annot,nguyen2024noisy} estimate crowdworker performances through confusion matrices, which we formalize as a function $\boldsymbol{Q}_{\boldsymbol{\theta}}: \Omega_X \times [M] \rightarrow \Delta_C^C$. Thereby, a confusion matrix entry has the following probabilistic interpretation:
\begin{equation}
    [\boldsymbol{Q}_{\boldsymbol{\theta}}(\boldsymbol{x}_n, m)]_{c, k} \defas \Pr\left(\boldsymbol{z}_{nm} = \boldsymbol{e}_k \mid \boldsymbol{x}_n, \boldsymbol{y}_n = \boldsymbol{e}_c, \boldsymbol{\theta}\right).
\end{equation}
This confusion matrix entry in row $c \in [C]$ and column $k \in [C]$ is the probability that crowdworker $m$ assigns the class label $\boldsymbol{e}_k$ to instance $\boldsymbol{x}_n$ with $\boldsymbol{e}_c$ as its true class label. Depending on the assumptions of the LFC approach, there are confusion matrices differing in their degree of freedom $\nu \in \mathbb{N}_{>0}$~\citep{herde2023multi}. Here, we distinguish between class-independent ($\nu = 1$) and class-dependent ($\nu = (C-1)^2$) confusion matrices. Moreover, the confusion matrices can be modeled as instance-independent:
\begin{equation}
    \forall \boldsymbol{x}_n, \boldsymbol{x}_l \in \Omega_X: \boldsymbol{Q}_{\boldsymbol{\theta}}(\boldsymbol{x}_n, m) = \boldsymbol{Q}_{\boldsymbol{\theta}}(\boldsymbol{x}_l, m),
\end{equation}
or as an instance-dependent function. Despite different assumptions about confusion matrices, the LFC approaches share the following inference scheme for their crowdworker classification model:
\begin{equation}
    \boldsymbol{g}_{\boldsymbol{\theta}}(\boldsymbol{x}_n, m) \defas  \boldsymbol{Q}^\mathrm{T}_{\boldsymbol{\theta}}(\boldsymbol{x}_n, m)\boldsymbol{f}_{\boldsymbol{\theta}}(\boldsymbol{x}_n)
\end{equation}
and their crowdworker performance model:
\begin{equation}
    h_{\boldsymbol{\theta}}(\boldsymbol{x}_n, m) \defas \sum_{c \in [C]} [\boldsymbol{f}_{\boldsymbol{\theta}}(\boldsymbol{x}_n)]_c \cdot [\boldsymbol{Q}_{\boldsymbol{\theta}}(\boldsymbol{x}_n, m)]_{c,c}.
\end{equation}

\subsection{Non-probabilistic Noise Adaptation Layers}
In contrast to probabilistic confusion matrices, we refer to noise adaptation layers in LFC approaches~\citep{rodrigues2018deep,chu2021learning,wei2022deep} as unconstrained transformations of the estimated true class probabilities. For this purpose, the approach \texttt{cl}~\citep{rodrigues2018deep} introduces a set of crowdworker-specific noise adaptation layers $\{\boldsymbol{W}_m \in \mathbb{R}^{C \times C}\}_{m=1}^M$, with which the crowdworker classification model performs inference via
\begin{equation}
    \boldsymbol{g}_{\boldsymbol{\theta}}(\boldsymbol{x}_n, m) \defas \mathrm{softmax}\left(\boldsymbol{W}_m^\mathrm{T}\boldsymbol{f}_{\boldsymbol{\theta}}(\boldsymbol{x}_n)\right).
\end{equation}
The approach \texttt{conal}~\citep{chu2021learning} extends this set of crowdworker-specific noise adaptation layers by another layer $\boldsymbol{\overline{W}} \in \mathbb{R}^{C \times C}$ modeling common confusions across crowdworkers, which leads to the following inference scheme:
\begin{equation}
    \boldsymbol{g}_{\boldsymbol{\theta}}(\boldsymbol{x}_n, m) \defas (1-\kappa_{nm}) \cdot \mathrm{softmax}\left(\boldsymbol{W}^\mathrm{T}_m\boldsymbol{f}_{\boldsymbol{\theta}}(\boldsymbol{x}_n)\right) + \kappa_{nm} \cdot \mathrm{softmax}\left(\boldsymbol{\overline{W}}^\mathrm{T}\boldsymbol{f}_{\boldsymbol{\theta}}(\boldsymbol{x}_n)\right),
\end{equation}
where $\kappa_{nm} \in [0, 1]$ is an instance- and worker-dependent scalar estimating the degree to which a crowdworker's class label distribution follows common confusions across crowdworkers. Another variant of a noise adaptation layer is implemented by the LFC approaches \texttt{union-a} and \texttt{union-b}~\citep{wei2022deep}. Instead of treating the crowdworkers independently, the two approaches' idea is to model the crowdworkers as a union through a single noise adaptation layer $\boldsymbol{\widetilde{W}} \in \mathbb{R}^{C \times (C \cdot M)}$. Therefore, they do not directly implement a crowdworker classification model but a classification model $\boldsymbol{\widetilde{g}}_{\boldsymbol{\theta}}: \Omega_X \rightarrow \Delta_{C \cdot M}$ treating the crowdworkers' class labels as a union with
\begin{align}
    \boldsymbol{\widetilde{g}}_{\boldsymbol{\theta}}(\boldsymbol{x}_n) &\defas \mathrm{softmax}\left(\boldsymbol{\widetilde{W}}^\mathrm{T}\boldsymbol{f}_{\boldsymbol{\theta}}(\boldsymbol{x}_n)\right) &(\texttt{union-a}), \\
    \boldsymbol{\widetilde{g}}_{\boldsymbol{\theta}}(\boldsymbol{x}_n) &\defas  \mathrm{softmax}\left(\boldsymbol{\widetilde{W}}\right)^\mathrm{T}\boldsymbol{f}_{\boldsymbol{\theta}}(\boldsymbol{x}_n) &(\texttt{union-b}).
\end{align}
As a workaround for approximating the crowdworker classification model, we normalize the outputs associated with each crowdworker, which corresponds to:
\begin{align}
    \boldsymbol{g}_{\boldsymbol{\theta}}(\boldsymbol{x}_n, m) \defas \mathrm{normalize}\left(\left[\boldsymbol{\widetilde{g}}_{\boldsymbol{\theta}}(\boldsymbol{x}_n)\right]_{(m-1) \cdot C + 1 : m \cdot C}\right),
\end{align}
where $[\cdot]_{i:j}$ denotes the entries from index $i$ to index $j$ in a vector. For all these LFC approaches, which do not explicitly implement a probabilistic confusion matrix per crowdworker, we resort to using marginal alignment accuracy, which is computed as the agreement between the predicted crowdworker distribution and the predicted true label distribution as an instance-level proxy measure for crowdworker accuracy:
\begin{equation}
    h_{\boldsymbol{\theta}}(\boldsymbol{x}_n, m) \defas \boldsymbol{f}^{\mathrm{T}}_{\boldsymbol{\theta}}(\boldsymbol{x}_n) \boldsymbol{g}_{\boldsymbol{\theta}}(\boldsymbol{x}_n, m).
\end{equation}

\section{Theoretical Analysis of Hyperparameter Selection Criteria}
\label{app:hyperparameter-selection-criteria}
This appendix expands on our design decisions underlying the HPS criteria framework introduced in Section~\ref{sec:approach}. Using simple propositions, we show that richer modeling assumptions can inject class-specific bias into the posterior computation of Eq.~\eqref{eq:independence}. We then prove that the computation we adopt in Eq.~\eqref{eq:simplified-class-probabilities} is immune to this issue and remains class-agnostic. Finally, we provide an overview of the time complexities of the HPS criteria.

\subsection{Prior Class Probabilities} 
\label{subapp:prior-class-probabilities}
Proposition~\ref{prop:uniform-class-probabilities} motivates our design choice not to use non-uniform prior class probabilities when computing the posterior class probabilities in Eq.~\eqref{eq:independence}. Intuitively, if the prior class probabilities are sufficiently biased toward a particular class, this bias will dominate in the posterior regardless of the observed class labels from the crowdworkers. Therefore, relying on the class probabilities estimated by the data classification model $\boldsymbol{f}_{\boldsymbol{\theta}}$ as prior can bias the aggregated class labels towards the data classification model's predictions.
\begin{proposition}
    \label{prop:uniform-class-probabilities}
    For an instance $\boldsymbol{x}_n\in\mathcal{X}$, let us assume strictly positive likelihoods for the class labels such that $\forall \boldsymbol{e}_c \in \Omega_Y$ and $\forall \boldsymbol{z}_{nm} \in \mathcal{Z}_n$:
    \begin{align}
        \label{eq-prop:uniform-class-probabilities-1}
        \widehat{\Pr}(\boldsymbol{z}_{nm} \mid \boldsymbol{x}_n, \boldsymbol{y}_n = \boldsymbol{e}_c) > 0.
    \end{align}
    Then, for any non-empty allocation of observed class labels $\mathcal{Z}_n \in \mathcal{P}(\Omega_Y) \setminus \{\emptyset\}$ and for any class label $\boldsymbol{e}_{k} \in \Omega_Y$, there exists a constant $\varepsilon \in (0, 1)$ such that from $\widehat{\Pr}(\boldsymbol{y}_n = \boldsymbol{e}_{k} \mid \boldsymbol{x}_n) > \varepsilon$ follows:
    \begin{align}
        \boldsymbol{\overline{z}}(\boldsymbol{x}_n, \mathcal{Z}_{n}) = \boldsymbol{e}_{k}.
    \end{align}
\end{proposition}
\begin{proof}
    Let us define:
    \begin{align}
        l_{nc} &\defas \prod_{\boldsymbol{z}_{nm} \in \mathcal{Z}_n} \widehat{\Pr}(\boldsymbol{z}_{nm} \mid \boldsymbol{x}_n, \boldsymbol{y}_n = \boldsymbol{e}_c) > 0, \\
        l_{\max} &\defas \max_{c \in [C]\setminus\{k\}} \left(l_{nc}\right), \\
        \varepsilon &\defas \frac{l_{\max}}{l_{\max} + l_{nk}}.
    \end{align}
    Then, we have for $\widehat{\Pr}(\boldsymbol{y}_n = \boldsymbol{e}_k \mid \boldsymbol{x}_n) > \varepsilon$ that $\forall \boldsymbol{e}_c \in \Omega_Y \setminus \{\boldsymbol{e}_k\}$:
    \begin{align}
        \widehat{\Pr}(\boldsymbol{y}_n = \boldsymbol{e}_k \mid \boldsymbol{x}_n, \mathcal{Z}_n) &= \widehat{\Pr}(\boldsymbol{y}_n = \boldsymbol{e}_k \mid \boldsymbol{x}_n)\prod_{\boldsymbol{z}_{nm} \in \mathcal{Z}_n} \widehat{\Pr}(\boldsymbol{z}_{nm} \mid \boldsymbol{x}_n, \boldsymbol{y}_n = \boldsymbol{e}_k) \\
        & > \frac{l_{\max}}{l_{\max} + l_{nk}} l_{nk} \\
        & = \frac{l_{nk}}{l_{\max} + l_{nk}} l_{\max} \\
        & = \widehat{\Pr}(\boldsymbol{y}_n \neq \boldsymbol{e}_k \mid \boldsymbol{x}_n) l_{\max} \\
        & \ge \widehat{\Pr}(\boldsymbol{y}_n = \boldsymbol{e}_{c} \mid \boldsymbol{x}_n) l_{c} \label{eq-prop:smaller-value} \\
        & = \widehat{\Pr}(\boldsymbol{y}_n = \boldsymbol{e}_{c} \mid \boldsymbol{x}_n) \prod_{\boldsymbol{z}_{nm} \in \mathcal{Z}_n} \widehat{\Pr}(\boldsymbol{z}_{nm} \mid \boldsymbol{x}_n, \boldsymbol{y}_n = \boldsymbol{e}_c) \\
        & = \widehat{\Pr}(\boldsymbol{y}_n = \boldsymbol{e}_c \mid \boldsymbol{x}_n, \mathcal{Z}_n),
    \end{align}
    where Eq.~\eqref{eq-prop:smaller-value} follows from monotonicity as $\boldsymbol{y}_n = \boldsymbol{e}_{c} \implies \boldsymbol{y}_n \neq \boldsymbol{e}_k$ and by definition as $l_c \le l_{\max}$.
    Therefore, $\boldsymbol{e}_k \in \Omega_Y$ is our MAP estimate and, thus, our aggregated class label.
\end{proof}

\subsection{Crowdworker Confusion Probabilities}
\label{subapp:crowdworker-probabilities}
Proposition~\ref{prop:crowdworker-confusion-probabilities} motivates our design choice not to use full confusion probability estimates for computing the posterior class probabilities in Eq.~\eqref{eq:independence}. Intuitively, suppose the full confusion probabilities for the crowdworkers are sufficiently biased toward a particular class. In that case, this bias will dominate in the posterior regardless of uniform prior class probabilities and observed class labels from the crowdworkers. Therefore, relying on full confusion matrices, as estimated by many LFC approaches in the form of $\boldsymbol{Q}_{\boldsymbol{\theta}}$ in Appendix~\ref{app:inference-overview-lfc-approaches}, can bias the aggregated class labels towards their own predictions. 
\begin{proposition}
\label{prop:crowdworker-confusion-probabilities}
For an instance $\boldsymbol{x}_n\in\mathcal{X}$, suppose uniform prior class probabilities according to Eq.~\eqref{eq:simplified-class-probabilities}. Then, for any non-empty allocation of observed class labels $\mathcal{Z}_n \in \mathcal{P}(\Omega_Y) \setminus \{\emptyset\}$ and for any class label $\boldsymbol{e}_{k} \in \Omega_Y$, there exist confusion probability estimates 
$\widehat{\Pr}(\boldsymbol{z}_{nm} \mid \boldsymbol{x}_n, \boldsymbol{y}_n)$ such that:
\begin{align}
    \boldsymbol{\overline{z}}(\boldsymbol{x}_n, \mathcal{Z}_{n}) = \boldsymbol{e}_k.
\end{align}
\end{proposition}

\begin{proof}
    With uniform prior class probabilities, the posterior probability for class label $\boldsymbol{e}_c \in \Omega_Y$ is proportional to the joint likelihood:
    \begin{align}
        \widehat{\Pr}(\boldsymbol{y}_n = \boldsymbol{e}_c \mid \boldsymbol{x}_n, \mathcal{Z}_n) \propto \prod_{\boldsymbol{z}_{nm} \in \mathcal{Z}_n} \widehat{\Pr}(\boldsymbol{z}_{nm}\mid \boldsymbol{x}_n, \boldsymbol{y}_n=\boldsymbol e_c).
    \end{align}
    Therefore, by requiring that $\forall \boldsymbol{e}_c \in \Omega_Y \setminus \{\boldsymbol{e}_k\}:$
    \begin{align}
            \widehat{\Pr}(\boldsymbol{z}_{nm}\mid \boldsymbol{x}_n, \boldsymbol{y}_n=\boldsymbol e_k) > \widehat{\Pr}(\boldsymbol{z}_{nm}\mid \boldsymbol{x}_n, \boldsymbol{y}_n=\boldsymbol e_c),
    \end{align}        
    we obtain $\forall \boldsymbol{e}_c \in \Omega_Y \setminus \{\boldsymbol{e}_k\}$:
    \begin{align}
        \widehat{\Pr}(\boldsymbol{y}_n = \boldsymbol{e}_k \mid \boldsymbol{x}_n, \mathcal{Z}_n) &= \prod_{\boldsymbol{z}_{nm} \in \mathcal{Z}_n} \widehat{\Pr}(\boldsymbol{z}_{nm}\mid \boldsymbol{x}_n, \boldsymbol{y}_n=\boldsymbol e_k)\\ &> \prod_{\boldsymbol{z}_{nm} \in \mathcal{Z}_n}\widehat{\Pr}(\boldsymbol{z}_{nm}\mid \boldsymbol{x}_n, \boldsymbol{y}_n=\boldsymbol e_c)\\ &= \widehat{\Pr}(\boldsymbol{y}_n = \boldsymbol{e}_c \mid \boldsymbol{x}_n, \mathcal{Z}_n),
    \end{align}
    which implies that $\boldsymbol{e}_k \in \Omega_Y$ is our MAP estimate and, thus, our aggregated class label.
\end{proof}
\subsection{Maximum-a-posterior Estimate as Weighted Majority Vote}
\label{subapp:weighted-majority-vote}
Proposition~\ref{prop:weighted-majority-voting} motivates our design choice to use uniform prior class probabilities as in Eq.~\eqref{eq:simplified-class-probabilities} and a Bernoulli model for the instance-wise crowdworkers' confusion probabilities as in Eq.~\eqref{eq:simplified-performance}. Intuitively, the posterior class probabilities correspond to a soft weighted majority vote, so the MAP estimate is the label with the most ``soft votes''. This weighted majority voting treats all classes symmetrically, and any difference in posterior probabilities or the MAP estimate arises only from the class labels provided by the crowdworkers, never from a built-in preference for a particular class.

\begin{proposition}
    \label{prop:weighted-majority-voting}
    Given the posterior class probability computation from Eq.~\eqref{eq:simplified-class-posteriors} for an instance $\boldsymbol{x}_n \in \mathcal{X}$ with any non-empty allocation of observed class labels $\mathcal{Z}_n \in \mathcal{P}(\Omega_Y) \setminus \{\emptyset\}$, the aggregated label in Eq.~\eqref{eq:aggregation-function} equals a weighted majority vote such that:
    \begin{align}
        \label{prop-eq:weighted-majority-vote}
        \boldsymbol{\overline{z}}(\boldsymbol{x}_n, \mathcal{Z}_n) = \argmax_{\boldsymbol{e}_c \in \Omega_Y} \left(\sum_{\boldsymbol{z}_{nm} \in \mathcal{Z}_n} \ln\left( \frac{\widehat{\Pr}(\boldsymbol{z}_{nm}^\mathrm{T}\boldsymbol{y}_n = 1 \mid \boldsymbol{x}_n)(C-1)}{\widehat{\Pr}(\boldsymbol{z}_{nm}^\mathrm{T}\boldsymbol{y}_n = 0 \mid \boldsymbol{x}_n)} \right) \cdot \left(\boldsymbol{z}_{nm}^\mathrm{T}\boldsymbol{e}_c\right)\right).
    \end{align}
\end{proposition}

\begin{proof}
    For ease of notation, let us define:
    \begin{equation}
        \label{proof-eq:alpha-definition}
        \alpha_{nm} \defas \widehat{\Pr}(\boldsymbol{z}_{nm}^\mathrm{T}\boldsymbol{y}_n = 1 \mid \boldsymbol{x}_n).
    \end{equation}
    Then, we can rewrite the un-normalized posterior probability from Eq.~\eqref{eq:simplified-class-posteriors} according to:
    \begin{align}
        \widehat{\Pr}(\boldsymbol{y}_n = \boldsymbol{e}_c \mid \boldsymbol{x}_n, \mathcal{Z}_n) 
        &\propto \exp\left(\ln\left(\prod_{\boldsymbol{z}_{nm} \in \mathcal{Z}_n} 
        (\alpha_{nm})^{\boldsymbol{z}_{nm}^\mathrm{T} \boldsymbol{e}_c}
        \left(\frac{1 - \alpha_{nm}}{C-1}\right)^{1-\boldsymbol{z}_{nm}^\mathrm{T}\boldsymbol{e}_c}\right)\right) \nonumber\\
        &= \exp\left(\sum_{\boldsymbol{z}_{nm} \in \mathcal{Z}_n}\ln\left( 
        (\alpha_{nm})^{\boldsymbol{z}_{nm}^\mathrm{T} \boldsymbol{e}_c}
        \left(\frac{1-\alpha_{nm}}{C-1}\right)^{1-\boldsymbol{z}_{nm}^\mathrm{T}\boldsymbol{e}_c}\right)\right) \nonumber\\
        &= \exp\left(\sum_{\boldsymbol{z}_{nm} \in \mathcal{Z}_n} 
        (\boldsymbol{z}_{nm}^\mathrm{T} \boldsymbol{e}_c)\ln(\alpha_{nm}) + 
        \left(1-\boldsymbol{z}_{nm}^\mathrm{T}\boldsymbol{e}_c\right) \ln\left(\frac{1-\alpha_{nm}}{C-1}\right)\right) \nonumber\\
        &= \exp\left(\sum_{\boldsymbol{z}_{nm} \in \mathcal{Z}_n} 
        (\boldsymbol{z}_{nm}^\mathrm{T} \boldsymbol{e}_c) \left(\ln(\alpha_{nm}) - \ln\left(\frac{1-\alpha_{nm}}{C-1}\right)\right) + \ln\left(\frac{1-\alpha_{nm}}{C-1}\right)\right) \nonumber\\
        &\propto \exp\left(\sum_{\boldsymbol{z}_{nm} \in \mathcal{Z}_n} 
        (\boldsymbol{z}_{nm}^\mathrm{T} \boldsymbol{e}_c) \left(\ln(\alpha_{nm}) - \ln\left(\frac{1-\alpha_{nm}}{C-1}\right)\right)\right) \nonumber\\
        &= \exp\left(\sum_{\boldsymbol{z}_{nm} \in \mathcal{Z}_n} 
        (\boldsymbol{z}_{nm}^\mathrm{T} \boldsymbol{e}_c) \ln\left(\frac{\alpha_{nm}(C-1)}{1 - \alpha_{nm}}\right)\right).
    \end{align}
    Replacing $\alpha_{nm}$ with its definition from Eq.~\eqref{proof-eq:alpha-definition}, we obtain the input expression of the $\argmax$ function in Eq.~\eqref{prop-eq:weighted-majority-vote}.
\end{proof}

\clearpage
\subsection{Worst-case Time Complexity Analysis}
\label{subapp:time-complexity}
\begin{wrapfigure}[18]{r}{0.45\textwidth}
    \vspace*{-0.465cm}
    \centering
    \includegraphics[width=\linewidth]{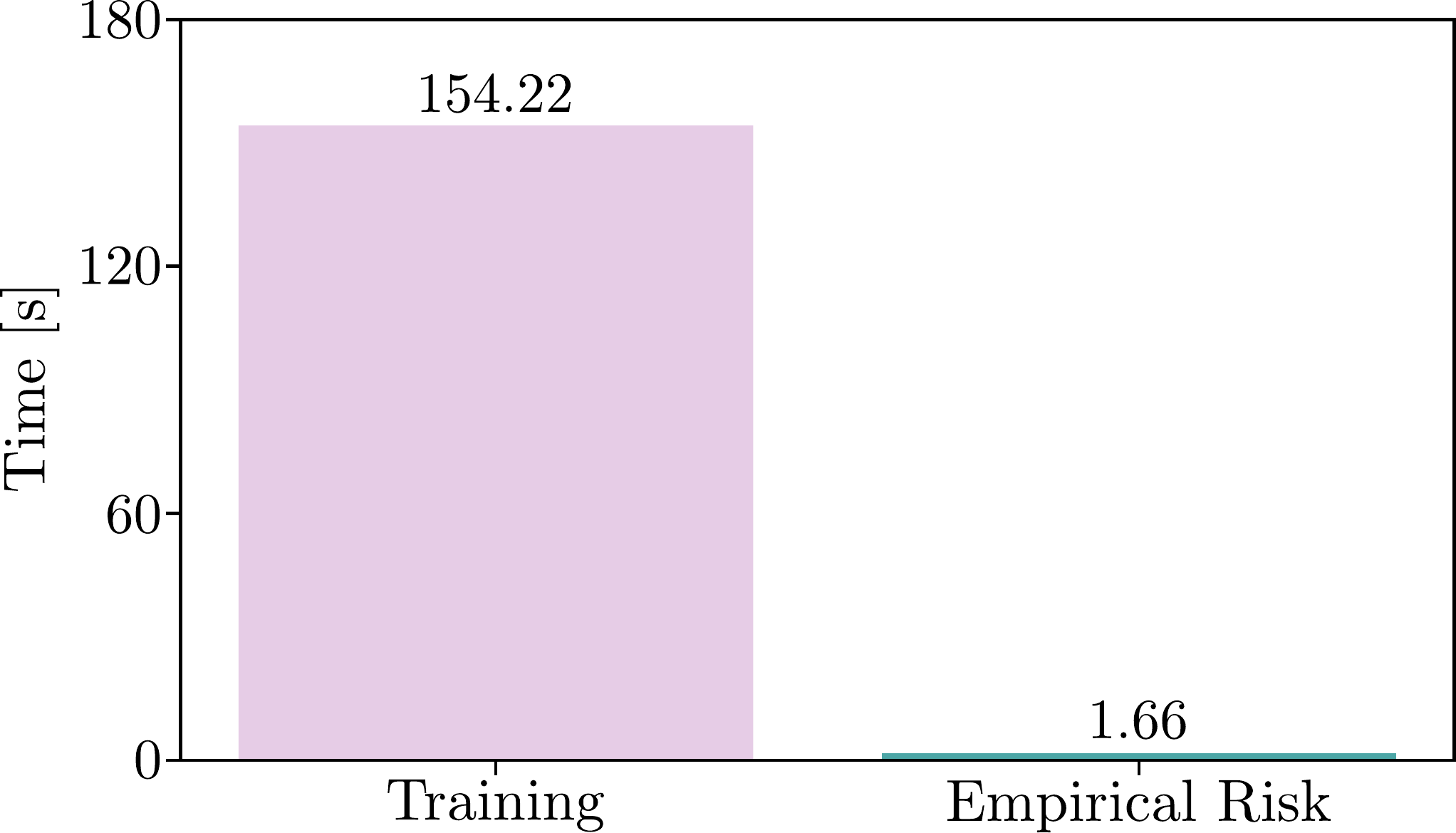}
    \caption{
        \textbf{Computation times.} Given an AMD Ryzen 9 7950X as CPU, the sum of training times for the LFC approach \texttt{coin}~\citep{nguyen2024noisy} with its default HPC and the times for computing all empirical risks in $\mathcal{R}$ (see Section~\ref{subsec:hps-criteria}) are reported in seconds across a $K=5$-fold cross validation for the dataset variant \texttt{dopanim-full}.
    }
    \label{fig:computation-times}
\end{wrapfigure}
We analyze the time complexity when evaluating an HPS criterion as one part of the HPO pipeline. 
More specifically, Table~\ref{tab:time-complexity-analysis} presents the worst-case time complexity of individual steps
involved by the non-default HPS criteria in $\mathrm{O}$-notation.
Typically, the training time complexity is the most expensive part of HPO. Figure~\ref{fig:computation-times} provides an example of this. Here, the training times are almost $100$ times higher than computing the empirical risks (including previous steps such as prediction computation) for evaluating the HPS criterion \textsc{ens}. Even when using a GPU, the difference can be higher for more complex neural network architectures, e.g., when we would fully fine-tune a backbone. Accordingly, differences between time complexities when evaluating the different HPS criteria are of minor importance. Since the training time complexity is identical for each HPS criterion, we do not further analyze this complexity and focus only on the steps related to the actual evaluation of an HPS criterion in the following. 

\begin{table}[!ht]
    \scriptsize
    \centering
    \setlength{\tabcolsep}{10.5pt}
    \def\arraystretch{1}
    \caption{
        \textbf{Worst-case time complexity analysis.} Each row corresponds to one HPS criterion, with its name indicated by the first column. The other column headings refer to the steps of evaluating an HPS criterion. The colors distinguish between {\color{baseline}baseline}, {\color{aggregation}aggregation-level}, {\color{crowd}crowd-level}, and {\color{ensemble}ensemble-based} criteria. There are $N$ observed instances, $M$ crowdworkers, $|\Lambda|$ candidate HPCs, and $C$ classes. The variable $T_{\boldsymbol{\theta}}$ denotes the worst-case time complexity of performing a single forward pass for any model trained by an LFC approach. A~$\star$ marks criteria that had access to the true validation labels.
    }
    \label{tab:time-complexity-analysis}
    \begin{tabular}{lrrrrrr}
        \toprule
        
        \multicolumn{1}{c}{\textbf{Hyperparameter}}  & \multicolumn{1}{c}{\textbf{Prediction}} & \multicolumn{1}{c}{\textbf{True Label Posterior}}   & \multicolumn{1}{c}{\textbf{Empirical Risk}} & \multicolumn{1}{c}{\textbf{Winner}} \\
        \multicolumn{1}{c}{\textbf{Selection Criterion}} & \multicolumn{1}{c}{\textbf{Computation}} & \multicolumn{1}{c}{\textbf{Probability Estimation}} & \multicolumn{1}{c}{\textbf{Computation}} & \multicolumn{1}{c}{\textbf{Selection}} \\

        \midrule

        \rowcolor{baseline!15}$\textsc{true}^\star$ & $\mathrm{O}\left(|\Lambda| \cdot N \cdot T_{\boldsymbol{\theta}} \right)$ & N/A & $\mathrm{O}(|\Lambda| \cdot N \cdot C)$ & $\mathrm{O}(|\Lambda|)$ \\
        \rowcolor{aggregation!15}\textsc{aeu}  & $\mathrm{O}\left(|\Lambda| \cdot N \cdot T_{\boldsymbol{\theta}} \right)$ & $\mathrm{O}\left(|\Lambda| \cdot N \cdot (M + C)\right)$ & $\mathrm{O}\left(|\Lambda| \cdot N \cdot C\right)$ & $\mathrm{O}(|\Lambda|)$ \\
        \rowcolor{aggregation!15} \textsc{aec}  & $\mathrm{O}\left(|\Lambda| \cdot N \cdot T_{\boldsymbol{\theta}} \right)$ & $\mathrm{O}\left(|\Lambda| \cdot N \cdot (M + C)\right)$ & $\mathrm{O}\left(|\Lambda| \cdot N \cdot C\right)$ & $\mathrm{O}(|\Lambda|)$ \\
        \rowcolor{aggregation!15}\textsc{ale}  & $\mathrm{O}\left(|\Lambda| \cdot N \cdot M \cdot T_{\boldsymbol{\theta}} \right)$ & $\mathrm{O}\left(|\Lambda| \cdot N \cdot (M + C)\right)$ & $\mathrm{O}\left(|\Lambda| \cdot N \cdot C\right)$ & $\mathrm{O}(|\Lambda|)$ \\
        \rowcolor{aggregation!15} \textsc{alc}  & $\mathrm{O}\left(|\Lambda| \cdot N \cdot M \cdot T_{\boldsymbol{\theta}} \right)$ & $\mathrm{O}\left(|\Lambda| \cdot N \cdot (M + C)\right)$ & $\mathrm{O}\left(|\Lambda| \cdot N \cdot C\right)$ & $\mathrm{O}(|\Lambda|)$ \\
        \rowcolor{crowd!15} \textsc{cxu}  & $\mathrm{O}\left(|\Lambda| \cdot N \cdot M \cdot T_{\boldsymbol{\theta}} \right)$ & N/A & $\mathrm{O}\left(|\Lambda| \cdot N \cdot M \cdot C\right)$ & $\mathrm{O}(|\Lambda|)$ \\
        \rowcolor{crowd!15} \textsc{cec}  & $\mathrm{O}\left(|\Lambda| \cdot N \cdot M \cdot T_{\boldsymbol{\theta}} \right)$ & $\mathrm{O}\left(|\Lambda| \cdot N \cdot (M + C)\right)$ & $\mathrm{O}\left(|\Lambda| \cdot N \cdot M \cdot C\right)$ & $\mathrm{O}(|\Lambda|)$ \\
        \rowcolor{crowd!15} \textsc{clc}  & $\mathrm{O}\left(|\Lambda| \cdot N \cdot M \cdot T_{\boldsymbol{\theta}} \right)$ & $\mathrm{O}\left(|\Lambda| \cdot N \cdot (M + C)\right)$ & $\mathrm{O}\left(|\Lambda| \cdot N \cdot M \cdot C\right)$ & $\mathrm{O}(|\Lambda|)$ \\
        \rowcolor{ensemble!15} \textsc{ens}  & $\mathrm{O}\left(|\Lambda| \cdot N \cdot M \cdot T_{\boldsymbol{\theta}} \right)$ & $\mathrm{O}\left(J \cdot |\Lambda| \cdot N \cdot (M + C)\right)$ & $\mathrm{O}\left(J \cdot |\Lambda| \cdot N \cdot M \cdot C\right)$ & $\mathrm{O}(J \cdot |\Lambda| \cdot \ln|\Lambda|)$ \\
        \bottomrule
    \end{tabular}
\end{table}

\paragraph{Prediction Computation} Computing model predictions represents the first step of an HPO criterion's evaluation. This step's complexity is affected by which of the three models, i.e., data classification model~$\boldsymbol{f}_{\boldsymbol{\theta}}$, crowdworker classification model~$\boldsymbol{g}_{\boldsymbol{\theta}}$, or crowdworker performance model~$h_{\boldsymbol{\theta}}$, is required to make predictions. If only the data classification model is involved, we need class probability predictions for each validation instance. When performing a $K$-fold cross-validation, each instance is used for validation once. Accordingly, this step scales linearly with $N$ as the number of observed instances and $|\Lambda|$ as the number of candidate HPCs. The variable $T_{\boldsymbol{\theta}}$ denotes the time complexity for obtaining a prediction from one of the three models. The step's complexity increases if the crowdworker classification or crowdworker performance model is involved, because then we need to compute for each of the $ N\cdot M$ instance-crowdworker pairs a prediction in the worst case, where each crowdworker has labeled every observed instance.

\paragraph{True Label Posterior Probability Estimation} The HPS criteria \textsc{true} and \textsc{cxu} do not require any posterior probabilities. For the other criteria, the posterior probabilities are obtained according to Eq.~\eqref{eq:simplified-class-posteriors} by iterating over all of the $M$ (worst case) observed class labels per instance and normalizing across all $C$ classes. Accordingly, this computation has a complexity of $\mathrm{O}(M +C)$ and must be repeated for each of the $N$ observed instances and for each of the $|\Lambda|$ candidate HPCs. For the ensemble-based criterion \textsc{ens}, this process is additionally repeated for each of its $J$ members.

\paragraph{Empirical Risk Computation} Given the probabilistic model predictions and the targets, this step refers to computing the zero-one losses, which are then averaged to obtain the empirical risk measurement. For the HPS criterion \textsc{true} and the aggregation level HPS criteria, we need to find one of the $C$ class labels with the maximum probability predicted by the classification model for each of the $N$ observed instances. In the case of the crowd-level criteria, this step extends to each of the $N \cdot M$ instance-crowdworker pairs. We need to repeat this step for each of the $|\Lambda|$ candidate HPCs. For the HPS criterion \textsc{ens}, this step must also cover each of its $J$ members.

\paragraph{Winner Selection} For nearly all HPS criteria, we find the winner HPC by selecting the one with the lowest empirical risk from the $|\Lambda|$ candidate HPCs. Only for the ensemble-based HPS criterion, we need to compute the Borda count. In this case, the complexity of sorting the $|\Lambda|$ candidate HPCs according to each of their $J$ empirical risk measurements is dominating the subsequent summation and finding of the minimum. In other words, the reported worst-case complexity corresponds to sorting $J$ lists of $|\Lambda|$ elements with the merge sort algorithm. 
\section{Experimental Evaluation}
\label{app:supplementary-experimental-evaluation}
This appendix describes our computational resources for experimentation and provides additional results to the experimental evaluation presented in Section~\ref{sec:empirical_evaluation}.

\subsection{Computational Resources}
Table~\ref{tab:zero-one-loss-results-1} lists the results for all $35$ dataset variants, $13$ LFC approaches, and $11$ HPS criteria. Moreover, we report the results for training with the ground truth (\texttt{gt}) class labels as the upper baseline approach. Each test zero-one loss value is the result of determining the selected HPC from a candidate set of $|\Lambda|=51$ HPCs via a $K=5$-fold cross-validation, of which $50$ HPCs have been generated via Sobol squences~\citep{sobol1998quasi} and one corresponds to the default (\textsc{def}) HPC. Subsequently, each selected HPC is tested with $5$ different initializations of the respective neural network architecture. In total, this corresponds to almost 
\begin{equation}
    \underbrace{5}_{\text{datasets}} \cdot \underbrace{7}_{\text{variants}} \cdot \underbrace{14}_{\text{approaches}}  \left(\underbrace{51}_{\text{HPCs}} \cdot \underbrace{5}_{\text{training \& validation}} + \underbrace{11}_{\text{criteria}} \cdot \underbrace{5}_{\text{training \& testing}}\right) = 151{,}900    
\end{equation}
training and evaluation runs. We executed all runs on a compute cluster equipped with NVIDIA A100 and V100 GPU servers, which we used to pre-compute the image and text embeddings. Almost all other computations were executed with AMD EPYC 7742 CPU servers. The time measurements from Appendix~\ref{subapp:time-complexity} are the only exception because they refer to an AMD Ryzen 9 7950X as the CPU of local workstation.

\subsection{Supplementary Results}
\label{subapp:supplementary-results}
We present supplementary results regarding the HPS criteria' and LFC approaches' performances. Further, we ablate the members' impact on the ensemble-based criterion \textsc{ens}.

\paragraph{Hyperparameter Selection Criteria} Figure~\ref{fig:criteria-relative-reduction-per-approach} depicts (analog graphic to Figure~\ref{fig:approaches-diff-per-criterion}) the criteria's relative zero-one loss reductions compared to the lower baseline criterion \textsc{def}. The results confirm our observations for the absolute zero-one loss reductions, namely, that a criterion's benefit varies across the LFC approaches, while \textsc{ens} is the most robust one for noisy crowd-labeled validation data.
\begin{figure}[!ht]
    \centering
    \includegraphics[width=\textwidth]{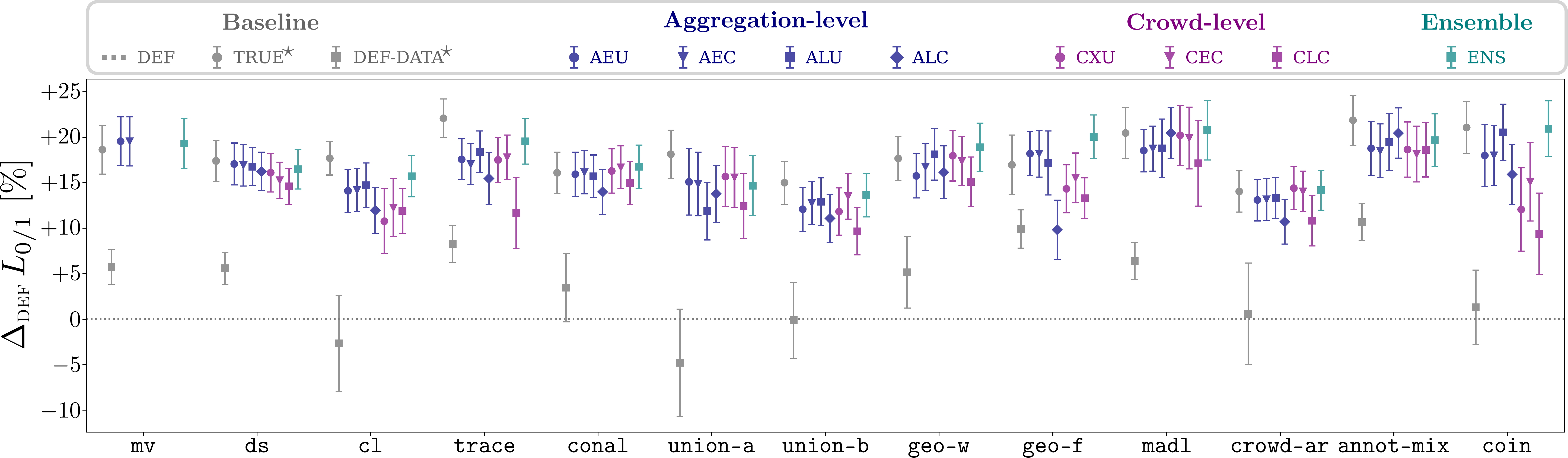}
    \caption{
        \textbf{Relative zero–one loss reductions of HPS criteria.} For each LFC approach ($x$-axis), the scatter plot displays the mean and standard error of a criterion's reduction in zero–one loss ($y$-axis) as a percentage [$\%$] relative to training with default (\textsc{def}) HPC. Higher reductions correspond to greater improvements. A~$\star$ marks criteria that had access to the true validation labels.
    }
    \label{fig:criteria-relative-reduction-per-approach}
\end{figure}

Beyond reporting average reductions in zero-one loss relative to the default HPC, we also examine how performance is distributed across pairs of HPS criteria. Figure~\ref{fig:criteria-winning-matrix}, therefore, presents a pairwise win-rate matrix, where each cell shows the proportion of experiments on which the row criterion beats the column criterion. The matrix shows that the \textsc{true} criterion dominates all alternatives: its win-rate is consistently higher than its loss-rate. Several cell pairs do not sum to one, indicating ties in which the two criteria selected HPCs with identical performance. For the default baselines, \textsc{def} loses to every other criterion, confirming that a single global default HPC per approach is inadequate. \textsc{def-data}, which transfers HPCs optimized on classification tasks without any noisy labels, performs somewhat better yet still lags behind. Accordingly, optimizing the HPC of a standard classification model per dataset with access to true validation labels does not satisfy the LFC approaches' individual requirements. Among criteria that explicitly account for noisy crowd-labeled validation data, the ensemble-based criterion performs best: against every rival except \textsc{true}, it wins more often than it loses. Together with \textsc{true}, it is the only criterion whose loss-rate versus \textsc{def} stays below $\si{10}{\%}$.
\begin{figure}[!h]
    \centering
    \includegraphics[width=0.95\textwidth]{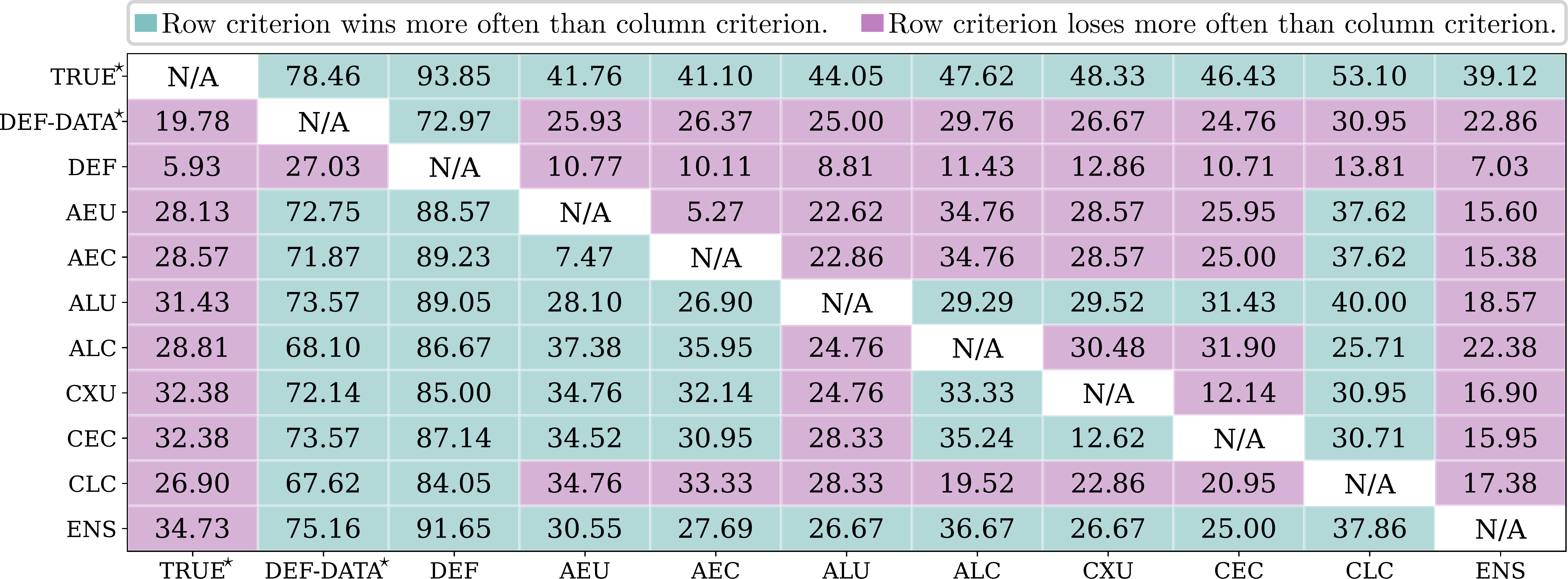}
    \caption{
        \textbf{Pairwise win-rate matrix for HPS criteria.} Across all datasets and LFC approaches, a cell reports the percentage [$\%$] on which the row criterion selects an HPC outperforming the HPC selected by the column criterion. A cell's color decodes whether the row criterion has more {\color{yescolor}wins} or {\color{nocolor}loses} than the column criterion. The diagonal shows \textit{not applicable} (N/A) because a criterion is not compared with itself. A~$\star$ marks criteria that had access to the true validation labels.
    }
    \label{fig:criteria-winning-matrix}
\end{figure}

\paragraph{Learning from Crowds Approaches} Figure~\ref{fig:approaches-relative-reduction-per-criterion} depicts (analog to Figure~\ref{fig:approaches-diff-per-criterion}) the approaches' relative zero-one loss reductions compared to the lower baseline approach \texttt{mv}[\textsc{def}], which corresponds to training with the majority vote labels and the default HPC. The results confirm our observations for the absolute zero-one loss reductions, namely, that one-stage approaches can achieve high performance gains, while the choice of the criterion affects the comparison of the approaches with each other.
\begin{figure}[!ht]
    \centering
    \includegraphics[width=\textwidth]{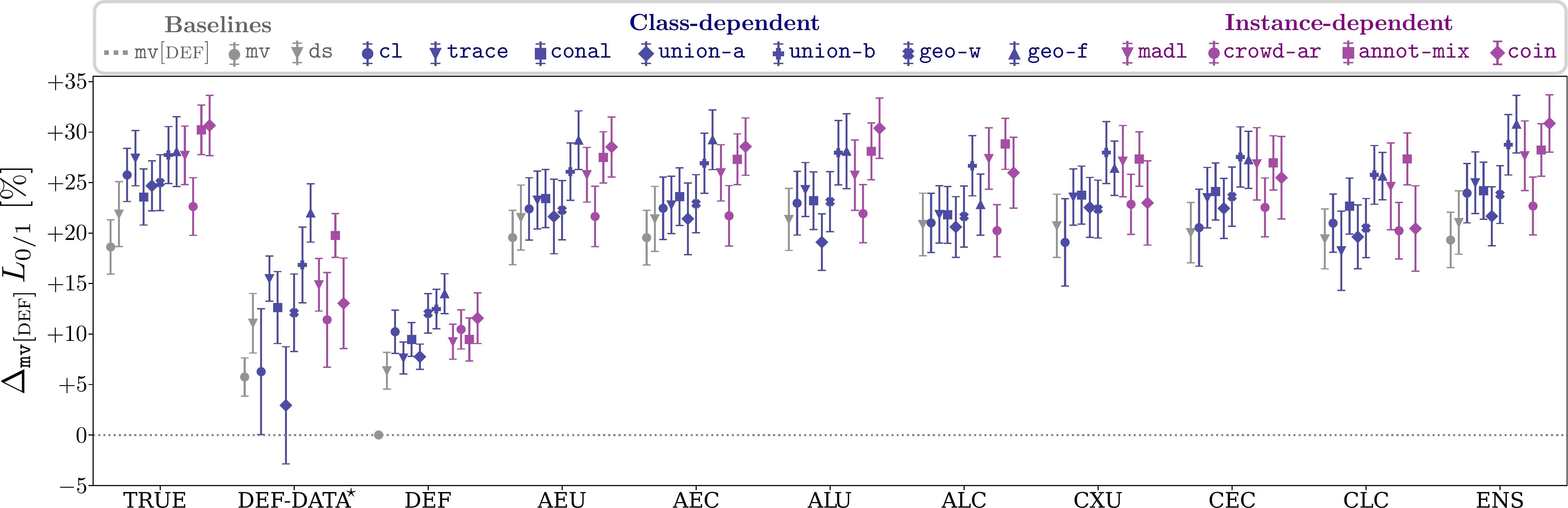}
    \caption{
        \textbf{Relative zero–one loss reductions of LFC approaches.} For each HPS criterion ($x$-axis), the scatter plot displays the mean and standard error of an LFC approach's reduction in zero–one loss ($y$-axis) as percentage [$\%$] compared to majority voting (\texttt{mv}) trained with the default (\textsc{def}) HPC. Higher reductions correspond to greater improvements. A~$\star$ marks criteria that had access to the true validation labels.
    }
    \label{fig:approaches-relative-reduction-per-criterion}
\end{figure}

Beyond reporting average reductions in zero-one loss relative to the lower baseline approach \texttt{mv}[\textsc{def}], we also examine how performance is distributed across pairs of approaches. Figure~\ref{fig:criteria-winning-matrix}, therefore, presents a pairwise win-rate matrix with \textsc{ens} as the HPS criterion, where each cell shows the proportion of dataset variants on which the row approach beats the column approach. Several cell pairs do not sum to one, indicating ties in which the two approaches reached identical performance. The two-stage baseline \texttt{mv} loses to every other approach, confirming that estimating crowdworkers' performances is beneficial. The other two-stage approach \texttt{ds},~\citep{dawid1979maximum}, which estimates crowdworker performances to aggregated labels in the first stage, performs somewhat better yet still lags behind. Accordingly, one-stage LFC approaches combining the crowdworker performance and true label estimation in one joint training lead to superior performances. Among these one-stage approaches, the \texttt{coin} approach~\citep{nguyen2024noisy} dominates all alternatives: its win-rate is consistently higher than its loss-rate. Together with \texttt{geo-f}~\citep{ibrahim2023deep}, it is the only approach whose loss-rate versus \texttt{mv} stays below $\si{10}{\%}$.

\begin{figure}[!ht]
    \centering
    \includegraphics[width=0.95\textwidth]{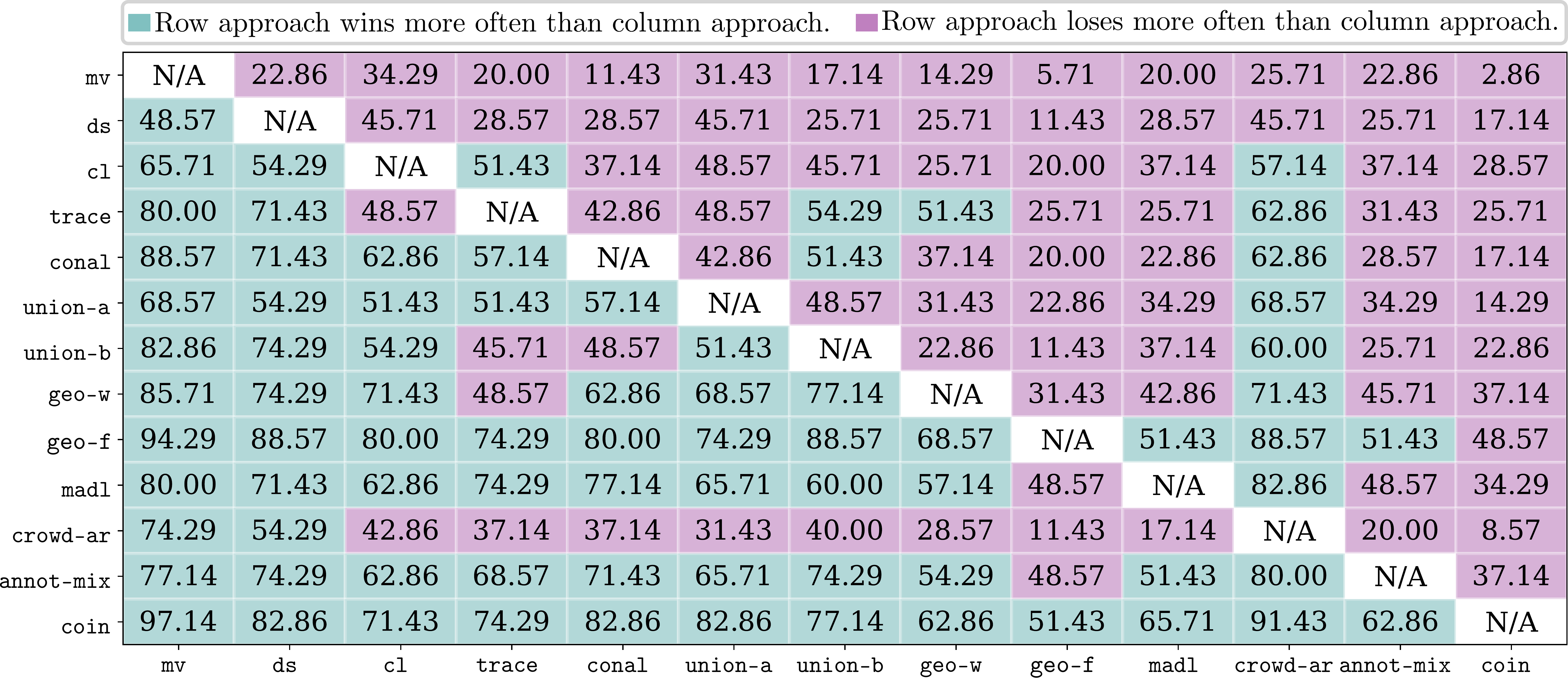}
    \caption{
        \textbf{Pairwise win-rate matrix for LFC approaches with} \textsc{ens} \textbf{as HPS criterion.} Across all dataset variants, a cell reports the percentage [$\%$] on which the row approach outperforms the column approach. A cell's color decodes whether the row approach has more {\color{yescolor}wins} or {\color{nocolor}loses} than the column approach. The diagonal shows \textit{not applicable} (N/A) because an approach is not compared with itself.
    }
    \label{fig:approaches-winning-matrix}
\end{figure}

\paragraph{Ablation Study} We ablate the individual members' impact on the ensemble-based HPS criterion \textsc{ens}. Therefore, we demonstrate the effect of removing risk measures from $\mathcal{R}$ (see Section~\ref{subsec:hps-criteria}) in the ranking order of their respective criterion from Table~\ref{tab:model-selection-criteria-results}, starting with the worst one. We observe that removing members mostly leads to a lower absolute and relative reduction of the zero-one loss compared to the \textsc{def} criterion. This observation confirms the importance of individual members. Nevertheless, there might be unexplored combinations of multiple members leading to higher reductions than the full ensemble. 
\begin{table}[!ht]
    \scriptsize
    \centering
    \setlength{\tabcolsep}{8.85pt}
    \def\arraystretch{1}
    \caption{
        \textbf{Ablation study for} \textsc{ens}\textbf{.} In comparison to \textsc{def} as HPS criterion, each column reports the zero-one loss reductions (absolute as percentage points [$\%_p$] and relative as percentages [$\%$]) for one subset of risk measures in $\mathcal{R}$ when employing the HPS criterion \textsc{ens}. The full set is given in the leftmost column, while each succeeding column shows the results after removing one criterion with its associated risk measure. The rightmost column refers to the case when only one member is remaining, corresponding to the criterion \textsc{alu}. The colors distinguish between {\color{baseline}baseline}, {\color{aggregation}aggregation-level}, {\color{crowd}crowd-level}, and {\color{ensemble}ensemble-based} criteria. Means and standard errors are computed over all combinations of dataset variants and LFC approaches (excluding the approach \texttt{mv} that is not compatible with each criterion's empirical risk measure). The arrows indicate that higher ($\uparrow$) values are better. The \textBF{best} and \underline{second best} value is marked per result type. 
    }
    \begin{tabular*}{\textwidth}{@{\extracolsep{\fill}} cccccccc }
        \toprule
        
        \color{ensemble}\textsc{ens} & \color{crowd}$-\textsc{clc}$ & \color{aggregation}$-\textsc{aeu}$ & \color{aggregation}$-\textsc{alc}$ & \color{aggregation}$-\textsc{aec}$ & \color{crowd}$-\textsc{cxu}$ & ${\color{crowd}-\textsc{cec}}={\color{aggregation}\textsc{alu}}$ \\
        
        \hline
        \rowcolor{yescolor!10} \multicolumn{8}{c}{\scriptsize Absolute Zero-one Loss Reductions Compared to \textsc{def} ($\Delta_{\textsc{def}} \,L_{0/1}$ [$\%_p$] $\uparrow$)}  \\
        \hline
        $+\phantom{0}\mathbf{5.15}$ & $+\phantom{0}5.03$ & $+\phantom{0}\underline{5.09}$ & $+\phantom{0}4.96$ & $+\phantom{0}4.84$ & $+\phantom{0}4.96$ & $+\phantom{0}4.73$ \\
        $\pm\phantom{0}0.26$ & $\pm\phantom{0}0.26$ & $\pm\phantom{0}0.27$ & $\pm\phantom{0}0.26$ & $\pm\phantom{0}0.27$ & $\pm\phantom{0}0.27$ & $\pm\phantom{0}0.26$  \\        
        \hline
        \rowcolor{yescolor!10} \multicolumn{8}{c}{\scriptsize Relative Zero-one Loss Reductions Compared to \textsc{def} ($\Delta_{\textsc{def}}\,L_{0/1}$ [$\%$] $\uparrow$)}  \\
        \hline
        $+\mathbf{17.60}$ & $+\underline{17.49}$ & $+17.43$ & $+17.17$ & $+16.56$ & $+17.14$ &  $+16.48$ \\
        $\pm\phantom{0}0.77$ & $\pm\phantom{0}0.76$ & $\pm\phantom{0}0.78$ & $\pm\phantom{0}0.77$ & $\pm\phantom{0}0.78$ & $\pm\phantom{0}0.78$ & $\pm\phantom{0}0.80$  \\
        \bottomrule
    \end{tabular*}
\end{table}

\newpage
\begin{table}[!p]
    \centering
    \scriptsize
    \setlength{\tabcolsep}{1.7pt}
    \def\arraystretch{0.907}
    \caption{Zero-one loss results [\%] (part I) — The first column lists the LFC approaches and the remaining columns the HPS criteria. Each criterion selects the estimated best HPC per approach, and results are reported as means with standard deviations. The {\textBF{best-performing}} approach per column (excluding \texttt{gt}) and the \underline{best-performing} selection criterion per row (excluding \textsc{true}) are highlighted. The symbol $\star$ marks criteria with access to true validation labels. Some criteria are \textit{not applicable} (N/A) to all approaches.}
    \label{tab:zero-one-loss-results-1}
    \resizebox{\textwidth}{!}{%
\begin{tabular}{@{} lccccccccccc @{}}
        \toprule
        
        \multirow{2}{*}{\textBF{Approach}} & \multicolumn{3}{c}{\color{baseline}\textBF{Baseline}} & \multicolumn{4}{c}{\color{aggregation}\textBF{Aggregation-level}} & \multicolumn{3}{c}{\color{crowd}\textBF{Crowd-level}} & \multicolumn{1}{c}{\color{ensemble}\textBF{Ensemble}} \\
        
        \cmidrule(lr){2-4} \cmidrule(lr){5-8} \cmidrule(lr){9-11} \cmidrule(lr){12-12}

        & \multicolumn{1}{c}{{\color{baseline}$\textsc{true}^\star$}} & \multicolumn{1}{c}{\color{baseline}$\textsc{def-data}^\star$} & \multicolumn{1}{c}{\color{baseline}\textsc{def}} & \multicolumn{1}{c}{\color{aggregation}\textsc{aeu}} & \multicolumn{1}{c}{\color{aggregation}\textsc{aec}} & \multicolumn{1}{c}{\color{aggregation}\textsc{alu}} & \multicolumn{1}{c}{\color{aggregation}\textsc{alc}}  & \multicolumn{1}{c}{\color{crowd}\textsc{cxu}} & \multicolumn{1}{c}{\color{crowd}\textsc{cec}} & \multicolumn{1}{c}{\color{crowd}\textsc{clc}} & \multicolumn{1}{c}{\color{ensemble}\textsc{ens}}\\

        \hline
        \rowcolor{yescolor!10} \multicolumn{12}{c}{\texttt{mgc-worst-1}} \\
        \hline

        \texttt{gt} & $20.27_{\pm 0.83}$ & $20.27_{\pm 0.83}$ & $24.60_{\pm 1.12}$ & $21.00_{\pm 1.94}$ & $21.00_{\pm 1.94}$ & N/A & N/A & N/A & N/A & N/A & $21.00_{\pm 1.94}$ \\
        \texttt{mv} & $81.27_{\pm 1.32}$ & $86.93_{\pm 1.30}$ & $81.47_{\pm 0.96}$ & $\underline{79.73}_{\pm 1.75}$ & $\underline{79.73}_{\pm 1.75}$ & N/A & N/A & N/A & N/A & N/A & $\underline{79.73}_{\pm 1.75}$ \\
        \texttt{ds} & $81.27_{\pm 1.32}$ & $86.93_{\pm 1.30}$ & $81.47_{\pm 0.96}$ & $\underline{79.73}_{\pm 1.75}$ & $\underline{79.73}_{\pm 1.75}$ & $\underline{79.73}_{\pm 1.75}$ & $\underline{79.73}_{\pm 1.75}$ & $\underline{79.73}_{\pm 1.75}$ & $\underline{79.73}_{\pm 1.75}$ & $\underline{79.73}_{\pm 1.75}$ & $\underline{79.73}_{\pm 1.75}$ \\
\hline
    \end{tabular}
    }
\end{table}

\begin{table}[p]
    \centering
    \scriptsize
    \setlength{\tabcolsep}{1.7pt}
    \def\arraystretch{0.925}
        \ContinuedFloat
    \caption{Zero-one loss results (part II) -- Continued from the previous page.}
    \label{tab:zero-one-loss-results-2}
    \resizebox{\textwidth}{!}{%
    \begin{tabular}{@{} lccccccccccc @{}}
        \toprule
        
        \multirow{2}{*}{\textBF{Approach}} & \multicolumn{3}{c}{\color{baseline}\textBF{Baseline}} & \multicolumn{4}{c}{\color{aggregation}\textBF{Aggregation-level}} & \multicolumn{3}{c}{\color{crowd}\textBF{Crowd-level}} & \multicolumn{1}{c}{\color{ensemble}\textBF{Ensemble}} \\
        
        \cmidrule(lr){2-4} \cmidrule(lr){5-8} \cmidrule(lr){9-11} \cmidrule(lr){12-12}

        & \multicolumn{1}{c}{{\color{baseline}$\textsc{true}^\star$}} & \multicolumn{1}{c}{\color{baseline}$\textsc{def-data}^\star$} & \multicolumn{1}{c}{\color{baseline}\textsc{def}} & \multicolumn{1}{c}{\color{aggregation}\textsc{aeu}} & \multicolumn{1}{c}{\color{aggregation}\textsc{aec}} & \multicolumn{1}{c}{\color{aggregation}\textsc{alu}} & \multicolumn{1}{c}{\color{aggregation}\textsc{alc}}  & \multicolumn{1}{c}{\color{crowd}\textsc{cxu}} & \multicolumn{1}{c}{\color{crowd}\textsc{cec}} & \multicolumn{1}{c}{\color{crowd}\textsc{clc}} & \multicolumn{1}{c}{\color{ensemble}\textsc{ens}}\\

        \hline
        \rowcolor{yescolor!10} \multicolumn{12}{c}{\texttt{mgc-rand-2}} \\
        \hline

        \texttt{gt} & $19.07_{\pm 0.28}$ & $19.07_{\pm 0.28}$ & $24.60_{\pm 1.12}$ & $18.93_{\pm 0.76}$ & $18.93_{\pm 0.76}$ & N/A & N/A & N/A & N/A & N/A & $18.93_{\pm 0.76}$ \\
        \texttt{mv} & $40.67_{\pm 0.78}$ & $43.60_{\pm 1.21}$ & $\underline{38.20}_{\pm 2.06}$ & $40.67_{\pm 0.78}$ & $40.67_{\pm 0.78}$ & N/A & N/A & N/A & N/A & N/A & $40.67_{\pm 0.78}$ \\
        \texttt{ds} & $39.53_{\pm 2.09}$ & $44.53_{\pm 1.52}$ & $41.07_{\pm 1.79}$ & $\underline{38.40}_{\pm 0.72}$ & $\underline{38.40}_{\pm 0.72}$ & $39.53_{\pm 2.09}$ & $40.27_{\pm 2.89}$ & $42.13_{\pm 0.77}$ & $42.13_{\pm 0.77}$ & $42.13_{\pm 0.77}$ & $39.53_{\pm 2.09}$ \\
\hline
    \end{tabular}
    }
\end{table}

\begin{table}[p]
    \centering
    \scriptsize
    \setlength{\tabcolsep}{1.7pt}
    \def\arraystretch{0.925}
        \ContinuedFloat
    \caption{Zero-one loss results (part III) -- Continued from the previous page.}
    \label{tab:zero-one-loss-results-3}
    \resizebox{\textwidth}{!}{%
    \begin{tabular}{@{} lccccccccccc @{}}
        \toprule
        
        \multirow{2}{*}{\textBF{Approach}} & \multicolumn{3}{c}{\color{baseline}\textBF{Baseline}} & \multicolumn{4}{c}{\color{aggregation}\textBF{Aggregation-level}} & \multicolumn{3}{c}{\color{crowd}\textBF{Crowd-level}} & \multicolumn{1}{c}{\color{ensemble}\textBF{Ensemble}} \\
        
        \cmidrule(lr){2-4} \cmidrule(lr){5-8} \cmidrule(lr){9-11} \cmidrule(lr){12-12}

        & \multicolumn{1}{c}{{\color{baseline}$\textsc{true}^\star$}} & \multicolumn{1}{c}{\color{baseline}$\textsc{def-data}^\star$} & \multicolumn{1}{c}{\color{baseline}\textsc{def}} & \multicolumn{1}{c}{\color{aggregation}\textsc{aeu}} & \multicolumn{1}{c}{\color{aggregation}\textsc{aec}} & \multicolumn{1}{c}{\color{aggregation}\textsc{alu}} & \multicolumn{1}{c}{\color{aggregation}\textsc{alc}}  & \multicolumn{1}{c}{\color{crowd}\textsc{cxu}} & \multicolumn{1}{c}{\color{crowd}\textsc{cec}} & \multicolumn{1}{c}{\color{crowd}\textsc{clc}} & \multicolumn{1}{c}{\color{ensemble}\textsc{ens}}\\

        \hline
        \rowcolor{yescolor!10} \multicolumn{12}{c}{\texttt{label-me-worst-2}} \\
        \hline

        \texttt{gt} & $6.43_{\pm 0.44}$ & $6.43_{\pm 0.44}$ & $6.31_{\pm 0.27}$ & $6.75_{\pm 0.14}$ & $6.75_{\pm 0.14}$ & N/A & N/A & N/A & N/A & N/A & $6.75_{\pm 0.14}$ \\
        \texttt{mv} & $18.08_{\pm 0.77}$ & $24.97_{\pm 0.47}$ & $22.20_{\pm 0.95}$ & $\underline{17.41}_{\pm 0.70}$ & $\underline{17.41}_{\pm 0.70}$ & N/A & N/A & N/A & N/A & N/A & $\underline{17.41}_{\pm 0.70}$ \\
        \texttt{ds} & $18.00_{\pm 0.51}$ & $25.57_{\pm 0.50}$ & $22.44_{\pm 1.01}$ & $18.13_{\pm 1.64}$ & $\underline{17.59}_{\pm 0.66}$ & $17.66_{\pm 0.67}$ & $17.66_{\pm 0.67}$ & $17.66_{\pm 0.67}$ & $17.66_{\pm 0.67}$ & $17.66_{\pm 0.67}$ & $17.66_{\pm 0.67}$ \\
\hline
    \end{tabular}
    }
\end{table}

\begin{table}[p]
    \centering
    \scriptsize
    \setlength{\tabcolsep}{1.7pt}
    \def\arraystretch{0.925}
        \ContinuedFloat
    \caption{Zero-one loss results (part IV) -- Continued from the previous page.}
    \label{tab:zero-one-loss-results-4}
    \resizebox{\textwidth}{!}{%
    \begin{tabular}{@{} lccccccccccc @{}}
        \toprule
        
        \multirow{2}{*}{\textBF{Approach}} & \multicolumn{3}{c}{\color{baseline}\textBF{Baseline}} & \multicolumn{4}{c}{\color{aggregation}\textBF{Aggregation-level}} & \multicolumn{3}{c}{\color{crowd}\textBF{Crowd-level}} & \multicolumn{1}{c}{\color{ensemble}\textBF{Ensemble}} \\
        
        \cmidrule(lr){2-4} \cmidrule(lr){5-8} \cmidrule(lr){9-11} \cmidrule(lr){12-12}

        & \multicolumn{1}{c}{{\color{baseline}$\textsc{true}^\star$}} & \multicolumn{1}{c}{\color{baseline}$\textsc{def-data}^\star$} & \multicolumn{1}{c}{\color{baseline}\textsc{def}} & \multicolumn{1}{c}{\color{aggregation}\textsc{aeu}} & \multicolumn{1}{c}{\color{aggregation}\textsc{aec}} & \multicolumn{1}{c}{\color{aggregation}\textsc{alu}} & \multicolumn{1}{c}{\color{aggregation}\textsc{alc}}  & \multicolumn{1}{c}{\color{crowd}\textsc{cxu}} & \multicolumn{1}{c}{\color{crowd}\textsc{cec}} & \multicolumn{1}{c}{\color{crowd}\textsc{clc}} & \multicolumn{1}{c}{\color{ensemble}\textsc{ens}}\\

        \hline
        \rowcolor{yescolor!10} \multicolumn{12}{c}{\texttt{label-me-rand-v}} \\
        \hline

        \texttt{gt} & $6.35_{\pm 0.39}$ & $6.35_{\pm 0.39}$ & $6.31_{\pm 0.27}$ & $6.35_{\pm 0.39}$ & $6.43_{\pm 0.27}$ & N/A & N/A & N/A & N/A & N/A & $6.43_{\pm 0.27}$ \\
        \texttt{mv} & $15.19_{\pm 0.80}$ & $17.96_{\pm 0.86}$ & $19.68_{\pm 0.57}$ & $\underline{15.19}_{\pm 0.80}$ & $\underline{15.19}_{\pm 0.80}$ & N/A & N/A & N/A & N/A & N/A & $\underline{15.19}_{\pm 0.80}$ \\
        \texttt{ds} & $14.34_{\pm 0.65}$ & $17.42_{\pm 0.73}$ & $18.60_{\pm 0.61}$ & $15.54_{\pm 0.27}$ & $15.54_{\pm 0.27}$ & $15.54_{\pm 0.27}$ & $15.54_{\pm 0.27}$ & $\underline{13.91}_{\pm 0.53}$ & $\underline{13.91}_{\pm 0.53}$ & $17.58_{\pm 0.69}$ & $14.34_{\pm 0.65}$ \\
\hline
    \end{tabular}
    }
\end{table}

\begin{table}[p]
    \centering
    \scriptsize
    \setlength{\tabcolsep}{1.7pt}
    \def\arraystretch{0.925}
        \ContinuedFloat
    \caption{Zero-one loss results (part V) -- Continued from the previous page.}
    \label{tab:zero-one-loss-results-5}
    \resizebox{\textwidth}{!}{%
    \begin{tabular}{@{} lccccccccccc @{}}
        \toprule
        
        \multirow{2}{*}{\textBF{Approach}} & \multicolumn{3}{c}{\color{baseline}\textBF{Baseline}} & \multicolumn{4}{c}{\color{aggregation}\textBF{Aggregation-level}} & \multicolumn{3}{c}{\color{crowd}\textBF{Crowd-level}} & \multicolumn{1}{c}{\color{ensemble}\textBF{Ensemble}} \\
        
        \cmidrule(lr){2-4} \cmidrule(lr){5-8} \cmidrule(lr){9-11} \cmidrule(lr){12-12}

        & \multicolumn{1}{c}{{\color{baseline}$\textsc{true}^\star$}} & \multicolumn{1}{c}{\color{baseline}$\textsc{def-data}^\star$} & \multicolumn{1}{c}{\color{baseline}\textsc{def}} & \multicolumn{1}{c}{\color{aggregation}\textsc{aeu}} & \multicolumn{1}{c}{\color{aggregation}\textsc{aec}} & \multicolumn{1}{c}{\color{aggregation}\textsc{alu}} & \multicolumn{1}{c}{\color{aggregation}\textsc{alc}}  & \multicolumn{1}{c}{\color{crowd}\textsc{cxu}} & \multicolumn{1}{c}{\color{crowd}\textsc{cec}} & \multicolumn{1}{c}{\color{crowd}\textsc{clc}} & \multicolumn{1}{c}{\color{ensemble}\textsc{ens}}\\

        \hline
        \rowcolor{yescolor!10} \multicolumn{12}{c}{\texttt{dopanim-worst-v}} \\
        \hline

        \texttt{gt} & $10.74_{\pm 0.20}$ & $10.74_{\pm 0.20}$ & $10.52_{\pm 0.22}$ & $11.03_{\pm 0.18}$ & $11.31_{\pm 0.14}$ & N/A & N/A & N/A & N/A & N/A & $11.03_{\pm 0.18}$ \\
        \texttt{mv} & $34.12_{\pm 0.43}$ & $36.47_{\pm 0.49}$ & $41.50_{\pm 0.69}$ & $34.12_{\pm 0.43}$ & $\underline{34.09}_{\pm 0.73}$ & N/A & N/A & N/A & N/A & N/A & $\underline{34.09}_{\pm 0.73}$ \\
        \texttt{ds} & $29.73_{\pm 0.55}$ & $30.58_{\pm 0.46}$ & $35.22_{\pm 0.58}$ & $\underline{29.05}_{\pm 0.38}$ & $\underline{29.05}_{\pm 0.38}$ & $\underline{29.05}_{\pm 0.38}$ & $\underline{29.05}_{\pm 0.38}$ & $\underline{29.05}_{\pm 0.38}$ & $\underline{29.05}_{\pm 0.38}$ & $\underline{29.05}_{\pm 0.38}$ & $\underline{29.05}_{\pm 0.38}$ \\
\hline
    \end{tabular}
    }
\end{table}

\begin{table}[p]
    \centering
    \scriptsize
    \setlength{\tabcolsep}{1.7pt}
    \def\arraystretch{0.925}
        \ContinuedFloat
    \caption{Zero-one loss results (part VI) -- Continued from the previous page.}
    \label{tab:zero-one-loss-results-6}
        \resizebox{\textwidth}{!}{%
    \begin{tabular}{@{} lccccccccccc @{}}
        \toprule
        
        \multirow{2}{*}{\textBF{Approach}} & \multicolumn{3}{c}{\color{baseline}\textBF{Baseline}} & \multicolumn{4}{c}{\color{aggregation}\textBF{Aggregation-level}} & \multicolumn{3}{c}{\color{crowd}\textBF{Crowd-level}} & \multicolumn{1}{c}{\color{ensemble}\textBF{Ensemble}} \\
        
        \cmidrule(lr){2-4} \cmidrule(lr){5-8} \cmidrule(lr){9-11} \cmidrule(lr){12-12}

        & \multicolumn{1}{c}{{\color{baseline}$\textsc{true}^\star$}} & \multicolumn{1}{c}{\color{baseline}$\textsc{def-data}^\star$} & \multicolumn{1}{c}{\color{baseline}\textsc{def}} & \multicolumn{1}{c}{\color{aggregation}\textsc{aeu}} & \multicolumn{1}{c}{\color{aggregation}\textsc{aec}} & \multicolumn{1}{c}{\color{aggregation}\textsc{alu}} & \multicolumn{1}{c}{\color{aggregation}\textsc{alc}}  & \multicolumn{1}{c}{\color{crowd}\textsc{cxu}} & \multicolumn{1}{c}{\color{crowd}\textsc{cec}} & \multicolumn{1}{c}{\color{crowd}\textsc{clc}} & \multicolumn{1}{c}{\color{ensemble}\textsc{ens}}\\

        \hline
        \rowcolor{yescolor!10} \multicolumn{12}{c}{\texttt{dopanim-full}} \\
        \hline

        \texttt{gt} & $11.02_{\pm 0.28}$ & $11.02_{\pm 0.28}$ & $10.52_{\pm 0.22}$ & $10.57_{\pm 0.26}$ & $10.57_{\pm 0.26}$ & N/A & N/A & N/A & N/A & N/A & $10.57_{\pm 0.26}$ \\
        \texttt{mv} & $17.32_{\pm 0.58}$ & $17.82_{\pm 0.28}$ & $20.59_{\pm 0.17}$ & $\underline{17.32}_{\pm 0.58}$ & $\underline{17.32}_{\pm 0.58}$ & N/A & N/A & N/A & N/A & N/A & $\underline{17.32}_{\pm 0.58}$ \\
        \texttt{ds} & $16.82_{\pm 0.51}$ & $17.16_{\pm 0.27}$ & $19.07_{\pm 0.17}$ & $17.45_{\pm 0.29}$ & $17.45_{\pm 0.29}$ & $\underline{16.82}_{\pm 0.51}$ & $\underline{16.82}_{\pm 0.51}$ & $17.45_{\pm 0.29}$ & $17.33_{\pm 0.31}$ & $17.33_{\pm 0.31}$ & $17.45_{\pm 0.29}$ \\
\hline
    \end{tabular}
    }
\end{table}

\begin{table}[p]
    \centering
    \scriptsize
    \setlength{\tabcolsep}{1.7pt}
    \def\arraystretch{0.925}
        \ContinuedFloat
    \caption{Zero-one loss results (part VII) -- Continued from the previous page.}
    \label{tab:zero-one-loss-results-7}
        \resizebox{\textwidth}{!}{%
    \begin{tabular}{@{} lccccccccccc @{}}
        \toprule
        
        \multirow{2}{*}{\textBF{Approach}} & \multicolumn{3}{c}{\color{baseline}\textBF{Baseline}} & \multicolumn{4}{c}{\color{aggregation}\textBF{Aggregation-level}} & \multicolumn{3}{c}{\color{crowd}\textBF{Crowd-level}} & \multicolumn{1}{c}{\color{ensemble}\textBF{Ensemble}} \\
        
        \cmidrule(lr){2-4} \cmidrule(lr){5-8} \cmidrule(lr){9-11} \cmidrule(lr){12-12}

        & \multicolumn{1}{c}{{\color{baseline}$\textsc{true}^\star$}} & \multicolumn{1}{c}{\color{baseline}$\textsc{def-data}^\star$} & \multicolumn{1}{c}{\color{baseline}\textsc{def}} & \multicolumn{1}{c}{\color{aggregation}\textsc{aeu}} & \multicolumn{1}{c}{\color{aggregation}\textsc{aec}} & \multicolumn{1}{c}{\color{aggregation}\textsc{alu}} & \multicolumn{1}{c}{\color{aggregation}\textsc{alc}}  & \multicolumn{1}{c}{\color{crowd}\textsc{cxu}} & \multicolumn{1}{c}{\color{crowd}\textsc{cec}} & \multicolumn{1}{c}{\color{crowd}\textsc{clc}} & \multicolumn{1}{c}{\color{ensemble}\textsc{ens}}\\

        \hline
        \rowcolor{yescolor!10} \multicolumn{12}{c}{\texttt{reuters-rand-1}} \\
        \hline

        \texttt{gt} & $3.94_{\pm 0.15}$ & $3.94_{\pm 0.15}$ & $4.14_{\pm 0.07}$ & $4.05_{\pm 0.21}$ & $4.05_{\pm 0.21}$ & N/A & N/A & N/A & N/A & N/A & $4.05_{\pm 0.21}$ \\
        \texttt{mv} & $13.29_{\pm 0.67}$ & $25.86_{\pm 0.66}$ & $26.55_{\pm 0.58}$ & $\underline{15.16}_{\pm 0.66}$ & $\underline{15.16}_{\pm 0.66}$ & N/A & N/A & N/A & N/A & N/A & $\underline{15.16}_{\pm 0.66}$ \\
        \texttt{ds} & $13.29_{\pm 0.67}$ & $25.86_{\pm 0.66}$ & $26.55_{\pm 0.58}$ & $\underline{15.16}_{\pm 0.66}$ & $\underline{15.16}_{\pm 0.66}$ & $\underline{15.16}_{\pm 0.66}$ & $\underline{15.16}_{\pm 0.66}$ & $\underline{15.16}_{\pm 0.66}$ & $\underline{15.16}_{\pm 0.66}$ & $\underline{15.16}_{\pm 0.66}$ & $\underline{15.16}_{\pm 0.66}$ \\
\hline
    \end{tabular}
    }
\end{table}

\begin{table}[p]
    \centering
    \scriptsize
    \setlength{\tabcolsep}{1.7pt}
    \def\arraystretch{0.925}
        \ContinuedFloat
    \caption{Zero-one loss results (part VIII) -- Continued from the previous page.}
    \label{tab:zero-one-loss-results-8}
    \resizebox{\textwidth}{!}{%
    \begin{tabular}{@{} lccccccccccc @{}}
        \toprule
        
        \multirow{2}{*}{\textBF{Approach}} & \multicolumn{3}{c}{\color{baseline}\textBF{Baseline}} & \multicolumn{4}{c}{\color{aggregation}\textBF{Aggregation-level}} & \multicolumn{3}{c}{\color{crowd}\textBF{Crowd-level}} & \multicolumn{1}{c}{\color{ensemble}\textBF{Ensemble}} \\
        
        \cmidrule(lr){2-4} \cmidrule(lr){5-8} \cmidrule(lr){9-11} \cmidrule(lr){12-12}

        & \multicolumn{1}{c}{{\color{baseline}$\textsc{true}^\star$}} & \multicolumn{1}{c}{\color{baseline}$\textsc{def-data}^\star$} & \multicolumn{1}{c}{\color{baseline}\textsc{def}} & \multicolumn{1}{c}{\color{aggregation}\textsc{aeu}} & \multicolumn{1}{c}{\color{aggregation}\textsc{aec}} & \multicolumn{1}{c}{\color{aggregation}\textsc{alu}} & \multicolumn{1}{c}{\color{aggregation}\textsc{alc}}  & \multicolumn{1}{c}{\color{crowd}\textsc{cxu}} & \multicolumn{1}{c}{\color{crowd}\textsc{cec}} & \multicolumn{1}{c}{\color{crowd}\textsc{clc}} & \multicolumn{1}{c}{\color{ensemble}\textsc{ens}}\\

        \hline
        \rowcolor{yescolor!10} \multicolumn{12}{c}{\texttt{spc-worst-1}} \\
        \hline

        \texttt{gt} & $15.47_{\pm 0.33}$ & $15.47_{\pm 0.33}$ & $17.27_{\pm 0.31}$ & $16.17_{\pm 0.33}$ & $16.17_{\pm 0.33}$ & N/A & N/A & N/A & N/A & N/A & $16.17_{\pm 0.33}$ \\
        \texttt{mv} & $53.87_{\pm 0.43}$ & $\underline{\textBF{51.47}}_{\pm 0.85}$ & $51.51_{\pm 0.97}$ & $53.44_{\pm 2.57}$ & $53.44_{\pm 2.57}$ & N/A & N/A & N/A & N/A & N/A & $53.44_{\pm 2.57}$ \\
        \texttt{ds} & $53.87_{\pm 0.43}$ & $\underline{\textBF{51.47}}_{\pm 0.85}$ & $51.51_{\pm 0.97}$ & $53.44_{\pm 2.57}$ & $53.44_{\pm 2.57}$ & $53.44_{\pm 2.57}$ & $53.44_{\pm 2.57}$ & $53.44_{\pm 2.57}$ & $53.44_{\pm 2.57}$ & $53.44_{\pm 2.57}$ & $53.44_{\pm 2.57}$ \\
\hline
    \end{tabular}
    }
\end{table}

\begin{table}[p]
    \centering
    \scriptsize
    \setlength{\tabcolsep}{1.7pt}
    \def\arraystretch{0.925}
        \ContinuedFloat
    \caption{Zero-one loss results (part IX) -- Continued from the previous page.}
    \label{tab:zero-one-loss-results-9}
    \resizebox{\textwidth}{!}{%
    \begin{tabular}{@{} lccccccccccc @{}}
        \toprule
        
        \multirow{2}{*}{\textBF{Approach}} & \multicolumn{3}{c}{\color{baseline}\textBF{Baseline}} & \multicolumn{4}{c}{\color{aggregation}\textBF{Aggregation-level}} & \multicolumn{3}{c}{\color{crowd}\textBF{Crowd-level}} & \multicolumn{1}{c}{\color{ensemble}\textBF{Ensemble}} \\
        
        \cmidrule(lr){2-4} \cmidrule(lr){5-8} \cmidrule(lr){9-11} \cmidrule(lr){12-12}

        & \multicolumn{1}{c}{{\color{baseline}$\textsc{true}^\star$}} & \multicolumn{1}{c}{\color{baseline}$\textsc{def-data}^\star$} & \multicolumn{1}{c}{\color{baseline}\textsc{def}} & \multicolumn{1}{c}{\color{aggregation}\textsc{aeu}} & \multicolumn{1}{c}{\color{aggregation}\textsc{aec}} & \multicolumn{1}{c}{\color{aggregation}\textsc{alu}} & \multicolumn{1}{c}{\color{aggregation}\textsc{alc}}  & \multicolumn{1}{c}{\color{crowd}\textsc{cxu}} & \multicolumn{1}{c}{\color{crowd}\textsc{cec}} & \multicolumn{1}{c}{\color{crowd}\textsc{clc}} & \multicolumn{1}{c}{\color{ensemble}\textsc{ens}}\\

        \hline
        \rowcolor{yescolor!10} \multicolumn{12}{c}{\texttt{spc-rand-2}} \\
        \hline

        \texttt{gt} & $15.19_{\pm 0.22}$ & $15.19_{\pm 0.22}$ & $17.27_{\pm 0.31}$ & $15.93_{\pm 0.21}$ & $15.19_{\pm 0.22}$ & N/A & N/A & N/A & N/A & N/A & $15.93_{\pm 0.21}$ \\
        \texttt{mv} & $19.53_{\pm 0.68}$ & $\underline{16.58}_{\pm 0.58}$ & $22.82_{\pm 0.40}$ & $\underline{16.58}_{\pm 0.58}$ & $\underline{16.58}_{\pm 0.58}$ & N/A & N/A & N/A & N/A & N/A & $\underline{16.58}_{\pm 0.58}$ \\
        \texttt{ds} & $15.87_{\pm 0.16}$ & $15.87_{\pm 0.16}$ & $20.28_{\pm 0.50}$ & $\underline{15.78}_{\pm 0.26}$ & $15.87_{\pm 0.16}$ & $15.87_{\pm 0.16}$ & $15.87_{\pm 0.16}$ & $15.87_{\pm 0.16}$ & $15.87_{\pm 0.16}$ & $16.27_{\pm 0.29}$ & $15.87_{\pm 0.16}$ \\
\bottomrule
    \end{tabular}
    }
\end{table}

\end{document}